\documentclass{article}
\usepackage{graphicx} 
\usepackage{amsmath}
\usepackage{authblk}
\usepackage{amssymb}
\usepackage{amsthm}
\usepackage{tikz}
\usepackage{xfrac}
\usepackage{subcaption}
\usepackage{mathtools}
\usepackage{hyperref}
\usepackage{cleveref}
\usepackage{fullpage}
\usepackage{todonotes}
\usepackage{natbib}
\newtheorem{theorem}{Theorem}[section]

\newtheorem{prop}[theorem]{Proposition}
\newtheorem{lemma}[theorem]{Lemma}

\newtheorem{conjecture}{Conjecture}

\theoremstyle{definition}
\newtheorem{definition}[theorem]{Definition}
\newtheorem{example}[theorem]{Example}
\newtheorem{remark}[theorem]{Remark}

\usepackage{xcolor}

\newcommand{\aff}{\operatorname{aff}}
\newcommand{\R}{\mathbb{R}}

\newcommand{\architecture}{\mathcal{A}}
\newcommand{\paramspace}[1]{\Theta_{#1}}
\newcommand{\gparamspace}[1]{\Tilde{\Theta}_{#1}}
\newcommand{\weights}[1]{W^{(#1)}}

\newcommand{\bias}[1]{b^{(#1)}}
\newcommand{\complex}{\mathcal{C}}
\newcommand{\activation}[1]{a^{(#1)}}
\newcommand{\diag}{\operatorname{diag}}

\newcommand{\preactivation}[1]{z^{(#1)}}
\newcommand{\bhyperplane}[1]{B_{#1}}

\newcommand{\hyperplanes}[2]{\mathcal{H}_{#1}(#2)}
\newcommand{\breakpoints}[1]{B({#1})}
\newcommand{\bcomplex}[1]{\mathcal{B}_{#1}}

\newcommand{\funspace}[1]{\mathcal{F}_{#1}}

\newcommand{\gfiber}[1]{\Tilde{\mathcal{S}}({#1})}

\newcommand{\str}{\operatorname{star}}

\newcommand{\inner}[1]{\langle #1 \rangle}
\newcommand{\id}{\operatorname{Id}}
\newcommand{\relint}{\operatorname{relint}}

\newcommand{\lin}{\operatorname{lin}}
\newcommand{\sgn}{\operatorname{sign}}

\DeclareMathOperator{\GL}{GL}

\newcommand{\HA}{\mathcal{H}}
\newcommand{\up}[1]{^{(#1)}}
\newcommand{\good}{\Theta^{\mathrm{good}}_\architecture}
\setcitestyle{authoryear,round,comma}

\title{Most ReLU Networks Admit Identifiable Parameters
}
\author[1]{Moritz Grillo}
\author[1,2]{Guido Mont\'ufar} 
\affil[1]{Max Planck Institute for Mathematics in the Sciences, Leipzig}
\affil[2]{Departments of Mathematics and Statistics \& Data Science, UCLA}

\date{\today}

\begin{document}
	
	\maketitle
	\begin{abstract}
    We study the realization map of deep ReLU networks, focusing on when a function determines its parameters up to scaling and permutation. 
    To analyze hidden redundancies beyond these standard symmetries, we introduce a framework based on weighted polyhedral complexes. 
    Our main result shows that for every architecture whose input and hidden layers have width at least two, there exists an open set of identifiable parameters.     
    This implies that the functional dimension of every such architecture is exactly the number of parameters minus the number of hidden neurons. 
    We further show that minimal functional representations can still have non-trivial parameter redundancies. 
    Finally, we establish a generic depth hierarchy, whereby for an open set of parameters the realized function cannot be represented generically by any shallower network. 
\smallskip 

    
\end{abstract}

\section{Introduction} 

A central question in the study of deep neural networks is the relationship between the parameter space and the space of functions they realize. 
For a given architecture $\architecture$, these spaces are linked by the \emph{realization map} $\mu_\architecture \colon \paramspace{\architecture} \to \funspace{\architecture}; \, \theta\mapsto f_\theta$, which assigns to each parameter $\theta$, consisting of the weights and biases of the network, the function $f_\theta$ computed by the network. 
The geometry of this map plays a fundamental role, as it shapes the loss landscape and its symmetries, governs implicit biases of gradient-based optimization, and determines when distinct parameter configurations represent genuinely different predictors.  

A key aspect of this relationship is \emph{identifiability}, namely whether a parameter can be uniquely recovered (up to trivial symmetries) from the function it realizes. 
The realization map is generally not injective due to inherent redundancies, including well-known 
\emph{trivial symmetries}  
such as \emph{permutations} of neurons within a layer and \emph{positive scalings} of weights across layers. 
Factoring out these symmetries, we define the quotient parameter space $\overline{\paramspace{\architecture}} = \paramspace{\architecture} / \sim$. 
A parameter $\theta$ is said to have \emph{no hidden symmetries}, or to be \emph{identifiable}, if its fiber in the quotient space consists of a single element. 

Parameter identifiability is closely related to the \emph{functional dimension}, defined as the maximal rank of the Jacobian of the realization map evaluated over finite input sets: 
\[
\dim(\mu_\architecture(\Theta_\architecture)) 
= 
\max_{\theta, X} \operatorname{rank} \left( \frac{\partial f_\theta(X)}{\partial \theta} \right), 
\]
where $X = \{x_1, \dots, x_N\}$ is a finite input set and $f_\theta(X)$ denotes the evaluation of the realized function on $X$. 
Intuitively, the functional dimension is the maximal number of independent directions in function space that can be realized by infinitesimal variations of the parameters. 
If the realization map is smooth on a nonempty open set of identifiable parameters, then it is locally injective modulo the trivial symmetries, and the architecture attains its \emph{expected functional dimension}, given by 
\(
\dim(\mu_\architecture(\Theta_\architecture)) 
=
\dim(\overline{\paramspace{\architecture}}) . 
\)

Previous work has established that certain classes of architectures attain the expected functional dimension. 
In particular, \cite{phuong2020functional} proved for pyramidal networks with decreasing layer widths that there is an open set of parameters that are identifiable among generic parameters, while \cite{grigsby2023hiddensymmetriesrelunetworks}, building on \cite{rolnick2020reverse}, established the expected functional dimension for architectures whose hidden layers have width at least the input dimension $n_i \ge d\ge 2$.
However, a general structural understanding of identifiability for ReLU architectures remains lacking, raising the following question: 
\begin{quote}
\emph{Which ReLU architectures attain the expected functional dimension?} 
\end{quote}

Factoring out trivial symmetries, one may ask whether any remaining non-identifiability can be explained by non-minimality of the architecture. 
In many settings, non-identifiability can be linked to a lack of \emph{minimality}, meaning that the function can be realized by a strictly smaller sub-architecture, or \emph{submodel} \citep{SUSSMANN1992589,
fukumizu00,
PetzkaTS20,
shahverdi2026identifiableequivariantnetworkslayerwise}. 
Such submodels induce redundancy through overparameterization, as parameters corresponding to redundant neurons can be varied without affecting the realized function. 
This raises the following fundamental question for deep ReLU networks: 
\begin{quote} 
\emph{Does non-identifiability in ReLU networks arise solely from  reducibility of the architecture, or can additional forms of redundancy persist even for minimal realizations?} 
\end{quote} 

The difficulty of these questions is compounded by the structure of ReLU networks. The ReLU activation is not differentiable at zero and vanishes on a half-space, leading to realization maps whose Jacobians, evaluated over a finite input set, depend intricately on the network's activation pattern. These patterns are determined by semi-algebraic regions in parameter space whose structure remains poorly understood. 
As a result, characterizing the rank of the Jacobian is highly nontrivial. 
This problem has been studied extensively in the context of the empirical neural tangent kernel (NTK), defined as the Gram matrix of these Jacobians. 
However, such works focus on showing that sufficiently overparametrized networks attain full-rank NTK with high probability over a random parameter. In contrast, the functional dimension is an intrinsic property of the architecture, requiring maximization of the Jacobian rank jointly over the parameters and input datasets. 

In this work, we address these questions by developing a unified framework for general ReLU architectures. 
We show that all deep ReLU architectures of width at least two attain the expected functional dimension. Moreover, such architectures can simultaneously admit open sets of identifiable parameters and open sets of minimal parametrizations that are not identifiable, revealing intrinsic redundancies beyond standard symmetries and submodel structure.

\subsection{Our Contributions} 

In this paper, we address fundamental questions regarding parameter identifiability of deep ReLU networks. 
Our contributions can be summarized as follows: 

\begin{enumerate}
    \item \textbf{A Unified Polyhedral Framework:} In \Cref{sec:polyhedral_framework}, we introduce a framework to study identifiability and parameter fibers, describing functional non-linearities via weighted polyhedral complexes. 
    
    \item \textbf{Identifiability in Almost all Architectures:} 
    In \Cref{sec:identifiable_parameters}, we show that 
    for any architecture whose input and hidden layers have width at least two, there exists a nonempty open subset of parameters that are identifiable among \emph{all} parameters. 
    This substantially generalizes prior genericity and width-dependent results of \cite{grigsby2023hiddensymmetriesrelunetworks} and \cite{phuong2020functional}. 

    \item \textbf{Functional Dimension:} As a consequence, we settle the functional dimension for nearly all ReLU architectures as the number of parameters minus the number of hidden neurons (\Cref{thm:identifiable_dimension}).

    \item \textbf{Minimality $\neq$ Identifiability:} 
    In \Cref{sec:minimality_identifiability}, 
    we construct an open set of minimal parameters with positive-dimensional non-trivial fibers. 
    This demonstrates that redundancy can persist even when no neurons can be removed.
    
    \item \textbf{Generic Depth Hierarchy:} In \Cref{sec:depth_hierarchy}, we  
    show that, for most architectures, there exists an open set of parameters whose realized functions cannot be realized by any shallower network with generic parameters, regardless of width. 
\end{enumerate}

\subsection{Related Work}

\paragraph{Symmetries and Identifiability} 

The study of functional equivalence and minimal realizations in neural networks has a long history. 
Early work by 
\cite{SUSSMANN1992589} studies identifiability and minimality for sigmoidal networks, showing that non-identifiability can arise from non-minimal realizations and providing conditions under which minimal representations are unique. 
In a related direction, \cite{kurkova1994equivalent} investigate when two feedforward networks represent the same function, characterizing equivalence classes of parameters and highlighting the role of redundancy in network representations. 
 \cite{Fefferman1994ReconstructingAN} gives conditions under which knowledge of the output map determines the architecture and parameters up to standard symmetries in the case of tanh activation. 
Moreover, \cite{vlacic2020affinesymmetriesneuralnetwork} develop a general framework in which fibers are generated by affine symmetries of the activation function, giving an exhaustive description except in the presence of nontrivial zero-realizing networks. 

A substantial body of work has examined the impact of parameter symmetries on optimization, from the geometry of loss landscapes in hierarchical structures \citep{FUKUMIZU2000317,
simsek2021geometry} 
to the topology of training dynamics \citep{nurisso2026topology}. A recent survey on parameter symmetries is provided by \citet{zhao2025symmetryneuralnetworkparameter}.

We next highlight a line of work on neural network identifiability in settings where the realization map admits an algebraic structure. For polynomial models, classical results on the dimension of secant varieties \citep{Alexander1995Polynomial} yield functional dimension results for shallow polynomial networks. 
For deep networks with polynomial activations, recent works 
\citep{shahverdi2025learningrazorsedgesingularity, 
usevich2025identifiabilitydeeppolynomialneural, finkel2025activationdegreethresholdsexpressiveness} 
study generic identifiability and expected dimension using tools from algebraic geometry, showing that sufficiently high activation degree ensures identifiability up to permutation and scaling symmetries. 
Recent works have started to classify the parameter symmetries in transformers \citep{tran2025equivariant}. 
 
We also highlight a line of work on  dimension and finite identifiability for restricted Boltzmann machines (RBMs). These are energy-based probabilistic models whose log-probabilities correspond to shallow softplus networks, while their tropicalizations correspond to shallow ReLU networks, both with unit output weights and evaluated on binary inputs \citep{10.1007/978-3-319-97798-0_4}. 
\cite{Cueto2010} showed that many tropicalized RBM architectures attain the expected dimension, and later \cite{doi:10.1137/16M1077489} established that the original model always does. 
However, the dimension of the tropical RBM, and thus of the corresponding shallow ReLU network on binary inputs, remains only partially understood. 
We next review works specific to ReLU networks. 

\paragraph{Identifiability and Dimension of ReLU Networks} 
For ReLU networks, identifiability is inherently combinatorial and geometric due to the piecewise linear activation. 
For shallow ReLU networks, \cite{PetzkaTS20} and \cite{ramakrishnan2026completesymmetryclassificationshallow} provide a complete characterization of symmetries in the two-layer case, identifying redundancies beyond permutation and scaling. Moreover, \cite{DEREICH2022121} characterize minimal representations of functions realized by shallow ReLU networks and determine both the dimension and the number of connected components of the corresponding function space. Building on this perspective, \cite{Dereich_2024} show that optimal approximations of continuous target functions exist within this space.
Our work builds on the reverse-engineering framework of \cite{rolnick2020reverse}, which exploits the geometry of bent hyperplanes and activation boundaries to reconstruct network parameters from the realized function. 

Following this line, \cite{grigsby2023hiddensymmetriesrelunetworks} establish the existence of open sets of parameters satisfying geometric conditions under which the reverse-engineering framework applies for architectures whose hidden layers have width at least the input dimension, thereby showing that these architectures attain the expected functional dimension. Likewise, \cite{phuong2020functional} prove the existence of an open set of parameters that are identifiable among generic parameters for pyramidal architectures with non-increasing widths, provided that all input and hidden layers have width at least two.
For pyramidal architectures, \cite{Bona-Pellissier2023} further provide conditions under which equality of realized functions implies equality of parameters up to trivial symmetries, and exhibit examples that fall outside the cases covered by \cite{phuong2020functional}. 
Moreover, \cite{Stock2022} introduce a locally linear parametrization of the realization map in terms of paths, derive conditions for local identifiability, and provide necessary and sufficient conditions in the shallow case. 

A related notion is that of functional dimension. For ReLU networks, \cite{ELISENDAGRIGSBY2025110636} show that it varies across parameter space, while \cite{grigsby2024functionaldimensionpersistentpseudodimension} relates this local dimension to the persistent pseudodimension as a measure of local capacity. 
Complementarily, the constraints of the realizable functions across regions of parameter space have recently been investigated in \cite{alexandr2025constrainingoutputsreluneural}.

Finally, we note that, closely related to the functional dimension, an extensive line of works has studied the degree of overparametrization required for the NTK to attain full rank with high probability. 
Several works have established lower bounds on the smallest eigenvalue \citep{pmlr-v139-nguyen21g,bombari2022memorization,montanari2022interpolation,karhadkar2024bounds}. 
Focusing on rank, \cite{karhadkar2024mildly} show that for mildly overparametrized networks most activation patterns 
correspond to parameter regions where the Jacobian attains full rank equal to the number of data points.

\paragraph{Linear Regions and Depth Separation} 

There is substantial work that studies subdivisions of the input space induced by linear regions and their connection to depth separation. 
One direction studies approximation-theoretic separation between shallow and deep networks, including the results of \cite{pmlr-v49-eldan16}, \cite{pmlr-v49-telgarsky16}, and \cite{Mhaskar2017when}, as well as works quantifying complexity through the proliferation of linear regions \citep{montufar2014number,pascanu2014number,pmlr-v70-raghu17a,pmlr-v80-serra18b,Balestriero2019power,pmlr-v247-ergen24a}. 
These results demonstrate that depth enhances expressive power, as reflected in approximation rates, oscillatory behavior, and exponential growth in the number of linear regions. 
More recent works in this direction provide explicit counting formulas by relating linear regions to Minkowski sums of polytopes 
\citep{montufar2022minkowski}, and recent works aim to obtain precise structural characterization of the associated polytopes \citep{balakin2025maxoutpolytopes}. 

A complementary line of work focuses on depth hierarchies for exact representation of continuous piecewise linear functions by ReLU networks irrespective of width. 
This perspective was initiated by \cite{doi:10.1137/22M1489332} via the representability of the maximum function and related maps. 
Subsequent works establish lower bounds in several regimes:  \cite{haase2023lower} prove logarithmic depth lower bounds for networks with integer weights, \cite{averkov2025on} extend this to rational-weights, and \cite{grillo2025depthbounds} derive lower bounds under compatibility assumptions with the braid arrangement. 
At the same time, \cite{bakaev2026betterneuralnetworkexpressivity} show that earlier conjectured upper bounds for exact depth hierarchies are not tight by constructing shallower representations of the maximum function. 
Our generic depth-hierarchy result is complementary to this line: rather than focusing on specific functions such as max, it shows that for generic parameters depth cannot be traded for width.

\paragraph{Weighted Polyhedral Complexes for Neural Networks} 

Polyhedral and tropical perspectives have become central tools for understanding ReLU networks and, more generally, continuous piecewise-linear models. 
 \cite{zhang2018tropical} and \cite{charisopoulos18_tropicalapproachneural} identify ReLU networks with tropical rational maps, relating their linear regions and decision boundaries to polyhedral and tropical-geometric structures. 
This viewpoint has since been refined in several directions. 
For instance, \cite{brandenburg2024the} study semialgebraic subdivisions of parameter space and the combinatorial types of decision boundaries, highlighting that parameter space itself admits a rich stratification. 
A broader overview of polyhedral methods in deep learning was recently presented by \cite{HuchetteMunozSerraTsay:2023}. 

Closer to our setting, weighted polyhedral complexes have been used to encode function representations and decompositions. In particular, \cite{Tran_2024} study minimal representations of tropical rational functions, directly addressing questions of redundancy and factorization length for piecewise-linear maps. 
In a similar spirit, \cite{brandenburg2025decomposition} introduce decomposition polyhedra for CPWL functions, showing that, once an underlying complex is fixed, the space of admissible decompositions is organized as a polyhedron whose minimal decompositions are vertices. 
This viewpoint is closely aligned with our use of weighted polyhedral complexes to separate visible function geometry from latent parameter geometry.

\section{Preliminaries}

\paragraph{Notation} 

For any $n \in \mathbb{N}$, we let $[n]:=\{1, \dots, n\}$. 
For a vector $x \in \R^d$, we denote by $[x]_+$ the entrywise application of the ReLU activation function $\max\{0, x_i\}$. 
For a set of indices $S \subseteq [n]$, let $D_S \in \R^{n \times n}$ denote the diagonal selection matrix where $(D_S)_{ii} = 1$ if $i \in S$ and $(D_S)_{ii} = 0$ otherwise. We also simply write $\| \cdot \|$ for the euclidean norm $\|\cdot \|_2$.
\paragraph{ReLU Networks}
A \emph{ReLU layer} with $n_{\ell-1}$ inputs and $n_\ell$ outputs, weight matrix $\weights{\ell} \in \R^{n_\ell \times n_{\ell-1}}$, and bias vector $\bias{\ell} \in \R^{n_\ell}$ computes the map 
\(
f_{\weights{\ell},\bias{\ell}}(x) = [\weights{\ell}x + \bias{\ell}]_+.
\)

A \emph{ReLU network} with architecture $\architecture=(n_0,\ldots,n_{L+1})$ is the composition
\[
f_\theta(x) = T_{\weights{L+1}, \bias{L+1}} \circ f_{\weights{L}, \bias{L}} \circ \cdots \circ f_{\weights{1}, \bias{1}}(x),
\]
where the final layer $T_{\weights{L+1}, \bias{L+1}}(y) = \weights{L+1}y + \bias{L+1}$ is an affine linear map. 
We sometimes write $d$ for the input dimension $n_0$. 

The tuple 
\(
\theta = (\weights{1},\bias{1},\dots,\weights{L+1},\bias{L+1})
\)
is called the \emph{parameter} of the network. The \emph{parameter space} is given by 
\(
\paramspace{\architecture} \cong \bigoplus_{\ell=1}^{L+1} \paramspace{\ell},
\) where $\paramspace{\ell} =\R^{n_\ell \times n_{\ell-1}} \times \R^{n_\ell}$. 
The \emph{realization map} $\mu_\architecture \colon \paramspace{\architecture} \to \funspace{\architecture}$ is defined by $\theta \mapsto f_\theta$. 
For $\ell \in [L]$, the \emph{preactivation} at layer $\ell$ is 
a function of the network's input defined by 
\(
\preactivation{\ell, \theta}(x) = \weights{\ell} \activation{\ell-1, \theta}(x) + \bias{\ell},
\)
where $\activation{\ell, \theta}(x) := [\preactivation{\ell, \theta}(x)]_+$ is the \emph{activation} at layer $\ell$ and $\activation{0, \theta} = x$. 
The tuple $(j,\ell)$ indexes a \emph{neuron in layer $\ell$}  with preactivation $\preactivation{\ell}_j$ and activation $\activation{\ell}_j$. 
The $j$-th neuron in layer $\ell$ has input weights given by $W_j^{(\ell)}$,
the $j$-th row,
and output weights given by $W_{:,j}^{(\ell+1)}$, the $j$th column of $W^{(\ell+1)}$.

We also use the notation $\Theta\up{\ell}=\bigoplus_{k=1}^{\ell} \paramspace{k}$ for the parameter space of the first $\ell$ layers. For a parameter $\theta = (\theta_1,\ldots,\theta_{L+1}) \in \paramspace{\architecture}$, we use the notation $\theta\up{\ell} = (\theta_1,\ldots,\theta_{\ell}) \in \Theta\up{\ell}$ and denote by $f_{\theta\up{\ell}}$ the corresponding truncated network with ReLU applied to the output layer for $\ell \leq L$. Note that the functions $f_{\theta\up{\ell}}$ and $\activation{\ell,\theta}$ are identical, but the distinguished notation will become handy in \Cref{sec:identifiable_parameters}. 
\paragraph{Symmetries, Identifiability and Minimality}
The parameter space $\paramspace{\architecture}$ carries a natural equivalence relation, denoted $\theta \sim \eta$, generated by the \emph{trivial symmetries} of the networks \citep[see, e.g.,][]{rolnick2020reverse,phuong2020functional}: 
\begin{enumerate}
    \item The neurons within any layer $\ell \in [L]$ can be reordered without altering the realized function $f_\theta$.
    \item  For any neuron $j$ in layer $\ell$ and any scaling factor $\lambda > 0$, one can multiply the $j$-th row of $\weights{\ell}$ and the bias entry $\bias{\ell}_j$ by $\lambda$, provided the $j$-th column of $\weights{\ell+1}$ is multiplied by $\lambda^{-1}$. 
\end{enumerate}
We denote by $\overline{\paramspace{\architecture}} = \paramspace{\architecture} / \sim$ the parameter space modulo these trivial symmetries. A parameter $\theta$ is \emph{identifiable} if its fiber consists only of its symmetry orbit, i.e., $f_\eta = f_\theta$ implies $\eta \sim \theta$. In terms of the realization map, $\theta$ is identifiable if its image in the quotient space has a trivial fiber: $\overline{\mu}_\architecture^{-1}(f_\theta) = \{[\theta]\}$. By abuse of notation, we will often identify a parameter $\theta \in \paramspace{\architecture}$ with its equivalence class $[\theta] \in \overline{\paramspace{\architecture}}$ whenever we work in the quotient space. 
Moreover, if $\architecture'=(n_0,n_1',\dots,n_L',n_{L+1})$ satisfies $n_\ell'\le n_\ell$ for all $\ell$, then $\architecture'$ is called a \emph{subarchitecture} or \emph{submodel} of $\architecture$, and $\theta$ is \emph{minimal} if no strict submodel can  realize the same function.

\paragraph{Polyhedral Geometry}
For a vector $a \in \R^{d}$ and a scalar $b \in \R$, the \emph{hyperplane} $H \coloneqq \{x \in \R^d \mid \inner{a,x}+b=0\}$ subdivides $\R^d$ into \emph{half-spaces} $H^{+}\coloneqq \{x \in \R^d \mid \inner{a,x}+b\geq 0\} \text{ and }{H}^{-} \coloneqq \{x \in \R^d \mid \inner{a,x}+b\leq0\}.$
A finite set of hyperplanes $\mathcal H = \{H_1,\dots,H_n\}$ is called a \emph{hyperplane arrangement}. 
A hyperplane arrangement $\mathcal H$ in $\R^d$ is called \emph{generic} if the intersection of any subset of $k \leq d$ hyperplanes has codimension $k$, and no $d+1$ hyperplanes have a common intersection. 
Given an arbitrary arrangement $\mathcal H$, let
$
L \coloneqq \bigcap_{H \in \mathcal H} \lin(H)
$
be the maximal linear subspace contained in all hyperplanes.
Projecting $\R^d$ orthogonally onto $L^\perp$ yields an induced arrangement
$\mathcal H^{\mathrm{ess}}$ in $L^\perp$, called the \emph{essentialization} of
$\mathcal H$.

A \emph{polyhedral complex} $\complex$ is a finite collection of polyhedra such that 
\begin{enumerate}
    \item $\emptyset \in \complex$,
    \item if $P \in \complex$, then all faces of $P$ are in $\complex$, and 
    \item if $P, P' \in \complex$ and $P\cap P'\neq\emptyset$, then $P \cap P'$ is a face of both $P$ and $P'$. 
\end{enumerate} 
For a polyhedral complex $\complex$ in $\R^d$ and $k \leq d$ we denote by $\complex^k$ the set of $k$-dimensional polyhedra in~$\complex$.  We call $\complex^{d-1}$ the \emph{facets}, $\complex^d$ the \emph{regions} and $\complex^{d-2}$ the \emph{ridges} of $\complex$.
The \emph{dimension} of a complex~$\complex$ is the maximal dimension of its polyhedra. A complex is \emph{pure} (of dimension $k$) if all maximal polyhedra are of dimension $k$. 
A pure polyhedral complex in $\R^d$ of dimension $d-1$ is called a \emph{piecewise linear hypersurface}. A \emph{local normal} at $x \in \R^d$ of a piecewise linear hypersurface is a vector that is orthogonal to one of the maximal polyhedra containing $x$.

Given a face $\sigma \in \complex$, we denote by $\aff(\sigma) \subseteq \R^d$ the unique smallest affine subspace containing $\sigma$. The \emph{relative interior} of $\sigma$ is the interior of $\sigma$ inside the affine space $\aff(\sigma)$ and is denoted by $\relint(\sigma)$.

Let $\complex$ be a polyhedral complex in $\R^d$ and let $\tau \in \complex$ be a face.
The \emph{star} of $\tau$ is
\(
\str_\complex(\tau) \coloneqq \{\, \sigma \in \complex \mid \tau \subseteq \sigma \,\}.
\)
When $\complex$ is clear from the context, we omit the subscript and write $\str(\tau)$. For any $k\leq d$ and any faces $\tau \in \complex^{k-1},\sigma \in \complex^{k}$ with $\tau \subseteq \sigma$, let $e_{\sigma/\tau} \in \R^d$ denote the normal vector of $\tau$ relative to $\sigma$, defined as the unique unit vector that lies in $\aff(\sigma)$, is orthogonal to $\aff(\tau)$, and points from the relative interior of $\tau$ into the relative interior of $\sigma$. 
For a subset  $S \subseteq \complex$, we denote the \emph{support} by $|S| \coloneqq \bigcup_{P \in S} P$ and by $\# S$ the number of elements contained in $S$.

A function $f\colon\R^d\to\R^m$ is \emph{continuous and piecewise linear} (CPWL), if there exists a complete polyhedral complex $\complex$ such that the restriction of $f$ to each polyhedron $P\in\complex$ is an affine linear function. 
If this condition is satisfied, we say that $f$ and $\complex$ are \emph{compatible} with each other. 
A vector $x \in  \R^d$ is a \emph{breakpoint} of $f$ if there is no open set $U \subseteq \R^d$ containing $x$ such that $f$ is affine linear on $U$. 
For a CPWL function $f\colon \R^d \to \R^m$, let $\breakpoints{f} $ be the set of breakpoints of $f$. 

A  polyhedral complex $\complex$ in $\R^d$ of dimension at least $d-1$ can be equipped with a \emph{weight function} $c \colon \complex^{d-1} \to \R^m$. 
One can use such a weight function to describe a CPWL function by storing how the linear parts change when crossing a facet in a compatible complex as formalized in the following lemma. See also \Cref{fig:tropical-weight-local-two-cell} for an illustration.
\begin{lemma}
\label{lem:structure_theorem_tropical_geometry}
   Let $f \colon \R^d \to \R^m$ be a CPWL function compatible with a polyhedral complex $\complex$. 
   For a facet $\sigma \in \complex^{d-1}$, let $P,Q \in \complex^d$ be the unique polyhedra such that $P \cap Q = \sigma$, 
   and 
   suppose that $f(x) = A_Px + b_P$ for all $x \in P$ and $f(x)= A_Qx + b_Q$ for all $x \in Q$. 
   Then the tropical weight $c_f \colon \complex^{d-1} \to \R^m$ defined by
   \[
   c_f(\sigma)  \coloneqq A_{P} e_{P/\sigma} +  A_{Q} e_{Q/\sigma} =  (A_{P} - A_{Q})e_{P/\sigma} 
   \]
uniquely determines $f$ up to adding a global linear function.
\end{lemma}
\begin{proof}
Write $f=(f_1,\dots,f_m)$ with $f_i\colon\mathbb{R}^d\to\mathbb{R}$. 
Since $f$ is CPWL and compatible with the polyhedral complex $\complex$, each coordinate function $f_i$ is CPWL and compatible with the same complex. 
Then the result follows from Tropical Geometry (\citealp[Proposition 3.3.10]{maclagan2015introduction} and \citealp[Proposition 3.3.2]{maclagan2015introduction}), as outlined by \cite{brandenburg2025decomposition}. 
\end{proof}
\begin{figure}[t]
\centering
\begin{subfigure}[t]{0.48\textwidth}
\centering
\begin{tikzpicture}[scale=0.7]
    \footnotesize
    \draw[thick] (0,0) -- (0,3.0);
    \draw[thick] (0,0) -- (4.0,0);
    \draw[thick] (0,0) -- (-2.6,-2.0);

    \node[above] at (0,3.0) {$c_f(\sigma_2)=\begin{pmatrix}
        1\\-1
    \end{pmatrix}$};
    \node[below] at (-2.6,-2.0) {$c_f(\sigma_3)=\begin{pmatrix}
        \sqrt{2}\\ -\sqrt{2}
    \end{pmatrix}$};

    \node[align=center] at (2.05,1.55) {$A_P=\left(\begin{smallmatrix}1&1\\1&0\end{smallmatrix}\right)$};
    \node[align=center] at (1.3,-1.6) {$A_Q=\left(\begin{smallmatrix}1&0\\1&1\end{smallmatrix}\right)$};
    \node[align=center] at (-2.15,1.05) {$A_R=\left(\begin{smallmatrix}0&1\\2&0\end{smallmatrix}\right)$};

    \node[below] at (2.15,0) {$\sigma_1=P\cap Q$};
    \node[above] at (4.15,-0.05) {$c_f(\sigma_1)=\begin{pmatrix}
        1\\-1
    \end{pmatrix}$};
\end{tikzpicture}
\caption{Illustration of a weighted polyhedral complex. The matrices on the maximal polyhedra represent the linear parts of $f$, and the labels on the facets show how the tropical weights record the magnitude of the change of the linear part across the facets. }
\label{fig:tropical-weight-local-two-cell}
\end{subfigure}\hfill
\begin{subfigure}[t]{0.48\textwidth}
\centering
\begin{tikzpicture}[x=1cm,y=1cm,line cap=round,line join=round]
    \footnotesize
    \draw[black!80, line width=0.8pt] (-2.7,0) -- (2.9,0);
    \draw[black!80, line width=0.8pt] (0,-2.2) -- (0,2.5);

    \draw[red!80!black, line width=1pt] (-2.35,1.7) -- (0,1.7);
    \draw[red!80!black, line width=1pt] (0,1.7) -- (2.25,0);
    \draw[red!80!black, line width=1pt] (2.25,0) -- (2.25,-2.45);

    \fill[red!80!black] (0,1.7) circle (1.4pt);
    \fill[red!80!black] (2.25,0) circle (1.4pt);
    \fill[black!80] (0,0) circle (1.4pt);

    \draw[->,red!80!black,thick] (-2.05,1.7) -- (-2.05,2.15);
    \draw[->,black!80,thick] (-2.05,0) -- (-2.05,0.45);
    \draw[->,black!80,thick] (0.05,-1.35) -- (0.65,-1.35);

    \node at (-1.1,2.08) {$-++$};
    \node at (1.25,1.28) {$+++$};
    \node at (-1.25,0.48) {$-+-$};
    \node at (0.58,0.45) {$++-$};
    \node at (-1.25,-1.0) {$---$};
    \node at (0.72,-1.05) {$+--$};
    \node at (2.95,-1.05) {$+-+$};

\end{tikzpicture}
\caption{Illustration of the canonical polyhedral complex. 
The black hyperplanes represent the zero-loci of two neurons in the first hidden layer and the red bent hyperplane the zero-locus of a neuron in the second hidden layer. 
The sign sequences index polyhedra of the complex.} 
\label{fig:sign-pattern-bent-hyperplane}
\end{subfigure}
\caption{Illustration of weighted polyhedral complexes, the canonical polyhedral complex and bent hyperplanes.}
\end{figure}
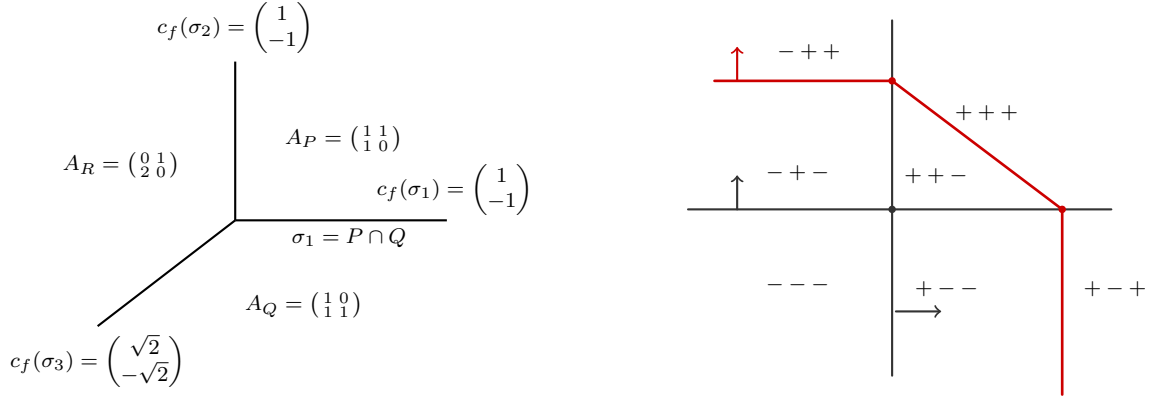

\section{A Polyhedral Framework for Identifiability}
\label{sec:polyhedral_framework}
In this section we develop the polyhedral language that will be used throughout the paper to study identifiability. The main idea is to separate two types of structure associated with a ReLU network: the \emph{canonical polyhedral complex}, which records all sign changes of neuron preactivations layer by layer, and the \emph{breakpoint complex}, which records only those pieces of this geometry that remain visible in the realized function.  We first introduce the canonical complex and bent hyperplanes, then identify a generic regime in which their geometry is well behaved, and finally show how the local weighted geometry of the breakpoint complex gives rise to the dependency graph that encodes layer precedence. Finally, we establish sufficient conditions for a parameter to be \emph{identifiable among generic parameters}. Some of the concepts introduced along the way have appeared, in various forms, in prior work \citep{ELISENDAGRIGSBY2025110636, marissa, phuong2020functional, rolnick2020reverse}. 
The main goal of this section is to develop a unified polyhedral language for these concepts to rigorously state a sufficient condition for identifiability. 
\subsection{Canonical Polyhedral Complex and Bent Hyperplanes}
In this subsection, following \cite{doi:10.1137/20M1368902}, we introduce the canonical polyhedral complex and the associated bent hyperplanes. The canonical complex records the activation regions of the hidden neurons, while bent hyperplanes isolate the loci where individual neurons change their affine behavior.
\paragraph{Canonical Polyhedral Complex} 
\label{sec:canonical_complex}
Let $\phi_{W,b} \colon \R^d \to \R^m$ be a ReLU layer with weight matrix $W \in \R^{m \times d}$ and bias vector $b \in \R^m$. 
The breakpoints of $\phi_{W,b}$ are contained in the hyperplane arrangement 
$\HA_{W,b} \coloneqq \{\{x \in \R^d \mid W_ix + b_i = 0\}\}_{i=1}^m$.
Each sign pattern $\mathbf{s} =(s_1,\ldots, s_m) \in \{+,0,-\}^m$ determines a (possibly empty) polyhedron 
\(
P_{\mathbf{s}} \coloneqq \bigcap_{i=1}^m H^{s_i}_{W_i,b_i},
\)
where $H^{+}$ and $H^{-}$ denote the corresponding half-spaces and $H^{0}$ the
hyperplane.
The collection of all such polyhedra forms the \emph{canonical polyhedral complex} $\complex_{W,b}$ of $\phi_{W,b}$. 
On each $P_{\mathbf{s}}$, the map $\phi_{W,b}$ is affine linear. If $\mathcal{H}$ is a hyperplane arrangement given by $W$ and $b$, then we also write $\complex_\mathcal{H}$ for $\complex_{W,b}$.

Let $\theta \in \paramspace{\architecture}$ for $\architecture=(n_0,\ldots,n_{L+1})$. 
The canonical polyhedral complex associated with $\theta$ is constructed iteratively, layer by layer. 
We start with the trivial complex $\complex_{\theta,0} \coloneqq \R^d$. 
Suppose that at stage $\ell-1$ we have a polyhedral complex $\complex_{\theta,\ell-1}$ such that all preactivations up to layer $\ell$ are affine linear on each polyhedron  $R\in \complex_{\theta,\ell-1}$. 
For each such polyhedron $R$ and each neuron $(\ell,j)$, the preactivation $\preactivation{\ell}_j$ restricts to an affine map on $R$. 
If $\preactivation{\ell}_{j}{|_R}(x)$ is not constant, then its zero set defines a hyperplane 
$
H_R(\ell,j) \coloneqq \{x \in \aff(R) \mid {\preactivation{\ell}_{j}}{|_R}(x) =0 \}
$ 
which coincides with the breakpoint set introduced by neuron $(\ell,j)$ on $R$. 
If $\preactivation{\ell}_{j}{|_R}(x)$ is constant, 
then $H_R(\ell,j)$ is simply $\aff(R)$ or the empty set, in which case it does 
not affect the further refinement described below. 
We refine the complex $\complex_{\theta,\ell-1}$ by subdividing each $R$ using the hyperplane arrangement induced by layer $\ell$ 
on $\aff(R)$: 
 
\[
\complex_{\theta,\ell} 
\coloneqq
\{\, R \cap \bigcap_{j=1}^{n_\ell} H_R^{s_j}(\ell,j) \mid R \in \complex_{\theta,\ell-1},\ \mathbf{s} \in \{+,0,-\}^{n_\ell} \,\}.
\]

We define the \emph{canonical polyhedral complex} of $\theta$ as
$\complex_{\theta} := \complex_{\theta,L}$. 
Consequently, the polyhedra of $\complex_\theta$ are indexed by global activation patterns, that is, 
\(
\complex_{\theta} = \{P_\mathbf{s} \mid \mathbf{s} \in \{+,0,-\}^{n_1} \times \cdots \times \{+,0,-\}^{n_L}\},
\)
where each polyhedron $P_{\mathbf{s}} \coloneqq \overline{\{x \in \R^d \mid \sgn \preactivation{\ell}_j(x) = \mathbf{s}_{\ell,j}, \ell\in[L], j\in[n_\ell]\}}$ is the closure of all inputs realizing the
corresponding signs of all neuron preactivations. For $P \in \complex_\theta$, we also write $\mathbf{s}_{\ell,j}(P) = \sgn \preactivation{\ell}_j(x)$ for any $x \in \relint(P)$. 
On every such polyhedron, the function $f_\theta$ restricts to an affine-linear map. 

For $S=(S_1, \cdots, S_\ell)$ with $S_k\subseteq [n_k]$ , let \[P(S) \coloneqq \overline{\{x \in \R^d \mid \sgn \preactivation{k, \theta}_j(x) = + \text{ iff } j \in S_k  \text{ for all } k \in [\ell], j \in [n_k]\}}.\]
All maximal polyhedra in $\complex_{\theta,\ell}$ are of this form. For $P \in \complex_\theta$, we denote by $\mathbf{s}(P)$ the corresponding sign pattern and by $S(P)=(S_1(P),\ldots,S_L(P))$ the set of active neurons, that is, $j \in S_\ell(P)$ if and only if $\mathbf{s}_{\ell,j}(P)=+$ for all $j \in [n_\ell]$.

We also need to describe the geometry locally, motivating the following definition.
    Let $P \subseteq \R^d$ be a polyhedron and $f_\theta \colon \R^d \to \R^m$. Then we define the \emph{local canonical polyhedral complex} as $\complex_{\theta,\ell}(P) \coloneqq \{R \cap P \mid R \in \complex_{\theta,\ell}\}$. The polyhedra in  $\complex_\theta(P)$ inherit the activation patterns from $\complex_\theta$. 

\paragraph{Weighted Canonical Complex}A useful way to encode a CPWL function compatible with a polyhedral complex is via its \emph{weight function}. By \Cref{lem:structure_theorem_tropical_geometry}, if $f \colon \R^d \to \R^m$ is compatible with a polyhedral complex $\complex$, then $f$ determines a map
$
c_f \colon \complex^{d-1} \to \R^m
$
that records the change of the affine-linear part of $f$ across each facet. This weight function determines $f$ up to adding a global affine-linear map. For a neural network $f_\theta$, we write the induced weight function on the canonical polyhedral complex $\complex_\theta$ simply as
$
c_\theta \colon \complex_\theta^{d-1} \to \R^m.
$

\paragraph{Bent Hyperplanes} 
For a neuron $(j,\ell)$ in layer $\ell \in [L]$, the breakpoints form a piecewise linear hypersurface. We define the \emph{bent hyperplane} associated with the preactivation $\preactivation{\ell,\theta}_j$ as:
\[
\bhyperplane{j,\ell}(\theta) \coloneqq \bigcup\limits_{\substack{R \in \complex_{\theta,\ell-1} \\ \preactivation{\ell}_{j}{|_R} \text{ non-constant}}} \{x \in R \mid \preactivation{\ell}_{j}{|_R}(x) = 0\}.
\]   
The bent hyperplane $\bhyperplane{j,\ell}(\theta)$ is the locus of points where the $j$-th neuron of layer $\ell$ is nonlinear. See also \Cref{fig:sign-pattern-bent-hyperplane} for an illustration of the canonical polyhedral complex and bent hyperplanes. This formulation of the bent hyperplane differs slightly from the one provided by \cite{doi:10.1137/20M1368902} and \cite{rolnick2020reverse}, which define it as the full zero-locus of the neuron's preactivation. While that definition would technically include entire regions where the preactivation vanishes identically, our definition focuses on the locus where the neuron actually changes its linear behavior. Importantly, for generic parameters, preactivations are non-zero almost everywhere, and the two definitions coincide

\paragraph{Linear Regions} 
A \emph{linear region} of a CPWL function $f \colon \R^d \to \R^m$ is a maximal subset of $\R^d$ with connected interior on which $f$ is affine-linear. 
In general, a linear region of $f_\theta$ may be a union of several maximal polyhedra of $\complex_\theta$ if the affine-linear functions on adjacent polyhedra happen to coincide. 
We say that $f_\theta$ satisfies the \emph{Linear Region Assumption (LRA)} on a set $X \subseteq \R^d$ if the linear regions of the restricted function $f_\theta|_X$ coincide exactly with the sets $P \cap X$, where $P$ ranges over the maximal polyhedra of $\complex_\theta$ intersecting $X$.

\subsection{The Geometry of Generic Parameters}
We now introduce geometric conditions ensuring that the canonical polyhedral complex and the bent hyperplanes are well behaved. These conditions hold for almost all parameter choices and define the generic regime used throughout the paper, which we define below.

    \begin{definition} 
	\label{def:generic_parameter}
	We call a parameter $\theta \in \paramspace{\architecture}$ \emph{generic} if it satisfies the following two conditions:
	\begin{enumerate}
		\item For every layer $\ell \in [L]$ and every region $R \in \complex_{\theta,\ell-1}$, the essentialization of the hyperplane arrangement $\{H_R(\ell,j)\}_{j \in[n_\ell]}$ is generic and does not intersect any vertex of $R$.
		\item For every $k,\ell \in [L+1]$ and every sequence of index sets $S_i \subseteq [n_i]$ for $k\le i \le \ell-1$, the matrix product 
		\[
		W^{(\ell)} D_{S_{\ell-1}} W^{(\ell-1)} \cdots D_{S_k} W^{(k)}
		\]
		achieves the maximum possible rank, namely $\min(n_{k-1}, n_\ell, \min_{k \le i \le \ell-1} |S_i|)$.
	\end{enumerate}
    We denote by $\gparamspace{\architecture} \subseteq \paramspace{\architecture}$ the set of generic parameters.
\end{definition}

\begin{lemma}
\label{lem:generic_open_dense}
    Generic parameters form an open and dense subset of $\paramspace{\architecture}$.
\end{lemma}

\begin{proof}
The set of generic parameters is the complement of a finite union of proper algebraic varieties. 
To see this for the first condition, fix a sequence of activation patterns. 
Under such an activation pattern, the zero set of each neuron defines a hyperplane whose normal and translation are polynomial functions of the network's weights and biases. 
Requiring that the essentialization of $\{H_R(\ell,j)\}_{j \in[n_\ell]}$ is generic and does not intersect any vertex of $R$, amounts to avoiding the zero set of a finite collection of polynomial equations given by the minors of matrices consisting of facet-defining normals of the region $R$ and normals of the hyperplanes. 

The second condition requires that products of weight matrices and diagonal selection matrices 
have rank at least 
$\min(n_{k-1}, n_\ell, \min_{k \le i \le \ell-1} |S_i|)$. 
The rank $\geq r$ condition fails if and only if all $r\times r$ minors vanish. 
Since these minors are polynomial functions of the network weights, 
the failure of the rank condition is characterized by the zero locus of a finite collection of polynomials. 

In both cases, the exceptional sets are proper algebraic varieties. Their union is therefore closed with empty interior, and its complement is open and dense. 
\end{proof} 
\begin{definition}
    A parameter $\theta$ is \emph{supertransversal} if every face $\tau \in \complex_{\theta}^{d-k}$ is contained in exactly $k$ bent hyperplanes, and \emph{cancellation-free} if, for every facet $\sigma \in \complex^{d-1}_{\theta}$, it holds that 
\[
c_{\theta}(\sigma)=0
\quad \Longleftrightarrow \quad
\text{there exists a layer } k \in [L] \text{ such that } \mathbf{s}_{k,j}(\sigma) = - \text{ for all } j \in [n_k].
\] 
\end{definition}

The concept of supertransversality was
introduced in an early version in the work of \cite{marissa} in a stronger form, 
whose defining conditions imply those in the definition given above. 
See \Cref{fig:nonsupertransversal} for an illustration of a situation where supertransversality fails. 
Cancellation-freeness, without using this specific name, is a property that has already been used by \cite{phuong2020functional}.

We will need the following standard lemma which we include for completeness.
\begin{lemma}
\label{lem:essentialization_preserves_codimension}
    Let $\mathcal{H}$ be a hyperplane arrangement whose essentialization is generic, and let $\tau$ be a codimension-$k$ face 
    of the polyhedral complex induced by $\mathcal{H}$. 
    Then $\tau$ is contained in exactly $k$ hyperplanes of $\mathcal{H}$. 
\end{lemma}

\begin{proof}
    Let $L$ be the lineality space of $\mathcal{H}$. Passing to the essentialization $\mathcal{H}^{\mathrm{ess}}$ in $L^\perp$ preserves face inclusions and intersections. Since $\mathcal{H}^{\mathrm{ess}}$ is a generic arrangement, any codimension-$k$ face in $\mathcal{H}^{\mathrm{ess}}$ is formed by the intersection of exactly $k$ hyperplanes. Lifting this back to $\R^d$, $\tau$ must also be contained in exactly $k$ hyperplanes of $\mathcal{H}$.
\end{proof}

\begin{lemma}
\label{lem:generic_implies_supertransversal}
Generic parameters are supertransversal.
\end{lemma}

\begin{proof}
We prove the statement by induction on the layers of the network. 
For the first layer, the bent hyperplanes are 
simply the 
affine hyperplanes $H_1,\dots,H_{n_1}$. 
Since $\theta$ is generic (\Cref{def:generic_parameter}), the essentialization of this hyperplane arrangement is generic (Condition 1 of \Cref{def:generic_parameter}). 
Hence, by \Cref{lem:essentialization_preserves_codimension}, any codimension-$k$ face $\tau \in \complex_{\theta,1}^{d-k}$ lies in exactly $k$ hyperplanes. 
This establishes the base case. 

Assume that for the first $\ell-1$ layers, every codimension-$k$ face $\tau \in \complex_{\theta,\ell-1}^{d-k}$ lies in exactly $k$ bent hyperplanes from $\ell-1$ layers. 
Let $\tau \in \complex_{\theta,\ell}^{d-k}$ be a codimension-$k$ face in the refined complex obtained after adding layer $\ell$. 
By construction, $\tau$ arises as the intersection of some polyhedron $R \in \complex_{\theta,\ell-1}$ with a collection of half-spaces or hyperplanes 
\[
\{H_R^{s_j}(\ell,j)\}_{j \in[n_\ell],\, s_j \in \{+,-,0\}}
\]
induced by the neurons in layer $\ell$.  
Let $k'$ denote the codimension of $R$ in $\R^d$. 
By the induction hypothesis, $R$ is contained in exactly $k'$ bent hyperplanes from the first $\ell-1$ layers. 

If $k'=d$ (so that $\dim(R)=0$), then $\tau = R$ is a vertex. 
Since $\theta$ is generic, 
no hyperplane from layer $\ell$ passes through a vertex of $R$ (Condition 1 of \Cref{def:generic_parameter}). 
Hence, no additional intersections with layer-$\ell$ hyperplanes occur, and $\tau$ is contained in exactly $k'=k$ bent hyperplanes. 

Otherwise, $R$ has codimension $k'<d$ in $\R^d$, and the codimension of $\tau \subseteq R$ relative to $\aff(R)$ is $k - k'$. 
Since $\theta$ is generic, the essentialization of the collection $\{H_R(\ell,j)\}$ is generic (\Cref{def:generic_parameter}). 
Hence, by \Cref{lem:essentialization_preserves_codimension}, $\tau$ lies in exactly $(k - k')$ of these  
hyperplanes. 
Altogether, $\tau$ is contained in exactly 
\(
k' + (k - k') = k 
\)
bent hyperplanes, as claimed. 
\end{proof}

See \Cref{fig:nonsupertransversal} for an illustration of a parameter that is not supertransversal. 
In supertransversal networks one can assign to each facet in the canonical polyhedral complex a unique bent hyperplane, which allows one to compute the tropical weights directly from the parameters. 

\begin{prop}
\label{prop:tropweightofNNs}
Let $\theta$ be supertransversal. Let $\sigma \in \complex_\theta^{d-1}$ be a facet, and let $B_{\ell,i}$ be the unique bent hyperplane containing $\sigma$. For each layer $k \in [L]$, let $S_k \subseteq [n_k]$ denote the set of strictly active neurons on the relative interior of $\sigma$. Then the tropical weight $c_\theta(\sigma)$ is given by
\[
c_\theta(\sigma)
=
\left\|
(\weights{1})^\top D_{S_1} \cdots (\weights{\ell-1})^\top D_{S_{\ell-1}} (\weights{\ell})^\top e_i
\right\|_2
\;
\weights{L+1} D_{S_L} \cdots \weights{\ell+1} e_i.
\]
\end{prop}

\begin{proof}
Let $R,Q \in \complex_\theta^d$ be the two full-dimensional polyhedra adjacent to $\sigma$, and suppose, without loss of generality, that neuron $i$ in layer $\ell$ is active on $R$ and inactive on $Q$. 
Because $\theta$ is supertransversal, all other neurons have a fixed sign across $\sigma$ and its adjacent regions. 

On each region, 
the linear part of $f_\theta$ is obtained by multiplying the active weights: 
\[
A_R = \weights{L+1} D_{S_L} \cdots \weights{\ell+1} D_{S_\ell \cup \{i\}} \weights{\ell} \cdots D_{S_1} \weights{1},
\]
while $A_Q$ is identical except that the index $i$ is excluded from the layer-$\ell$ diagonal matrix. 
Across $\sigma$, only neuron $i$ changes its state, 
so the difference in the linear parts is 
\[
A_R - A_Q = \weights{L+1} D_{S_L} \cdots \weights{\ell+1} \diag(e_i) \weights{\ell} D_{S_{\ell-1}} \cdots \weights{1}.
\]

The normal vector of $\sigma$ pointing into $R$ is given by the linear part of the $i$-th preactivation in layer $\ell$, normalized to unit length: 
\[
e_{R/\sigma} = \frac{(\weights{1})^\top D_{S_1} \cdots (\weights{\ell-1})^\top D_{S_{\ell-1}} (\weights{\ell})^\top e_i}
{\left\| (\weights{1})^\top D_{S_1} \cdots (\weights{\ell-1})^\top D_{S_{\ell-1}} (\weights{\ell})^\top e_i \right\|_2}.
\]

By definition of the tropical weight (\Cref{lem:structure_theorem_tropical_geometry}),
\[
c_\theta(\sigma) = (A_R - A_Q) e_{R/\sigma}.
\]

We now expand this expression. 
Using the identity $\diag(e_i) \weights{\ell} v = e_i \, (e_i^\top \weights{\ell} v)$, we can factor out the scalar inner product: 
\[
(A_R - A_Q) e_{R/\sigma} 
= \weights{L+1} D_{S_L} \cdots \weights{\ell+1} e_i 
\;\underbrace{\left(e_i^\top \weights{\ell} D_{S_{\ell-1}} \cdots \weights{1} e_{R/\sigma}\right)}_{\text{scalar norm factor}}.
\]

The scalar factor is precisely the norm of the unnormalized linear part: 
\[
e_i^\top \weights{\ell} D_{S_{\ell-1}} \cdots \weights{1} e_{R/\sigma} 
= \left\| (\weights{1})^\top D_{S_1} \cdots (\weights{\ell-1})^\top D_{S_{\ell-1}} (\weights{\ell})^\top e_i \right\|_2.
\]
Substituting this identity yields the desired expression for $c_\theta(\sigma)$. 
\end{proof}

\begin{lemma}
\label{lem:generic_cancellation_free}
    Generic parameters are cancellation-free.
\end{lemma}

\begin{proof}
Let $\theta$ be generic (\Cref{def:generic_parameter}). 
Let $\sigma \in \complex_\theta^{d-1}$ be a facet, let $B_{\ell,i}$ be the unique bent hyperplane containing $\sigma$, and 
for each layer $k \in [L]$, 
let $S_k \subseteq [n_k]$ be the set of active neurons on the relative interior of~$\sigma$. 

By \Cref{prop:tropweightofNNs}, we have
\[
c_\theta(\sigma) = 
\left\| (\weights{1})^\top D_{S_1} \cdots (\weights{\ell-1})^\top D_{S_{\ell-1}} (\weights{\ell})^\top e_i \right\|
\;
\weights{L+1} D_{S_L} \cdots \weights{\ell+1} e_i.
\]

Since $\theta$ is generic, it satisfies the rank conditions of \Cref{def:generic_parameter}. 
In particular, the first factor (the norm) is zero if and only if there exists a layer $k < \ell$ with $S_k = \emptyset$. 
Similarly, the forward matrix product is the zero vector if and only if there exists a layer $k > \ell$ with $S_k = \emptyset$. 

It follows that $c_\theta(\sigma) = 0$ if and only if there exists a layer $k \in [L]$ such that $S_k = \emptyset$. 
This is equivalent to $\mathbf{s}_{k,j}(\sigma) = -$ for all $j \in [n_k]$, and hence the parameter is cancellation-free. 
\end{proof}

Thus, in the generic regime, codimension in the canonical complex matches the number of bent hyperplanes passing through a face, and the only way a facet can carry zero tropical weight is that all neurons of a later layer are inactive on the corresponding facet.
\subsection{Transparency}
We now introduce transparency, a simple visibility condition that prevents later layers from annihilating previously created breakpoints. Combined with genericity, transparency will imply the linear region assumption on the relevant domain.
\begin{lemma}
\label{lem:LRA_tropical_weight}
    The network $f_\theta$ satisfies LRA on a polyhedron $X \subseteq \R^d$ if and only if
    $c_{\theta}(\sigma) \neq 0$ for all facets $\sigma \in \complex_{\theta}^{d-1}$ with
    $\sigma \cap \relint(X) \neq \emptyset$.
\end{lemma}

\begin{proof}
By definition, $f_\theta$ satisfies LRA if and only if no two adjacent maximal polyhedra $P, Q \in \complex_{\theta}^d$ with $(P \cap Q) \cap \relint(X) \neq \emptyset$ share the same affine-linear function. Let $f_\theta|_P(x) = A_P x + b_P$ and $f_\theta|_Q(x) = A_Q x + b_Q$, where $A_P$ and $A_Q$ are the linear parts. Due to the continuity of $f_\theta$ across the facet $\sigma = P \cap Q$, these affine-linear functions are identical if and only if their linear parts are equal, $A_P = A_Q$.

According to Lemma~\ref{lem:structure_theorem_tropical_geometry}, the weight is $c_\theta(\sigma) = (A_P - A_Q)e_{P/\sigma}$. If $A_P = A_Q$, then $c_\theta(\sigma) = 0$. Conversely, for any CPWL function, the difference between the linear parts on adjacent polyhedra is restricted to the form $A_P - A_Q = v e_{P/\sigma}^\top$ for some $v \in \R^m$. Consequently, $(A_P - A_Q)e_{P/\sigma} = 0$ implies $A_P = A_Q$. 

Thus $c_\theta(\sigma)\neq 0$ for every facet $\sigma$ intersecting $\relint(X)$ if and only if the affine-linear map changes across every facet in the interior of $X$, which is precisely LRA on $X$.
\end{proof}

Following \cite{phuong2020functional}, we define transparent ReLU layers as follows. 

\begin{definition} 
\label{def:transparent}
	Let $X \subseteq \R^d$. 
    We call a ReLU layer $f_{W,b} \colon \R^d \to \R^m$ \emph{transparent} on $X$ if, for all $x \in X$, there exists an index $i \in [m]$ such that $W_ix + b_i \geq 0$. 
    In words, for every input in $X$, at least one neuron is active. 
\end{definition}

\begin{lemma}
\label{lem:transparent_implies_LRA_on_X}
    Let $\theta \in \gparamspace{\architecture}$ be generic, and let $X \subseteq \R^d$ be a polyhedron. Suppose that for every $\ell \in \{2,\ldots,L\}$, the layer $f_{\theta_\ell}$ is transparent on $\activation{\ell-1}(X)$. 
    Then $\theta$ satisfies LRA on $X$. 
\end{lemma}

\begin{proof}
Assume that $\theta$ does not satisfy LRA on $X$. Then, by Lemma~\ref{lem:LRA_tropical_weight}, there exists a facet $\sigma \in \complex_\theta^{d-1}$ intersecting $\relint(X)$ such that $c_\theta(\sigma) = 0$. 
Since $\theta$ is supertransversal (by \Cref{lem:generic_implies_supertransversal}), 
there is a unique bent hyperplane $B_{k,j}$ containing $\sigma$. 

Since $\theta$ is cancellation-free, there exists a layer $\ell$ with
$S_\ell(\sigma)=\emptyset$. Since $\sigma$ lies in the bent hyperplane $B_{k,j}$, no
layer $\ell\le k$ can be entirely inactive on $\sigma$; otherwise the preactivation
defining $B_{k,j}$ would be locally constant there. Hence $\ell>k$. 

This implies that $f_{\theta_\ell}$ is not transparent on $\activation{\ell-1}(X)$, completing the proof. 
\end{proof}
See
\Cref{fig:LRA-Transparceny-Supertransversal} for an illustration of how transparency implies LRA.
\begin{figure}[t]
\centering
\begin{subfigure}[t]{0.31\textwidth}
\centering
\begin{tikzpicture}[x=0.9cm,y=0.9cm,line cap=round,line join=round]
    \footnotesize
    \draw[black!80, line width=0.8pt] (-2,0) -- (0.75,0);
    \draw[black!80,dotted, line width=1pt] (0.75,0) -- (1.8,0);
    \draw[black!80, line width=0.8pt] (1.8,0) -- (2,0);
    \draw[black!80, line width=0.8pt] (0,-2) -- (0,2);

    \draw[red!70!black, line width=0.9pt] (-1.75,1.6) -- (0.0,1.6) -- (0.75,0) -- (0.75,-2.15);
    \draw[red!70!black, line width=0.9pt] (-1.75,0.95) -- (0.0,0.95) -- (1.8,0) -- (1.8,-2.15);

    \draw[->,>=stealth,red!70!black] (-1.2,1.58) -- (-1.2,1.18);
    \draw[->,>=stealth,red!70!black] (-0.6,0.98) -- (-0.6,1.35);

    \node at (1.2,-0.28) {$\sigma$};
    \draw[black!80,->] (2.15,1.15) -- (1.25,0.05);
    \node[anchor=south] at (2.2,1.18) {$c_\theta(\sigma)=0$};
\end{tikzpicture}
\caption{A generic parameter that does not satisfy LRA because the second layer is not transparent on $\sigma$. In particular, $\sigma$ would not be contained in the breakpoint complex. }
\label{fig:dependency-local-zero}
\end{subfigure}
\hfill
\begin{subfigure}[t]{0.31\textwidth}
\centering
\begin{tikzpicture}[x=0.9cm,y=0.9cm,line cap=round,line join=round]
   \footnotesize
    \draw[black!80, line width=0.8pt] (-2,0) -- (2,0);
    \draw[black!80, line width=0.8pt] (0,-2) -- (0,2);

    \draw[red!70!black, line width=0.9pt] (-1.75,1.6) -- (0.0,1.6) -- (1.8,0) -- (1.8,-2.15);
    \draw[red!70!black, line width=0.9pt] (-1.75,0.95) -- (0.0,0.95) -- (0.75,0) -- (0.75,-2.15);

    \draw[->,>=stealth,red!70!black] (-1.2,1.58) -- (-1.2,1.18);
    \draw[->,>=stealth,red!70!black] (-0.6,0.98) -- (-0.6,1.35);
\end{tikzpicture}
\caption{A generic parameter that is transparent and hence satisfies LRA. It also satisfies cTPIC.}
\label{fig:dependency-local-bending}
\end{subfigure}
\hfill
\begin{subfigure}[t]{0.31\textwidth}
\centering
\begin{tikzpicture}[x=0.9cm,y=0.9cm,line cap=round,line join=round]
    \footnotesize
    \draw[black!80, line width=0.8pt] (-2,0) -- (2,0);
    \draw[black!80, line width=0.8pt] (0,-2) -- (0,2);
    \draw[black!80, line width=0.8pt] (-2,1) -- (1,-2);

    \fill[black!80] (-1,0) circle (1.3pt);
    \node[above right] at (-0.9,0) {$\tau$};

    \draw[red!70!black, line width=0.9pt] (2,2) -- (-0,1.5) -- (-1,0) -- (-2,-0.5);

\end{tikzpicture}

\caption{A non-supertransversal situation: two first-layer hyperplanes meet at $\tau$, and a second-layer bent hyperplane passes through the same intersection.}
\label{fig:nonsupertransversal}
\end{subfigure}
\caption{Illustration of how transparency implies LRA, and how supertransversality can fail.} 
\label{fig:LRA-Transparceny-Supertransversal}
\end{figure}
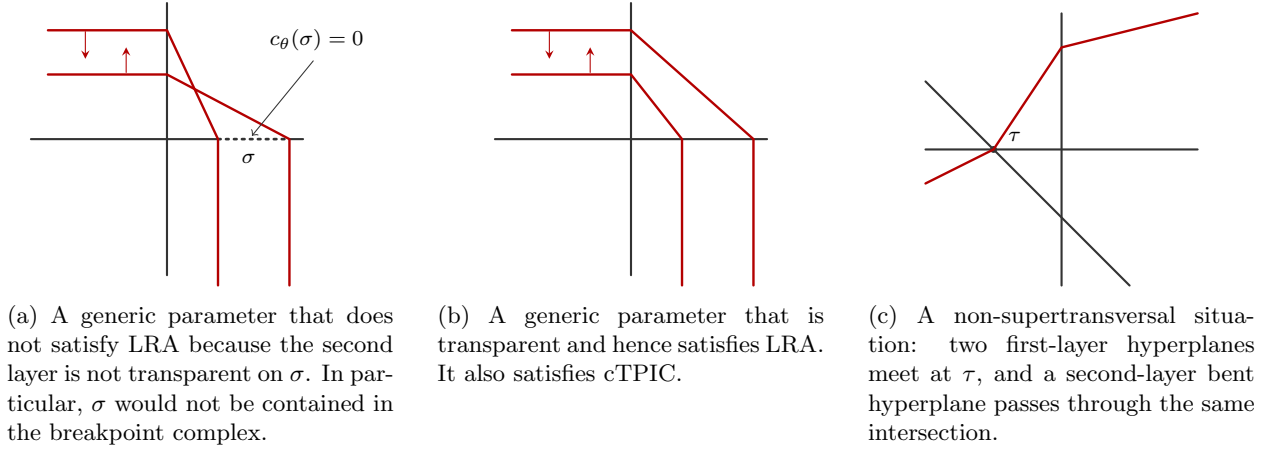 

\subsection{The Breakpoint Complex}
\label{sec:breakpoint_complex}

In the case that $\theta$ is not transparent, 
not every piece of every 
bent hyperplane is necessarily visible as a breakpoint of the final function. 
The set of breakpoints $\breakpoints{f_\theta}$ of the function forms the support of a polyhedral complex, $\bcomplex{\theta} \coloneqq \{P  \in \complex_{\theta} \mid P \subseteq \breakpoints{f_\theta}\} \subseteq \complex_\theta$, which we call the \emph{breakpoint complex}. 
While the support $\breakpoints{f_\theta}$ is fully determined by the function $f_\theta$, the complex $\bcomplex{\theta}$ might in general depend on the specific parameter $\theta$ that realizes the function. 
However, we will prove that for generic parameters, the breakpoint complex $\bcomplex{\theta}$ only depends on the function $f_\theta$. 

\begin{remark}
It follows immediately from \Cref{lem:LRA_tropical_weight} that \(
\bcomplex{\theta}^{d-1} = \{ P \in \complex_\theta^{d-1} \mid c_\theta(P) \neq 0\} 
\).
\end{remark}

We aim to use the local geometry of the breakpoint complex to access certain invariants of the neural network.  
Particularly important are the ridges, which lie in the pairwise intersections of bent hyperplanes. 
\begin{definition}
    Let $(\complex,c)$ be a weighted complex in $\R^d$, and let $\tau \in \complex^{d-2}$ be a ridge. 
    We say that a facet $\sigma \in \str(\tau)^{d-1}$ is \emph{non-bending at $\tau$} if there exists another facet $\sigma' \in \str(\tau)^{d-1}$ such that $e_{\sigma/\tau} = -e_{\sigma'/\tau}$ and $c(\sigma)=c(\sigma')$. Otherwise, we say that $\sigma$ is \emph{bending at $\tau$}. 
    We call the ridge $\tau$ \emph{bending} if there exists a facet $\sigma \in \str(\tau)^{d-1}$ that is bending at $\tau$, and \emph{non-bending} otherwise. 
\end{definition}

See \Cref{fig:breakpoint_complex} for an illustration of bending and non-bending ridges.
A one-hidden-layer ReLU network $f_\theta \colon \R^d \to \R^m$ is a CPWL function whose breakpoints form a hyperplane arrangement. In this case, the weight function is constant along these hyperplanes, as shown in the following proposition. 

\begin{prop}
\label{prop:tropical_weight_1layer}
    Let $\theta = (\weights{1},\bias{1},\weights{2},\bias{2})$ be the parameter of a one-hidden-layer network, and let $\sigma \in \complex_\theta^{d-1}$. 
    Then we have that \(c_{f_\theta}(\sigma) = \sum_{i \in I}  \weights{2}_{:,i} \|\weights{1}_i\|\), where $I =\{ i \in [n_1] \mid H_i = \aff(\sigma)\}$. 
\end{prop}

\begin{proof}

Let $P, Q \in \complex_\theta^d$ be the two $d$-dimensional polyhedra such that $P \cap Q = \sigma$, and let $e_{P/\sigma}$ be the unit normal vector to $\sigma$ pointing from $Q$ into $P$. The function $f_\theta$ is affine on $P$ and $Q$ with linear parts $A_P$ and $A_Q$, respectively. For a one-hidden-layer network, the linear part in any region $R$ is given by $A_R = \sum_{k \in S(R)} \weights{2}_{:, k} \weights{1}_{k}$, where $S(R)$ is the set of active neurons in $R$. 

As we cross from $Q$ into $P$ through the facet $\sigma$, the only neurons whose activation state can change are those whose preactivations $\preactivation{1}_i(x) = \langle \weights{1}_i, x \rangle + \bias{1}_i$ vanish on $\aff(\sigma)$. This is precisely the set $I$. Thus, the difference in the linear parts is:
\[
A_P - A_Q = \sum_{i \in I \cap S(P)} \weights{2}_{:,i} \weights{1}_i - \sum_{i \in I \cap S(Q)} \weights{2}_{:,i} \weights{1}_i.
\]
By definition, $c_{f_\theta}(\sigma) = (A_P - A_Q) e_{P/\sigma}$. For any $i \in I$, the vector $\weights{1}_i$ is normal to $\sigma$, so $\weights{1}_i = \lambda_i e_{P/\sigma}$ for some $\lambda_i \in \mathbb{R}$. Note that $\weights{1}_i e_{P/\sigma} = \lambda_i$ and hence $\|\weights{1}_i\|_2 = |\lambda_i|$. There are two cases for each $i \in I$:
\begin{enumerate}
    \item If $\lambda_i > 0$, the preactivation $\preactivation{1}_i$ increases in the direction of $e_{P/\sigma}$. Thus,  $i \in S(P)$ and $i \notin S(Q)$. The contribution to $(A_P - A_Q) e_{P/\sigma}$ is $\weights{2}_{:,i} \lambda_i = \weights{2}_{:,i} \|\weights{1}_i\|$.
    \item If $\lambda_i < 0$, the preactivation $\preactivation{1}_i$ decreases in the direction of $e_{P/\sigma}$. Thus, $i \notin S(P)$ and $i \in S(Q)$. The contribution to $(A_P - A_Q) e_{P/\sigma}$ is $-\weights{2}_{:,i} \lambda_i = \weights{2}_{:,i} |\lambda_i| = \weights{2}_{:,i} \|\weights{1}_i\|$.
\end{enumerate}
Summing over all $i \in I$, we obtain $c_{f_\theta}(\sigma) = \sum_{i \in I} \weights{2}_{:,i} \|\weights{1}_i\|$.
\end{proof}

See also \Cref{fig:hyperplane_arrangement} for an illustration.

\begin{lemma}
\label{lem:hyperplaneAimpliesnonbending}
    If $\tau \in \bcomplex{\theta}$ lies only on bent hyperplanes from a single layer $\ell \in [L]$, i.e., $\tau \nsubseteq B_{k,j}$ for all neurons $(k,j)$ with $k \neq \ell$, then $\tau$ is non-bending. 
\end{lemma}

\begin{proof}
Let $P \in \complex_{\theta,\ell-1}^d$ be the maximal polyhedron containing $\tau$. 
Then $\activation{\ell-1}$ is affine linear on $P$, and consequently the breakpoints of $\activation{\ell}$ on $P$ form a hyperplane arrangement. 
Thus, each facet in $\str(\tau)^{d-1}$ has an opposite facet. 
By \Cref{prop:tropical_weight_1layer}, opposite facets have the same weights, and hence $\tau$ is non-bending. 
\end{proof} 

The contraposition of \Cref{lem:hyperplaneAimpliesnonbending} implies that any bending ridge lies in the intersection of bent hyperplanes associated with neurons from different layers. 
The converse of \Cref{lem:hyperplaneAimpliesnonbending} does not hold in general. 

\begin{definition}
\label{def:honest}
    We call $\theta$ \emph{honest} if the converse of \Cref{lem:hyperplaneAimpliesnonbending} holds, i.e., if $\tau\in\bcomplex{\theta}$ is non-bending, then it lies only on bent hyperplanes from a single layer. 
\end{definition}
\begin{lemma}
\label{lem:generic_honest}
    Generic parameters are honest.
\end{lemma} 

\begin{proof}
Let $\theta$ be generic, and let $\tau \in \complex_\theta^{d-2}$ be a ridge. Since
$\theta$ is generic, $\tau$ is contained in precisely two bent hyperplanes. Suppose that
these two bent hyperplanes are $B_{k,i}$ and $B_{\ell,j}$, with $k<\ell$. We show that
$\tau$ is bending.

By supertransversality, there exists an open neighborhood $U$ of $\relint(\tau)$ such
that no other bent hyperplanes intersect $U$. The bent hyperplane $B_{k,i}$ subdivides
$U$ into two regions, $U^+$ and $U^-$, on which neuron $(k,i)$ is active and inactive,
respectively.

Let $a^+$ and $a^-$ be the local gradients of the preactivation
$\preactivation{\ell}_j$ on the two sides of $B_{k,i}$. Let $n$ be the local normal of
$B_{k,i}$. Since crossing $B_{k,i}$ only changes the activation state of neuron
$(k,i)$, the gradient jump of $\preactivation{\ell}_j$ across $B_{k,i}$ has the form
$
a^+ - a^- = \gamma n
$
for some scalar $\gamma$. By the rank condition in
\Cref{def:generic_parameter}, this scalar is nonzero.

We claim that $a^+$ and $a^-$ are not proportional. If $a^+=\lambda a^-$, then
$
\gamma n = (\lambda-1)a^-.
$
Since $\gamma\neq0$, this would imply that $n$ is proportional to $a^-$. But $n$ and
$a^-$ are the local normals of the two bent hyperplanes $B_{k,i}$ and $B_{\ell,j}$
meeting at the codimension-two ridge $\tau$. By genericity, these bent hyperplanes meet
transversely there, so their local normals are not proportional. This contradiction
shows that $a^+$ and $a^-$ are not proportional.

Thus the two facets of $B_{\ell,j}$ adjacent to $\tau$ do not have opposite relative
normal directions. Hence $\tau$ is bending. Therefore any non-bending ridge can only lie
on bent hyperplanes from a single layer, which is precisely honesty.
\end{proof}

Next, we prove the following for generic parameters.

\begin{lemma}
\label{lem:3or4facets}
Let $\theta$ be supertransversal and cancellation-free, and let $\tau \in \bcomplex{\theta}^{d-2}$. 
Then 
there are either three or four facets adjacent to $\tau$: 
$\#\str_{\bcomplex{\theta}}(\tau)^{d-1} \in \{3,4\}$. 
\end{lemma} 

\begin{proof}
    Since $\bcomplex{\theta}$ is a subcomplex of $\complex_\theta$, we have
$
\str_{\bcomplex{\theta}}(\tau)^{d-1}
\subseteq
\str_{\complex_\theta}(\tau)^{d-1}.
$
Because $\theta$ is supertransversal and $\tau$ has codimension $2$,
$\tau$ is contained in exactly two bent hyperplanes. 
The local arrangement of these two hyperplanes is combinatorially equivalent to the intersection of two transverse affine hyperplanes. 
Hence $\str_{\complex_\theta}(\tau)^{d-1}$ consists of exactly four facets, and therefore
$
\#\str_{\bcomplex{\theta}}(\tau)^{d-1} \le 4.
$

We show that at least three facets remain.
Suppose, for contradiction, that there exists a layer $\ell'$ such that $\mathbf{s}_{\ell',j}(\tau)=-$ for all $j \in [n_{\ell'}]$. 
Then $f_\theta$ is locally constant in a neighborhood of $\relint(\tau)$, contradicting the assumption that $\tau \in \bcomplex{\theta}$. 
Thus, in every layer $\ell'$, there exists at least one neuron $j \in [n_{\ell'}]$ such that $\mathbf{s}_{\ell',j}(\tau)\neq-$. 

Let $B_{\ell',i'}$ and $B_{\ell,i}$ be the two bent hyperplanes containing $\tau$ and assume without loss of generality that $\ell' \leq \ell$. 
Since $\theta$ is supertransversal, no bent hyperplanes from layers $k \notin \{\ell', \ell\}$ intersect $\relint(\tau)$. 
Consequently, for each layer $k \notin \{\ell', \ell\}$, there exists a neuron $j \in [n_{k}]$ such that $\mathbf{s}_{k,j}(\tau)=+$, and hence, for every facet
$\sigma \in \str_{\complex_\theta}(\tau)^{d-1}$, we also have that $\mathbf{s}_{k,j}(\sigma)=+$. 

Now, let $P,Q \in \complex_{\theta,\ell-1}^d$ be the regions such that $\relint(\tau) \subseteq \relint(P \cap Q) \subseteq B_{\ell',i'}$. 
Since $B_{\ell,i}$ intersects the interiors of both $P$ and $Q$, the preactivation $\preactivation{\ell}_i$ cannot be constant on either region. 
Consider layer $k=\ell'$. If the only neuron active on $P$ corresponds to the bent hyperplane $B_{\ell',i'}$ intersecting $\tau$, then, by supertransversality, there must exist another neuron that is active on $Q$ 
whose associated hyperplane 
cannot intersect $\tau$, and hence must be active on $\tau$. 
In any case, there exists a neuron $j \in [n_{\ell'}]$ that is active on $\tau$, and therefore active on all facets $\sigma \in \str(\tau)^{d-1}$.

Moreover, among the four facets in $\str_{\complex_\theta}(\tau)^{d-1}$, there is only one $\sigma \in \str_{\complex_\theta}(\tau)^{d-1}$ with $\mathbf{s}_{\ell,j}(\sigma) =-$. 
Therefore, at most one facet can be removed when passing to $\bcomplex{\theta}$, which
implies
$
\#\str_{\bcomplex{\theta}}(\tau)^{d-1} \in \{3,4\}.$
\end{proof}

To prove that the breakpoint complex only depends on the function itself and not on the specific parameter representing it, we combine \Cref{lem:3or4facets} with the following basic lemma. 

\begin{lemma}
\label{lem:uniquecoarsest} 
Let $\mathcal{B}$ be a polyhedral complex in $\R^d$ that is pure of codimension $1$. 
If for each ridge $\tau \in \mathcal{B}^{d-2}$ one has
$
\#\str_{\mathcal{B}}(\tau)^{d-1} \ge 3,
$
then $\mathcal{B}$ is the unique coarsest polyhedral complex on its support.
\end{lemma}

\begin{proof}
Let $\complex$ be a polyhedral complex with $|\complex|=|\mathcal{B}|$. 
We show that $\complex$ is a refinement of $\mathcal{B}$.

Assume for contradiction that some facet $P\in\complex^{d-1}$ is not contained in any facet of $\mathcal{B}$. 
Then the intersections
$
\{\sigma\cap P \mid \sigma\in\mathcal{B}^{d-1},\ \sigma\cap P\neq\emptyset\}
$
induce a nontrivial polyhedral subdivision of $P$. Hence this subdivision has an interior codimension-$1$ face $F\subset P$. By construction, there are two distinct facets $\sigma_1,\sigma_2\in\mathcal{B}^{d-1}$ such that
$
F \subseteq \sigma_1\cap\sigma_2.
$
Since $\mathcal{B}$ is pure of codimension $1$, the intersection $\sigma_1\cap\sigma_2$ is a common face of dimension at most $d-2$. Because $F$ has dimension $d-2$, it follows that
$
\tau:=\sigma_1\cap\sigma_2 \in \mathcal{B}^{d-2}
$ and $
F\subseteq\tau.
$

Now choose $x\in\relint(F)$. Since $F$ is a face of the induced subdivision of $P$, exactly two cells of that subdivision meet along $F$ near $x$, namely the traces of $\sigma_1$ and $\sigma_2$. Therefore, among all facets of $\mathcal{B}$ containing $\tau$, only $\sigma_1$ and $\sigma_2$ can meet a sufficiently small neighborhood of $x$ inside $P$. This implies
$
\#\str_{\mathcal{B}}(\tau)^{d-1}=2,
$
contradicting the hypothesis that every ridge of $\mathcal{B}$ is contained in at least three facets.

Hence every facet of $\complex$ is contained in a facet of $\mathcal{B}$, so $\complex$ is a refinement of $\mathcal{B}$. Thus $\mathcal{B}$ is coarsest. Uniqueness follows because two coarsest complexes on the same support must refine each other.
\end{proof}

\begin{prop}
\label{prop:canonicalbcomplex}
If $\theta,\eta$ are cancellation-free and supertransversal, and satisfy $f_\theta =f_\eta$, 
then $\bcomplex{\theta} = \bcomplex{\eta}$. 
\end{prop}

\begin{proof}
    Since $f_\eta=f_\theta$, the support of $\bcomplex{\theta}$ equals the support of $\bcomplex{\eta}$. By \Cref{lem:3or4facets} and \Cref{lem:uniquecoarsest}, the complexes must be the same.
\end{proof}
Having this at hand, we prove that LRA is a property of the function for generic parameters.
\begin{lemma}
\label{lem:LRA_function_property_generic}
Let $\theta,\eta\in \widetilde{\Theta}_{\architecture}$ be generic parameters with
$f_\eta=f_\theta$, and let $P\subseteq \R^d$ be a full-dimensional polyhedron.
If $\theta$ satisfies LRA on $P$, then $\eta$ also satisfies LRA on $P$.
\end{lemma}
\begin{proof}
Since $\theta$ and $\eta$ are generic, they are supertransversal and cancellation-free by
\Cref{lem:generic_implies_supertransversal,lem:generic_cancellation_free}. Since
$f_\eta=f_\theta$, their breakpoint complexes coincide by \Cref{prop:canonicalbcomplex}:
$
\bcomplex{\eta}=\bcomplex{\theta}.
$

Assume that $\theta$ satisfies LRA on $P$. Then every codimension-one face of the
canonical complex of $\theta$ meeting $\relint(P)$ is visible, and hence every ridge
$\tau\in \bcomplex{\theta}^{d-2}$ with $\tau\cap\relint(P)\neq\emptyset$ is adjacent to
all four local facets of the canonical complex. By \Cref{lem:3or4facets}, the only
possible numbers of adjacent breakpoint facets are $3$ and $4$, so for $\theta$ every such
ridge has exactly $4$ adjacent facets in $\bcomplex{\theta}$.

Since $\bcomplex{\eta}=\bcomplex{\theta}$, the same is true for $\eta$. Therefore no
codimension-one face of the canonical complex of $\eta$ meeting $\relint(P)$ is lost when
passing to the breakpoint complex, and hence $\eta$ satisfies LRA on $P$.
\end{proof}

\begin{figure}[t]
\centering
\begin{subfigure}[t]{0.31\textwidth}
\centering
\begin{tikzpicture}[x=0.9cm,y=0.9cm,line cap=round,line join=round]
    \footnotesize

    \draw[black!80,dotted, line width=0.8pt] (-2,0) -- (1.8,0);
    \draw[black!80, line width=0.8pt] (1.8,0) -- (2.3,0);
    \draw[black!80, line width=0.8pt,dotted] (0,-2) -- (0,0.95);
    \draw[black!80, line width=0.8pt] (0,0.95) -- (0,2)    node[above] {$\sigma_3$};
    \draw[red!70!black, line width=0.9pt] (-1.75,0.95) -- (0.0,0.95) -- (1.8,0) -- (1.8,-2.15);

    \draw[->,>=stealth,red!70!black] (1.8,-1) -- (2.2,-1);
     \node[above right] at (0.1,0.1) {$0$};
     \node[above right] at (-1,0.1) {$0$};
       \node[above right] at (0.3,-1) {$0$};
          \node[above right] at (0,0.95) {$\tau$};
           \fill[black!80] (0,0.95) circle (1.3pt);
          \node[above right] at (-1,1) {$\sigma_1$};
            \node[above right] at (0.6,0.5) {$\sigma_2$};
\end{tikzpicture}
\caption{$\tau$ is a bending ridge with $3$ adjacent facets. Since the function is constant on regions adjacent to $\sigma_1$ and $\sigma_2$, they must lie in the bent hyperplane from the 
later 
layer.} 
\label{fig:3facets}
\end{subfigure}
\hfill
\begin{subfigure}[t]{0.31\textwidth}
\centering
\begin{tikzpicture}[x=0.9cm,y=0.9cm,line cap=round,line join=round]
    \footnotesize
    \draw[black!80, line width=0.8pt] (-2,0) -- (0.75,0);
    \draw[black!80,dotted, line width=0.8pt] (0.75,0) -- (1.75,0);
    \draw[black!80, line width=0.8pt] (1.75,0) -- (2,0);
    \draw[black!80, line width=0.8pt] (0,-2) -- (0,2);

    \draw[red!70!black, line width=0.9pt] (-1.75,1.75) -- (0.0,1.75) -- (0.75,0) -- (0.75,-2.15);
    \draw[red!70!black, line width=0.9pt] (-1.75,0.75) -- (0.0,0.75) -- (1.75,0) -- (1.75,-2.15);

    \draw[->,>=stealth,red!70!black] (-1.2,1.75) -- (-1.2,1.45);
    \draw[->,>=stealth,red!70!black] (-0.6,0.75) -- (-0.6,1.05);

    \node at (-0.3,-0.3) {$\tau_1$};
    \fill[black!80] (0,0) circle (1.3pt);
    \fill[black!80] (0.52,0.52) circle (1.3pt);
    \node at (0.7,0.7) {$\tau_2$};
    \fill[black!80] (0,1.75) circle (1.3pt);
    \node at (0.3,1.75) {$\tau_3$};
     \fill[black!80] (1.75,0) circle (1.3pt);
     \node at (1.8,0.3) {$\tau_4$};
\end{tikzpicture}
\caption{$\tau_1$ and $\tau_2$ are non-bending ridges, $\tau_3$ is a bending ridge with $4$ adjacent facets, and $\tau_4$ is a bending ridge with $3$ adjacent facets.} 
\label{fig:bendings}
\end{subfigure}
\hfill
\begin{subfigure}[t]{0.31\textwidth}
\centering
\tiny
\begin{tikzpicture}[scale=0.9,x=0.9cm,y=0.9cm,line cap=round,line join=round]
\scriptsize
    \draw[black!80, line width=0.8pt] (-2,0) -- (2,0) node[above]{$\|\weights{1}_1\|_2\weights{2}_{:1}$};
    \draw[black!80, line width=0.8pt] (0,-2) -- (0,2) node[above]{$\|\weights{1}_2\|_2\weights{2}_{:2}$};;
    \draw[black!80, line width=0.8pt] (-2,1) -- (1,-2) node[right]{$\|\weights{1}_3\|_2\weights{2}_{:3}$};;

\end{tikzpicture}

\caption{A hyperplane arrangement of a one-hidden-layer network. 
The weights on the facets are constant along each hyperplane.} 
\label{fig:hyperplane_arrangement}
\end{subfigure}
\caption{Illustration of bending and non-bending ridges.} 
\label{fig:breakpoint_complex}
\end{figure}

\subsection{The Dependency Graph}
\label{sec:dependency_graph}

We now show how the local geometry of the breakpoint complex encodes layer precedence for
generic parameters. The key observation is that bending ridges distinguish earlier from
later bent hyperplanes, and therefore allow us to organize visible nonlinearities into a
directed dependency graph.

The following lemma demonstrates that, at a bending ridge, we can distinguish which adjacent facets originate from the earlier layer and which from the later one. See also \Cref{fig:bendings} for an illustration. 

\begin{lemma}
\label{lem:bendingoptions}
  Let $\theta$ be generic. 
  Let $\tau \in \bcomplex{\theta}^{d-2}$ be a bending ridge, and denote the unique bent hyperplanes containing $\tau$ as $H \in \hyperplanes{k}{\theta}$, $B \in \hyperplanes{\ell}{\theta}$ with $k < \ell$. 
  Let $R \in \complex_{\theta,k-1}^d$ be the region such that $\relint(\tau) \subseteq \relint(R)$. Then:   
    \begin{enumerate}
        \item If $\#\str_{\bcomplex{\theta}}(\tau)^{d-1} =4$, then there are exactly two facets $\sigma_1,\sigma_2 \in \str_{\bcomplex{\theta}}(\tau)^{d-1}$ such that $e_{\sigma_1/\tau}=-e_{\sigma_2/\tau}$. 
        In this case, $H \cap R= \aff(\sigma_1) \cap R$. 

        \item If $\#\str_{\bcomplex{\theta}}(\tau)^{d-1} =3$, then there is a unique facet $\sigma \in \str_{\bcomplex{\theta}}(\tau)^{d-1}$ that is not adjacent to a region $P \in \str_{\complex_\theta}(\tau)^{d}$ on which $f_\theta{|_P}$ is constant. 
        In this case, $H \cap R= \aff(\sigma) \cap R$. 
    \end{enumerate}
\end{lemma}

\begin{proof}

\emph{Case 1: $\#\str_{\bcomplex{\theta}}(\tau)^{d-1} = 4$.} 
In this case, all four facets in $\str_{\complex_\theta}(\tau)^{d-1}$ are contained in $\bcomplex{\theta}$. 
Exactly two of these facets are contained in $H$. 
Since $H \cap R$ is the zero-locus of an affine-linear function, both facets have the same affine span. 
Thus there exist exactly two facets $\sigma_1,\sigma_2 \in \str_{\bcomplex{\theta}}(\tau)^{d-1}$ such that 
$\aff(\sigma_1)=\aff(\sigma_2)$ (and hence $e_{\sigma_1/\tau}=-e_{\sigma_2/\tau}$). For these facets, we have $H \cap R = \aff(\sigma_1) \cap R$. 
The remaining 
two facets cannot have the same affine span, because $\tau$ is bending. 

\medskip
\noindent
\emph{Case 2: $\#\str_{\bcomplex{\theta}}(\tau)^{d-1} = 3$.}
In this case, exactly one of the two facets contained in $H$ is missing from $\str_{\bcomplex{\theta}}(\tau)^{d-1}$. 
Let $\sigma_1$ denote the missing facet and $\sigma_2$ the remaining facet with $\sigma_2 \subseteq H$. 

By cancellation-freeness, $\sigma_1 \notin \bcomplex{\theta}$ implies that there exists a layer in which all neurons are inactive on $\sigma_1$. 
Hence $f_\theta$ is constant on $\sigma_1$ and on the two adjacent regions $P_1,Q_1 \in \str_{\complex_\theta}(\tau)^d$. 
These two regions are also adjacent to the two facets in $\str_{\complex_\theta}(\tau)^{d-1}$ that are contained in $B$. 
Let $P_2,Q_2 \in \str_{\complex_\theta}(\tau)^d$ be the two regions adjacent to $\sigma_2$, 
and denote by $\sigma_P = P_1 \cap P_2$ and $\sigma_Q = Q_1 \cap Q_2$ the corresponding facets contained in $B$. 
If $f_\theta$ were constant on either $P_2$ or $Q_2$, then $\sigma_P$ or $\sigma_Q$ would also be 
absent from $\bcomplex{\theta}$, contradicting the
assumption that exactly one facet is missing. 
Therefore, $f_\theta$ is non-constant on both regions adjacent to $\sigma_2$. 

It follows that among the three facets in
$\str_{\bcomplex{\theta}}(\tau)^{d-1}$ there is exactly one facet that is not adjacent to any region on which $f_\theta$ is constant. 
This facet must be $\sigma_2$, and since $\sigma_2 \subseteq H$,
we obtain 
$H \cap R = \aff(\sigma_2) \cap R$. See also \Cref{fig:3facets} for an illustration. 
\end{proof}
By leveraging the geometric asymmetry established in \Cref{lem:bendingoptions}, we reconstruct a directed graph from the breakpoints of $f_\theta$ that reflects the relative layer order of the associated neurons. 

\begin{definition}[Candidate bent hyperplanes]
Let $\theta$ be generic and let $\mathcal B_\theta$ be the breakpoint complex. 
We define an equivalence relation on the facets $\mathcal B_\theta^{d-1}$ as follows.

Let $\tau\in\mathcal B_\theta^{d-2}$ be a ridge and let
$\sigma,\sigma'\in\operatorname{star}_{\mathcal B_\theta}(\tau)^{d-1}$.

\begin{itemize}
    \item If $\tau$ is non-bending, we declare $\sigma\sim\sigma'$ whenever
    $
    e_{\sigma/\tau}=-e_{\sigma'/\tau}
  $ and $
    c_\theta(\sigma)=c_\theta(\sigma').
    $
    Thus opposite facets with the same weight are continuations of the same candidate bent hyperplane.

    \item If $\tau$ is bending, we declare $\sigma\sim\sigma'$ whenever the local
    description of Lemma~\ref{lem:bendingoptions} identifies them as belonging to the
    same one of the two bent hyperplanes through $\tau$.
\end{itemize}

A \emph{candidate bent hyperplane} is an equivalence class of facets under the transitive closure of this relation. The set of candidate bent hyperplanes is denoted by
$
V_{f_\theta}.
$
\end{definition}
\begin{definition}
We define the \emph{dependency graph} $(V_f,E_f)$ as follows. 
The vertex set $V_f$ consists of the candidate bent hyperplanes. 
The edge set $E_f$ consists of all pairs $(v,w)\in V_f\times V_f$ for which there exists a bending ridge $\tau$ where $w$ ``bends'' along $v$ as in \Cref{lem:bendingoptions}. 
\end{definition} 

\begin{remark}
    The definition of the dependency graph is a generalization of the definition given in the work of \cite{phuong2020functional} for transparent networks. 
    \end{remark}
    
    The following proposition establishes that any generic parameter realizing $f_\theta$ must respect the partial order induced by this graph. 
    
    \begin{prop}
    \label{prop:embedd_candidate_bh} 
    Let $\theta$ be a generic parameter of some architecture, and let $(V_{f_\theta}, E_{f_\theta})$ be the corresponding dependency graph. 
    Let $\architecture=(n_0,\ldots,n_L,m)$ be a (potentially different) architecture and $V_\architecture = \{(i,\ell) \mid \ell \in [L],\, i \in [n_\ell]\}$ be the set of hidden neurons. 
        Then, for a generic parameter $\eta \in \gparamspace{\architecture}$ with $f_\eta = f_\theta$, there is a map $\phi_\eta \colon V_{f_\theta} \to V_\architecture$ such that for every $(u,v) \in E_{f_\theta}$, 
        if 
        $(i,\ell)=\phi_\eta(u)$ and $(j,k)=\phi_\eta(v)$, 
        then 
        $\ell < k$. 
    \end{prop} 
    
\begin{proof}
Since $\theta$ and $\eta$ are generic, it is supertransversal and cancellation-free by 
\Cref{lem:generic_implies_supertransversal,lem:generic_cancellation_free}, 
and honest by \Cref{lem:generic_honest}. 
Moreover, since $f_\eta=f_\theta$, the breakpoint complex is determined by the function, and hence
$
\bcomplex{\eta}=\bcomplex{\theta}
$
by \Cref{prop:canonicalbcomplex}. 
Thus the local weighted geometry of the breakpoint complex is the same for the realizations $\theta$ and $\eta$.

We first define the map $\phi_\eta$. 
Let $u \in V_{f_\theta}$ be a candidate bent hyperplane, i.e., an equivalence class of facets in $\bcomplex{\theta}^{d-1}$ under the transitive closure of direct equivalence. 
Choose any facet $\sigma \in u$. Since $\sigma$ is a facet of the breakpoint complex of $f_\eta$, it is contained in the bent hyperplane of some neuron of the realization $\eta$. 
We define $\phi_\eta(u)$ to be such a neuron.

We must check that this is well defined on the whole class $u$. 
Suppose $\sigma,\sigma' \in u$ are directly equivalent. Then, by definition, there exists a ridge $\tau \in \bcomplex{\theta}^{d-2}$ with $\sigma,\sigma' \in \str(\tau)$ such that $\sigma$ and $\sigma'$ come from the same layer according to \Cref{lem:bendingoptions}. 
Because the weighted breakpoint complex depends only on the realized function, the same local configuration occurs for $\eta$. Hence the two facets $\sigma$ and $\sigma'$ must also belong to the same local bent-hyperplane piece in the realization $\eta$. 
Passing to the transitive closure, all facets in the class $u$ are assigned consistently, and therefore $\phi_\eta$ is well defined.

Now let $(u,v)\in E_{f_\theta}$. By definition of the dependency graph, there exists a bending ridge 
$
\tau \in \bcomplex{\theta}^{d-2}
$
such that the candidate bent hyperplane $v$ bends along $u$. 
Let
$
\phi_\eta(u)=(i,\ell), 
\phi_\eta(v)=(j,k).
$
We claim that $\ell<k$.

Since $\theta$ is honest, the ridge $\tau$ lies in bent hyperplanes coming from two different layers. 
By \Cref{lem:bendingoptions}, the local geometry at $\tau$ determines which adjacent facets belong to the earlier bent hyperplane and which belong to the later one. 
Because $\bcomplex{\eta}=\bcomplex{\theta}$ as weighted complexes, the same distinction must hold for the realization $\eta$. 
Thus the neuron of $\eta$ corresponding to the class $u$ must lie in a strictly earlier layer than the neuron corresponding to the class $v$. Hence $\ell<k$.

Therefore, for every edge $(u,v)\in E_{f_\theta}$, the map $\phi_\eta$ sends $u$ and $v$ to neurons whose layer indices satisfy
$
\ell<k.
$
This proves the claim.
\end{proof}

\subsection{A Sufficient Condition for Identifiability Among Generic Parameters}
We say that a parameter $\theta \in \gparamspace{\architecture}$ is \emph{identifiable among generic parameters} if $\eta \in \gparamspace{\architecture}$ with $f_\eta = f_\theta$ implies $\eta \sim \theta$.
In this subsection, we derive a sufficient criterion for identifiability among generic
parameters. The argument has three steps. First, we isolate a geometric regime in which the
reverse-engineering procedure of \cite{rolnick2020reverse} applies. Second, we show with the breakpoint complex and the dependency graph that
the relevant visibility and transversality conditions are properties of function on the generic locus. Third, we use the rank conditions built into
our notion of genericity to remove a remaining sign ambiguity in the last hidden layer resulting from the reverse-engineering procedure.

\begin{definition} 
\label{def:intersect transversely}
	Let $P \subseteq \R^d$ be a polyhedron, and let $B_1,B_2 \subseteq \R^d$ be two piecewise linear hypersurfaces. 
    We say $B_1$ and $B_2$ \emph{intersect transversely} in $P$ if $B_{1} \cap B_{2} \cap \relint(P) \neq \emptyset$ and $B_{1} \cap B_{2} \cap P $ is pure of dimension $\dim(P) -2$. 
A ReLU network $f \colon \R^d \to \R^m$ satisfies \emph{TPIC} (Transverse Pairwise Intersection Condition) on $P$ if, for every pair of adjacent neurons $(j,\ell-1)$ and $(i,\ell)$, the corresponding bent hyperplanes intersect transversely in $P$. 
\end{definition} 
The condition TPIC is the natural parameter-level transversality condition used in the
literature. However, for our purposes it is useful to strengthen it to the following cTPIC, which keeps
track of connected visible pieces of the bent hyperplanes on the working polytope.
\begin{definition}
\label{def:connected_tpic}
Let $\theta\in\Theta_{\architecture}$ and let $P\subseteq \R^d$ be a polyhedron. We say
that $f_\theta$ satisfies the \emph{connected transverse pairwise intersection condition}
(cTPIC) on $P$ if, for every hidden neuron $(j,\ell)$, there exists a connected piecewise linear hypersurface
$
C_{j,\ell}\subseteq B_{j,\ell}(\theta)
$
 such that for every adjacent-layer pair of neurons
$(j,\ell)$ and $(i,\ell+1)$, the sets $C_{j,\ell}$ and $C_{i,\ell+1}$ intersect
transversely in $P$.
\end{definition}

See
\Cref{fig:LRA-Transparceny-Supertransversal} 
examples of networks that satisfy cTPIC on the entire domain. The strengthening from TPIC to cTPIC is tailored to our later use of the dependency graph:
the connected pieces $C_{j,\ell}$ allow us to identify, from the breakpoint complex, the
candidate bent hyperplanes corresponding to individual neurons. Before turning to this
transfer mechanism, we recall the reverse-engineering theorem on the restricted class of
parameters where it applies.

To isolate the geometric regime in which the reverse-engineering algorithm of \cite{rolnick2020reverse} applies, we bundle
these assumptions into the following class of parameters.
Let
$
\Theta^{\mathrm{good}}_\architecture
$ be the set of parameters in $\widetilde{\Theta}$ that satisfy cTPIC and LRA on some full-dimensional polytope.
The reverse-engineering argument of \citet{rolnick2020reverse} should be interpreted as a reconstruction theorem inside the class $\good$. In particular, it yields uniqueness only among parameters in this restricted class, not yet among all generic parameters or all parameters; see \Cref{app:restricted-identifiability}. The resulting reconstruction is unique up to permutation and
positive scaling in the first $L-1$ hidden layers, with a remaining neuronwise sign
ambiguity in the final hidden layer.
\begin{theorem}[Reverse-engineering up to sign, \citealp{rolnick2020reverse}]
\label{thm:tpic_generic_identifiable_inductive}
Let $\theta,\eta \in \Theta^{\mathrm{good}}_{\architecture}$ satisfy
$
f_\eta=f_\theta.
$
Then, up to permutation and positive scaling, the first $L-1$ hidden layers of $\eta$
and $\theta$ coincide, and the $L$-th hidden layer coincides up to permutation, positive
scaling, and neuronwise sign.
\end{theorem}

Before addressing this remaining sign ambiguity, we explain how our breakpoint-complex
framework transfers the properties cTPIC and LRA across generic realizations of the same
function. This is the step that upgrades the reverse-engineering theorem from a uniqueness
statement inside the restricted class $\Theta^{\mathrm{good}}_\architecture$ to a statement
about all generic realizations.
\begin{lemma}
\label{lem:cTPIC_implies_full_dep_subgraph}
Let $\theta\in \widetilde{\Theta}_{\architecture}$ be generic and suppose that
$f_\theta$ satisfies cTPIC on a full-dimensional polytope $P$. Then the dependency graph
$(V_{f_\theta},E_{f_\theta})$ contains a layered subgraph whose vertices correspond to the
hidden neurons of $\architecture$, and in which every pair of adjacent layers is fully
connected.
\end{lemma}
\begin{proof}
For each hidden neuron $(j,\ell)$, cTPIC provides a connected piecewise linear hypersurface
$
C_{j,\ell}\subseteq B_{j,\ell}(\theta)
$
of codimension one such that the ones from adjacent neurons intersect transversely in $P$. Since $\theta$ is generic, every bending ridge of the breakpoint
complex lies in exactly two bent hyperplanes, and by \Cref{lem:bendingoptions} the local
weighted geometry determines which adjacent facets belong to the earlier and later bent
hyperplanes. Because each $C_{j,\ell}$ is connected, all its facets belong to a single
candidate bent hyperplane, and hence define a vertex of the dependency graph.

Now let $(j,\ell)$ and $(i,\ell+1)$ be neurons in adjacent layers. By cTPIC, the sets
$C_{j,\ell}$ and $C_{i,\ell+1}$ intersect transversely in $P$. Their intersection contains
a bending ridge of the breakpoint complex, and by construction this ridge yields an edge in
the dependency graph from the candidate bent hyperplane corresponding to $(j,\ell)$ to the
one corresponding to $(i,\ell+1)$. Since this holds for every adjacent-layer pair, the
resulting vertices form a layered subgraph with all adjacent layers fully connected.
\end{proof}
The first lemma shows that cTPIC produces the layered combinatorics required by the dependency graph. Conversely, the next lemma shows that this layered subgraph structure is sufficient to recover cTPIC on the generic locus.
\begin{lemma}
\label{lem:dep_subgraph_implies_TPIC}
Let $\theta\in \widetilde{\Theta}_{\architecture}$ be generic. If the dependency graph
$(V_{f_\theta},E_{f_\theta})$ contains a layered subgraph corresponding to the hidden
architecture, with all adjacent layers fully connected, then $f_\theta$ satisfies cTPIC.
\end{lemma}
\begin{proof}
Let $(j,\ell)$ and $(i,\ell+1)$ be neurons in adjacent layers. By assumption, the
dependency graph contains an edge between the corresponding vertices. By definition of the
dependency graph, this edge arises from a bending ridge $\tau$ of the breakpoint complex at
which the later candidate bent hyperplane bends along the earlier one. Since $\theta$ is
generic, \Cref{lem:bendingoptions} implies that $\tau$ lies in the intersection of the two
corresponding bent hyperplanes and that this intersection has codimension $2$. Hence the two bent hyperplanes intersect
transversely. For each vertex of the layered subgraph, take the union of the facets belonging to the corresponding candidate bent hyperplane. By construction this union is connected, and the preceding argument shows that adjacent-layer representatives intersect transversely. Hence these sets witness cTPIC.
\end{proof}

The example in the work of \citep{ramakrishnan2026completesymmetryclassificationshallow} shows that the remaining ambiguity in \Cref{thm:tpic_generic_identifiable_inductive} is genuine: even under cTPIC and LRA, the last hidden layer need not be identifiable from the bent hyperplanes alone. In fact, the phenomenon already appears in the simplest setting of a network with a single hidden layer, where no deeper hidden layer is available to resolve the orientation of the breakpoint hyperplanes. In that case, the ambiguity can be interpreted as arising from a cancellation among the rank-one matrices
$
\weights{2}_{:,j}\weights{1}_{j,:}.
$
More generally, in our framework such examples arise from a nontrivial cancellation in the output contribution of the final hidden layer: a sign flip of a nonempty subset $J \subseteq [n_L]$ can remain invisible in the realized function only if the corresponding masked products
$
\weights{L+1} D_J \weights{L} D_{S_{L-1}(R)} \weights{L-1} \cdots D_{S_1(R)} \weights{1}
$
vanish on the relevant full-dimensional regions $R \in \complex_{\theta,L-1}$.

The definition of genericity excludes precisely this type of cancellation. On the other hand, \Cref{lem:LRA_function_property_generic,lem:cTPIC_implies_full_dep_subgraph,lem:dep_subgraph_implies_TPIC} show that the assumptions needed for reverse engineering are functionally determined on the generic locus. Combining these two ingredients upgrades \Cref{thm:tpic_generic_identifiable_inductive} from a reconstruction statement inside $\Theta^{\mathrm{good}}_\architecture$ to identifiability among generic parameters.
\begin{theorem}
\label{thm:tpic_generic_identifiable}
Let $\theta \in \gparamspace{\architecture}$ be generic. 
If $\theta$ satisfies cTPIC and LRA on a full-dimensional polyhedron $P \subseteq \R^d$, then $\theta$ is identifiable among generic parameters.
\end{theorem}

\begin{proof}
Let $\eta\in \widetilde{\Theta}_{\architecture}$ satisfy
$
f_\eta=f_\theta.
$
Since $\theta$ is generic and satisfies cTPIC and LRA, we have
$
\theta\in \Theta^{\mathrm{good}}_{\architecture}.
$

By \Cref{lem:LRA_function_property_generic}, the parameter $\eta$ also satisfies LRA.
Moreover, by the functional invariance of the dependency graph on the generic locus and
\Cref{lem:cTPIC_implies_full_dep_subgraph,lem:dep_subgraph_implies_TPIC}, the parameter
$\eta$ also satisfies cTPIC. Therefore
$
\eta\in \Theta^{\mathrm{good}}_{\architecture}.
$

By \Cref{thm:tpic_generic_identifiable_inductive}, after applying permutation and positive scaling symmetries to $\eta$, we may assume that
$
(\weights{\ell,\eta},\bias{\ell,\eta})=(\weights{\ell,\theta},\bias{\ell,\theta})
$ for all $\ell\in[L-1],
$
and that for each $j\in[n_L]$ there is a sign $s_j\in\{\pm1\}$ such that
$
\preactivation{L,\eta}_j=s_j\,\preactivation{L,\theta}_j
$
.

We first determine the output layer. 
Let $\sigma\in \complex_{\theta}^{d-1}(P)$ be a facet contained in the bent hyperplane of a neuron $(j,L)$ in the last hidden layer. 
Since $\theta$ is generic, it is supertransversal and cancellation-free by \Cref{lem:generic_implies_supertransversal,lem:generic_cancellation_free}. 
Hence \Cref{prop:tropweightofNNs} applies and gives
$
c_\theta(\sigma)
=
\left\|
(\weights{1,\theta})^\top D_{S_1(\sigma)} \cdots (\weights{L,\theta})^\top e_j
\right\|_2
\;
\weights{L+1,\theta}e_j.
$
The reverse-engineering procedure determines the local hyperplane of the bent hyperplane containing $\sigma$, and therefore determines the norm
$
\left\|
(\weights{1,\theta})^\top D_{S_1(\sigma)} \cdots (\weights{L,\theta})^\top e_j
\right\|_2
$
from the function. Since $c_\theta(\sigma)$ is also determined by the function, the column $\weights{L+1,\theta}e_j$ is uniquely determined. Repeating this for all neurons in the last hidden layer recovers $\weights{L+1,\theta}$ uniquely. The output bias $\bias{L+1,\theta}$ is then determined by evaluating the affine-linear piece of the function on any full-dimensional region. Thus, after replacing $\eta$ by an equivalent parameter if necessary, we may assume
$
(\weights{L+1,\eta},\bias{L+1,\eta})=(\weights{L+1,\theta},\bias{L+1,\theta}).
$

It remains to show that $s_j=+1$ for all $j\in[n_L]$. Suppose not, and let
$
J \coloneqq \{j\in[n_L]\mid s_j=-1\}.
$
Then $J\neq\emptyset$.
Choose $j\in J$. Since the bent hyperplane $\bhyperplane{j,L}(\theta)$ is visible since $\theta$ satisfies LRA, there exists a full-dimensional polyhedron
$
R\in \complex^d_{\theta,L-1}(P)
$
such that
$
\bhyperplane{j,L}(\theta)\cap \relint(R)\neq\emptyset.
$
Equivalently, the restricted preactivation $\preactivation{L,\theta}_j|_R$ is non-constant. 
We claim that this implies
$
S_k(R)\neq\emptyset
$ for all $k\in[L-1].
$
Indeed, if $S_k(R)=\emptyset$ for some $k\in[L-1]$, then $\activation{k,\theta}$ is identically zero on $R$. 
Consequently,
$
\preactivation{k+1,\theta}
=
\weights{k+1,\theta}\activation{k,\theta}+\bias{k+1,\theta}
=
\bias{k+1,\theta}
$
is constant on $R$, and therefore $\activation{k+1,\theta}$ is constant on $R$ as well. 
Iterating this argument shows that $\preactivation{\ell',\theta}$ is constant on $R$ for all $\ell'\ge k+1$, in particular that $\preactivation{L,\theta}_j|_R$ is constant, a contradiction.

On $R$, the linear part of the last hidden preactivation is
$
A_R
\coloneqq
\weights{L,\theta}D_{S_{L-1}(R)}\weights{L-1,\theta}\cdots D_{S_1(R)}\weights{1,\theta}.
$
For each $j\in J$, the contribution of the $j$-th neuron in the last hidden layer changes from
$
\weights{L+1,\theta}_{:,j}[\preactivation{L,\theta}_j]_+
$
to
$
\weights{L+1,\theta}_{:,j}[-\preactivation{L,\theta}_j]_+.
$
Using the identity
$
[-u]_+=[u]_+-u,
$
the difference in the contribution of neuron $j$ is
$
-\weights{L+1,\theta}_{:,j}\preactivation{L,\theta}_j.
$
Hence the difference $f_\eta-f_\theta$ on $R$ is affine-linear with linear part
$
-\weights{L+1,\theta}D_JA_R
=
-\weights{L+1,\theta}D_J\weights{L,\theta}D_{S_{L-1}(R)}\weights{L-1,\theta}\cdots D_{S_1(R)}\weights{1,\theta}.
$
Since $f_\eta=f_\theta$, this linear part must vanish:
\[
\weights{L+1,\theta}D_J\weights{L,\theta}D_{S_{L-1}(R)}\weights{L-1,\theta}\cdots D_{S_1(R)}\weights{1,\theta}=0.
\]

By Condition~(2) in \Cref{def:generic_parameter}, the above matrix has rank
$
\min\bigl(d,\; n_{L+1},\; |J|,\; |S_1(R)|,\dots,|S_{L-1}(R)|\bigr).
$
Since $J\neq\emptyset$ and $S_k(R)\neq\emptyset$ for all $k\in[L-1]$, this minimum is at least $1$. Hence the matrix is nonzero, a contradiction.

Therefore $J=\emptyset$, so $s_j=+1$ for all $j\in[n_L]$. Thus the final hidden layer is also uniquely determined up to permutation and positive scaling. We conclude that every generic parameter $\eta$ with $f_\eta=f_\theta$ is equivalent to $\theta$, and hence $\theta$ is identifiable among generic parameters.
\end{proof}

\section{Identifiable Parameters in Most Architectures} 
\label{sec:identifiable_parameters}
In this section, we show that for every architecture with layers of widths at least $2$, there exists a nonempty open subset of identifiable parameters.  We first construct an open set on which functions determine parameters among
generic parameters, and then upgrade this to identifiability among all parameters by excluding the
lower-dimensional set admitting nongeneric parameters in the fiber.

Our strategy for the first part is to construct a generic parameter satisfying cTPIC and LRA, and then use the openness of these conditions.
Our construction is inspired by the inductive bent-hyperplane approach of \cite{grigsby2023hiddensymmetriesrelunetworks}, but refines it by using bounded polytopes and slab layers to keep track of which induced nonlinearities remain visible in the final function. 
This refinement enables us to extend the construction beyond the constraint  $n_\ell \ge d$ and to ensure preservation of LRA. 
In particular, while the inductive construction of \cite{grigsby2023hiddensymmetriesrelunetworks} enforces transverse intersections at the level of the canonical polyhedral complex, additional care is needed to ensure these intersections remain visible as breakpoints of the function. We discuss this in more detail in \Cref{sec:grigsby-gap}.

	\subsection{Slab Layers}
     We aim to construct a sequence of polyhedra in the domains of the intermediate layers on which the two consecutive layers satisfy cTPIC, and then pull this structure back to the input space. The individual layers will be \emph{slab layers} on these polytopes.
\begin{definition}
\label{def:inside and slab layer}
		Let $P \subseteq \R^d$ be a polyhedron. We say that a hyperplane $H$ is \emph{inside} $P$ if it intersects $P$ transversely, that is, if $\relint(P)\cap H \neq \emptyset$ and \(\dim(P \cap H) = \dim(P) -1.\) 
        A hyperplane arrangement $\HA \subseteq \R^d$ is \emph{inside} $P$ if every $H \in \HA$ is inside $P$. We say that a ReLU layer $f_\theta \colon \R^d \to \R^n$ is a \emph{slab layer} on $P$ if \begin{enumerate}
		    \item $H$ is inside $P$ for all $H \in \HA_\theta$, and 
            \item $H_1 \cap H_2 \cap P = \emptyset$ for all $H_1,H_2 \in \HA_\theta$. 
		\end{enumerate}   
\end{definition} 

\begin{definition}
            Let $f_\theta \colon \R^d \to \R^n$ be a slab layer on $P$. Let $H_1, \dots, H_n$ be the hyperplanes of the slab layer, ordered such that they partition $P$ into regions $R_0, \dots, R_n \in \complex_\theta(P)$ with $R_{k-1} \cap R_k =H_k \cap P$. Then we call the slab layer \emph{oriented} if $S(R_i) = 
            \{1,\ldots, i\} \triangle \{n\}$, where $\triangle$ denotes the symmetric difference.
\end{definition} 

See \Cref{fig:slab_layer} for an illustration of an oriented slab layer that partitions a polytope into regions. 
Since our goal is to construct an open set, 
we need to ensure the notions we introduce are open, 
meaning that they 
are preserved under sufficiently small perturbations. 

\begin{definition}
 
		Let $H= \{ x \in \R^d \mid \inner{w,x}+b=0\}$ be a hyperplane and let $\varepsilon >0$. 
        We say that a hyperplane $H'=\{ x \in \R^d \mid \inner{w',x}+b'=0\}$ is  $\varepsilon$-close to $H$ if 
        \[
       \min\left\{ \big\|\frac{(w',b')}{\|w'\|} -{\frac{(w,b)}{\|w\|} }\big\|_\infty,    \big\|\frac{(w',b')}{\|w'\|} -{\frac{(-w,-b)}{\|w\|} }\big\|_\infty \right\}<\varepsilon . 
        \]
        A hyperplane arrangement $\HA$ is $\varepsilon$-close to $H$ if all its hyperplanes are $\varepsilon$-close to $H$.
\end{definition}

    \begin{remark}
    \label{rem:generic_perturbation_stay_inside}
        If a hyperplane $H$ is inside a polyhedron $P$ (\Cref{def:inside and slab layer}), then there is an $\varepsilon>0$ such that all hyperplanes $H'$ that are $\varepsilon$-close to $H$ are also inside $P$. 
    \end{remark}

    Next, given a polytope $P$ inside a bounded set $X$ and a hyperplane inside $P$, 
    we construct an oriented slab layer inside $P$ and transparent on $X$, 
    by suitably perturbing and orienting the hyperplanes. 
    For the iterative construction, the reader can think of $X$ as the image of the previously constructed layers on a prescribed polytope in the input space. 

\begin{lemma}
\label{lem:hyperplane_to_slab_layer}
    Let $X \subseteq \R^d$ be a bounded set, $P \subseteq X$ a polytope, $H$ a hyperplane inside $P$, $\varepsilon>0$, and $n \geq 2$. 
    Then there exists a generic oriented slab layer $f_\theta\colon \R^d \to \R^n$ on $P$ that is transparent on $X$ and whose hyperplane arrangement $\HA_\theta$ is $\varepsilon$-close to $H$.
\end{lemma}

\begin{proof}
    Let $H$ be given by $w^\top x + b = 0$. Since $H$ is inside $P$, it intersects $\relint(P)$ transversely. Because transversality is an open condition, there exists a neighborhood of $(w,b)$ such that all hyperplanes in this neighborhood remain inside $P$. Choose $n$ distinct values $b_1 > \dots > b_n$  in an $\varepsilon$ neighborhood of $b$ and define perfectly parallel hyperplanes $H'_i$ with weight vectors $\pm w$ such that the activation pattern follows $S(R_i) = \{1,\ldots, i\} \triangle \{n\}$. 

    In this parallel configuration, neuron $n$ is active on the half-space $V_n$ containing $R_0, \dots, R_{n-1}$, and neurons $1, \dots, n-1$ are active on the half-spaces $V_1, \dots, V_{n-1}$ containing $R_n$. Since $b_1 > b_n$, the union of the active half-spaces $V_1 \cup V_n$ covers the entire space $\R^d$, ensuring that the layer is transparent. 
    To maintain transparency on $X$ 
    under perturbation, 
    the intersection of the inactive half-spaces $V_1^c \cap V_n^c$ must remain disjoint from $X$. 

 Since $X$ and $P$ are bounded sets, there are sufficiently small generic perturbations $\eta_1, \dots, \eta_n$ of the weight vectors (in particular, $\|\eta_i-w\| < \varepsilon)$ such that the intersections $H_i \cap H_j$ of the perturbed hyperplanes lie outside the bounded set $X$ and each $H_i$ remains inside $P$. 
    In this way, we ensure that for every $x \in X$, either neuron $n$ or neuron $1$ is active. 
    In particular, the resulting layer is a generic oriented slab layer on $P$ that is transparent on $X$ and the hyperplane arrangement is $\varepsilon$-close to $H$. 
\end{proof}

\subsection{From Slab Layers to cTPIC}
To carry out the inductive construction, we need to compose a slab layer with another slab layer in such a way that all bent hyperplanes intersect. 
The following lemma provides an intermediate step by identifying a neuron with this property for the second layer.

\begin{lemma}
\label{lem:slab_layer_to_hyperplane}
Let $f_\theta \colon \R^d \to \R^n$ be a generic oriented slab layer on a polyhedron $P \subseteq \R^d$. Then there exist a hyperplane $H' \subseteq \R^n$ and a region $R \in \complex_\theta(P)$ such that:
\begin{enumerate}
    \item $f_\theta^{-1}(H') \cap P$ is connected;
    \item for every $H \in \HA_\theta$, the preimage $f_\theta^{-1}(H')$ intersects $H$ transversely in $P$;
    \item $H'$ is inside $f_\theta(R)$.
\end{enumerate}
\end{lemma}

\begin{proof}
Let $H_1, \dots, H_n$ be the hyperplanes of the slab layer, ordered such that they partition $P$ into regions $R_0, \dots, R_n$ with $R_{k-1} \cap R_k =H_k \cap P$. 
For each $k \in [n]$, pick a point $x^{(k)} \in \relint(H_k \cap P)$ and let $y^{(k)} = f_\theta(x^{(k)}) \in \R^n$. The $j$-th component of $y^{(k)}$ is given by 
$y^{(k)}_j = [\weights{1}_j x^{(k)} + \bias{1}_j]_+$.

By the definition of an oriented slab layer ($S(R_i) = \{1,\ldots, i\} \triangle \{n\}$), 
the neurons active on the $k$-th hyperplane are $S(H_k) = S(R_{k-1}) \cap S(R_k)$. Specifically, for $k < n$, $S(H_k) = \{1, \dots, k-1, n\}$, and for $k = n$, $S(H_n) = \{1, \dots, n-1\}$. Hence, the matrix $Y = [y^{(1)}, \dots, y^{(n)}]$ of image points has the following structure: 
\[
Y = 
\begin{pmatrix} 
0 & y^{(2)}_1 & y^{(3)}_1 & \dots & y^{(n-1)}_1 & y^{(n)}_1 \\
0 & 0 & y^{(3)}_2 & \dots & y^{(n-1)}_2 & y^{(n)}_2 \\
\vdots & \vdots & \vdots & \ddots & \vdots & \vdots \\
0 & 0 & 0 & \dots & y^{(n-1)}_{n-2} & y^{(n)}_{n-2} \\
0 & 0 & 0 & \dots & 0 & y^{(n)}_{n-1} \\
y^{(1)}_n & y^{(2)}_n & y^{(3)}_n & \dots & y^{(n-1)}_n & 0
\end{pmatrix} , 
\]
where all $y^{(i)}_j$ are nonzero. 
The submatrix formed by the first $n-1$ rows and the last $n-1$ columns is 
upper-triangular with strictly positive diagonal entries $y^{(k)}_{k-1}$, $k=2,\ldots, n$. 
This ensures these $n-1$ columns are linearly independent. 
Since $y^{(1)}$ is the only column with a non-zero entry in the $n$-th row and zeros in the first $n-1$ rows, the full set of $n$ columns is linearly independent. 
In $\R^n$, linear independence of $n$ vectors implies their affine independence, uniquely defining a hyperplane $H'$ passing through these points. 

Now, for $m\up{i} = \frac{1}{2}(y\up{i}+y\up{i+1})$, we have that $m\up{i} \in H'$ and $f_\theta^{-1}(m\up{i}) \cap \relint(R_i) \neq \emptyset$.  
Thus $f_\theta^{-1}(H') \cap \relint(R_i) \neq \emptyset$. 
This implies that $f_\theta^{-1}(H')$ and $H_i$ intersect transversely in $P$. 
 
Indeed, since 
$H'$ is the affine hull of $\{y^{(1)}, \dots, y^{(n)}\}$, the midpoint $m^{(i)} = \frac{1}{2}(y^{(i)} + y^{(i+1)})$ is contained in $H'$. 
The ReLU layer $f_\theta$ is affine-linear on the region $R_i$ bounded by $H_i$ and $H_{i+1}$. 
By linearity, $f_\theta(\frac{1}{2}x^{(i)} + \frac{1}{2}x^{(i+1)}) = m^{(i)}$. Because $P$ is convex and the hyperplanes are disjoint in $P$, 
the midpoint $x_m\up{i} = \frac{1}{2}(x^{(i)} + x^{(i+1)})$ lies in $\relint(R_i)$. 
Thus, the pullback $f_\theta^{-1}(H')$ contains a point on the boundary ($x^{(i)}$) and a point in the interior ($x_m\up{i}$) of $R_i$. 
This implies that $f_\theta^{-1}(H')$ and $H_i$ intersect transversely. Specifically, for each $i$, the pullback $f_\theta^{-1}(H')$ contains a point on the boundary of $R_i$ ($x^{(i)} \in H_i \cap \relint(P)$) and a point in the interior ($x_m^{(i)} \in \relint(R_i)$). Since the pullback $f_\theta^{-1}(H')$ on $R_i$ is an affine hyperplane and contains a point in $\relint(R_i)$, it cannot be identical to the boundary hyperplane $H_i$. Consequently, on $R_i$, we have that $f_\theta^{-1}(H')$ and $H_i$ are distinct hyperplanes sharing at least one point in $\relint(P)$, ensuring they intersect transversely in $P$, that is, their intersection in $P$ is of dimension $\dim(P)-2$.

We next show that $f_\theta^{-1}(H') \cap P$ is connected. For each $i\in\{0,\dots,n\}$, the set
$
f_\theta^{-1}(H') \cap R_i
$
is the intersection of the polyhedron $R_i$ with an affine hyperplane, since $f_\theta$ is affine-linear on $R_i$. Hence it is convex, and therefore connected whenever nonempty. We have already shown that it is nonempty for every $i$.
Moreover, for each $i\in[n]$, the set $f_\theta^{-1}(H')$ contains the point $x^{(i)} \in H_i \cap P$, so
$
f_\theta^{-1}(H') \cap R_{i-1}
\quad\text{and}\quad
f_\theta^{-1}(H') \cap R_i
$
meet along the common boundary hyperplane $H_i \cap P$. Thus the connected sets
$
f_\theta^{-1}(H') \cap R_0,\,
f_\theta^{-1}(H') \cap R_1,\,
\dots,\,
f_\theta^{-1}(H') \cap R_n
$
form a chain with consecutive intersections nonempty. Their union is therefore connected. Since this union equals $f_\theta^{-1}(H') \cap P$, the latter is connected.

Finally, since $f_\theta$ is linear on $R_{n-1}$ (a polyhedron), it maps points in the relative interior of $R_{n-1}$ to points in the relative interior of $f_\theta(R_{n-1})$. Since $x_{m}\up{n-1}\in\relint(R_{n-1})$, we have $f_\theta(x_m\up{n-1}) \in \relint f_\theta(R_{n-1})$. The latter is $m\up{n-1}\in H'$ and thus $H'$ intersects $\relint(f_\theta(R_{n-1}))$. 
Moreover, since $f_\theta^{-1}(H')$ and $H_{n-1}$ intersect transversely in $P$, we have $f_\theta(R_{n-1})\not\subseteq H'$. Thus $\dim (f_\theta(R_{n-1})\cap H') = \dim (f_\theta(R_{n-1}))-1$ and hence $H'$ is inside $f_\theta(R_{n-1})$ in the sense of \Cref{def:inside and slab layer}.

\end{proof}

\begin{lemma}
\label{lem:generic_perturbation_intersection}
Let $P \subseteq \R^d$ be a polyhedron, let $f_\theta \colon \R^d \to \R^n$ be a ReLU layer, and let $H' \subseteq \R^n$ be a hyperplane such that $f_\theta^{-1}(H') \cap P$ is connected and, for all $H \in \HA_\theta$, the hypersurfaces $f_\theta^{-1}(H')$ and $H$ intersect transversely in $P$. Then there exists $\varepsilon >0$ such that for every hyperplane $H''$ that is $\varepsilon$-close to $H'$,
\begin{enumerate}
    \item $f_\theta^{-1}(H'') \cap P$ is connected, and
    \item $f_\theta^{-1}(H'')$ and $H$ intersect transversely in $P$ for all $H \in \HA_\theta$.
\end{enumerate}
\end{lemma}

\begin{proof}
By the definition of transversality in $P$ (\Cref{def:intersect transversely}), for each $H \in \HA_\theta$, there exists an intersection point $x \in f_\theta^{-1}(H') \cap H \cap \relint(P)$, where each local normal of $f_\theta^{-1}(H')$ 
and the normal of $H$ are linearly independent. 
Because linear independence is an open condition and the intersection point $x$ varies continuously with the parameters $(w', b')$ corresponding to the hyperplane  $H'$, there exists a neighborhood in the parameter space such that these properties are preserved. 

According to our definition of $\varepsilon$-closeness, this neighborhood corresponds to a ball $B_\varepsilon$ around $(w', b')/\|w'\|$ or $(-w', -b')/\|w'\|$. Since $\HA_\theta$ is finite, we can choose $\varepsilon > 0$ small enough to satisfy these conditions for all $H$ simultaneously. For any $H''$ that is $\varepsilon$-close to $H'$, the pullback $f_\theta^{-1}(H'')$ maintains a non-empty intersection with each $H \cap \relint(P)$ and preserves the linear independence of local normal vectors, ensuring the intersection remains transverse in $P$. By the same argument as in the proof of \Cref{lem:slab_layer_to_hyperplane}, $f_\theta^{-1}(H'')\cap P$ is connected.
\end{proof}

\subsection{Pulling cTPIC Back to Input Space}
\Cref{lem:slab_layer_to_hyperplane} and \Cref{lem:hyperplane_to_slab_layer} give us a recipe to iteratively construct slab layers 
$f_{\theta_1} , \ldots , f_{\theta_L}$ 
such that $f_{\theta_{\ell+1}} \circ f_{\theta_{\ell}}$ satisfy cTPIC on a (possibly) lower-dimensional polytope $Q\up{\ell}$ in $\R^{n_{\ell-1}}$. Next, we aim to pull the transverse intersections back to the input space, for which we need the following definition.

\begin{definition}
\label{def:transverse linear map}
    Let $T \colon \R^d \to \R^n$ an affine linear map given by $x \mapsto Ax +b$, and let $B \subseteq \R^n$ be a piecewise linear hypersurface. 
    Then we say that $T$ is \emph{transverse} to $B$ if for every local normal $w$ of $B$ we have that $A^\top w \neq 0$ and hence $T^{-1}(B)$ is a piecewise linear hypersurface.
\end{definition} 

With this definition at hand, we show that, generically, these intersections can be pulled back to the input space. 

\begin{lemma}
\label{lem:pullback_transversality}
    Let $P \subseteq \R^d$ be a full-dimensional polyhedron, let $T \colon \R^d \to \R^n$ be an affine map, and let $f_\theta \colon \R^n \to \R^m$ be an oriented slab layer on $T(P)$. 
    Let $H' \subseteq \R^m$ be a hyperplane such that, for every $H \in \HA_\theta$, $f_\theta^{-1}(H')$ and $H$ intersect transversely in $T(P)$. 
    Furthermore, assume that $T$ is transverse to both $f_\theta^{-1}(H')$ and $H$. 
    Then $(f_\theta \circ T)^{-1}(H')$ and $T^{-1}(H)$ intersect transversely in $P$. 
\end{lemma}

\begin{proof}
    Let $H \in \HA_\theta$ be a hyperplane in $\R^n$ and let $B = f_\theta^{-1}(H')$ be the piecewise linear hypersurface in $\R^n$. By the first assumption, there exists a point $y \in B \cap H \cap \relint(T(P))$. Since affine maps send relative interiors to relative interiors of their images, and
$y \in \relint(T(P))$, there exists a point $x \in \relint(P)$ such that $T(x)=y$. By the definition of the pullback, $x \in (f_\theta \circ T)^{-1}(H')$ and $x \in T^{-1}(H)$, which establishes that the intersection in the relative interior of $P$ is non-empty.

    To show that the intersection is transverse in $P$, we analyze the local linear parts. Let $w_H$ be the normal vector of $H$ and $w_B$ be any local normal vector of $B$ at $y$. 
    Let $A$ be the linear part of the affine map $T$ and $V = \text{im}(A)$ be its image. The affine hull $\aff(T(P))$ is a translate of $V$. The assumption that $H$ and $B$ intersect transversely in $T(P)$ implies that their projections onto $V$, denoted $\pi_V(w_H)$ and $\pi_V(w_B)$, are linearly independent.

    The local normal vectors of the pulled-back hypersurfaces at $x$ are $u = A^\top w_H$ and $v = A^\top w_B$. Suppose there exist $\alpha, \beta \in \R$ such that $\alpha u + \beta v = 0$. This is equivalent to $A^\top (\alpha w_H + \beta w_B) = 0$. Let $z = \alpha w_H + \beta w_B$. The condition $A^\top z = 0$ implies that $z \in V^\perp$, or equivalently, $\pi_V(z) = 0$. By the linearity of the projection, we have:
    \[
    \alpha \pi_V(w_H) + \beta \pi_V(w_B) = 0.
    \]
    Since $\pi_V(w_H)$ and $\pi_V(w_B)$ are linearly independent in $V$, we must have $\alpha = \beta = 0$. Consequently, the pulled-back normals $u$ and $v$ are linearly independent in $\R^d$. 

    Since the intersection is non-empty in $\relint(P)$ and the local normal vectors are linearly independent, the intersection $(f_\theta \circ T)^{-1}(H') \cap T^{-1}(H) \cap P$ is pure of dimension $\dim(P)-2$. This satisfies the definition of transversality in $P$.
\end{proof}
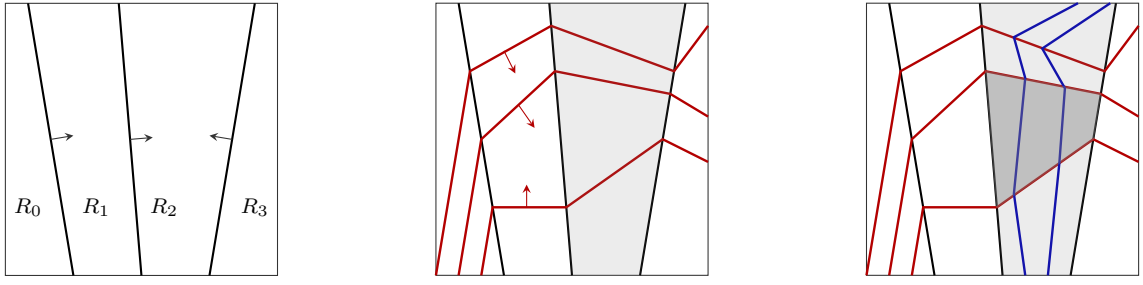
\begin{figure}[t]
\centering
\begin{subfigure}[t]{0.31\textwidth}
\centering
\begin{tikzpicture}[scale=0.3,x=1cm,y=1cm,line cap=butt,line join=miter]
    \footnotesize
    \draw[black!80] (0,0) -- (0,12) -- (12.0,12.0) -- (12,0) -- cycle;

    \draw[black,line width=0.8pt] (3,0) -- (1,12);
    \draw[black,line width=0.8pt] (6,0) -- (5,12);
    \draw[black,line width=0.8pt] (9,0) -- (11,12);

    \node at (1,3) {$R_0$};
    \node at (4,3) {$R_1$};
    \node at (7,3) {$R_2$};
    \node at (11,3) {$R_3$};
    \draw[->,>=stealth,black!80] (2,6) --(3,6+1/6);
    \draw[->,>=stealth,black!80] (5.5,6) --(6.5,6+1/12);
    \draw[->,>=stealth,black!80] (10,6) --(9,6+1/6);

\end{tikzpicture}
\caption{A generic oriented slab layer $f_{\theta_1}$ with $3$ neurons on a polytope (the square). 
The layer's hyperplanes partition the polytope into $4$ regions. 
}
\label{fig:slab_layer}
\end{subfigure}\hfill
\begin{subfigure}[t]{0.31\textwidth}
\centering
\begin{tikzpicture}[scale=0.3,x=1cm,y=1cm,line cap=butt,line join=miter]
    \footnotesize
    \draw[black!80] (0,0) -- (0,12) -- (12.0,12.0) -- (12,0) -- cycle;

    \draw[black,line width=0.8pt] (3,0) -- (1,12);
    \draw[black,line width=0.8pt] (6,0) -- (5,12);
    \draw[black,line width=0.8pt] (9,0) -- (11,12);
    \draw[red!70!black,line width=0.9pt]
        (2,0) -- (2.5,3) -- (5.75,3) -- (10,6) -- (12,5);
        \draw[red!70!black,line width=0.9pt]
        (1,0) -- (2,6) -- (5.25,9) -- (10.3333,8) -- (12,7);
         \draw[red!70!black,line width=0.9pt]
        (0,0) -- (1.5,9) -- (5+1/12,11) -- (10.5,9) -- (12,11);
        \draw[->,>=stealth,red!70!black] (4,3) --(4,4);
    \draw[->,>=stealth,red!70!black,] (3.65,7.5) --(4.35,6.5);
        \draw[->,>=stealth,red!70!black] (3,9.9) --(3.5,8.9);
          \filldraw[gray,opacity=0.15] (6,0) -- (9,0) -- (11,12) -- (5,12) -- cycle;
\end{tikzpicture}
\caption{The 
gray area is $P\up{1}$. The red bent hyperplanes are the pullbacks of the hyperplanes of the slab layer $f_{\theta_2}$ on $Q\up{1}=f_{\theta_1}(P\up{1})$ along $f_{\theta_1}$.}
\label{fig:slab_layer_1}
\end{subfigure}\hfill
\begin{subfigure}[t]{0.31\textwidth}
\centering
\begin{tikzpicture}[scale=0.3,x=1cm,y=1cm,line cap=butt,line join=miter]
    \footnotesize
    \draw[black!80] (0,0) -- (0,12) -- (12.0,12.0) -- (12,0) -- cycle;

    \draw[black,line width=0.8pt] (3,0) -- (1,12);
    \draw[black,line width=0.8pt] (6,0) -- (5,12);
    \draw[black,line width=0.8pt] (9,0) -- (11,12);
  
    \draw[red!70!black,line width=0.9pt]
        (2,0) -- (2.5,3) -- (5.75,3) -- (10,6) -- (12,5);
        \draw[red!70!black,line width=0.9pt]
        (1,0) -- (2,6) -- (5.25,9) -- (10.3333,8) -- (12,7);
         \draw[red!70!black,line width=0.9pt]
        (0,0) -- (1.5,9) -- (5+1/12,11) -- (10.5,9) -- (12,11);
          \draw[blue!70!black,line width=0.9pt]
          (7,0) -- (6.5,3.6)--  (7,8.7)--(6.5, 10.5) --(6.5+4.24/1.5, 12);

              \draw[blue!70!black,line width=0.9pt]
          (8,0) -- (8.5,4.9)--  (8.75,8.25)--(7.75, 10) --(7.75+3, 10+2);
    \filldraw[gray,opacity=0.4] (5.75,3) -- (10,6) 
        -- (10.3333,8) -- (5.25,9)-- cycle;
       \filldraw[gray,opacity=0.15] (6,0) -- (9,0) -- (11,12) -- (5,12) -- cycle;
\end{tikzpicture}
\caption{The dark gray area is $P\up{2}$. The blue bent hyperplanes are the pullbacks of the hyperplanes of the slab layer $f_{\theta_3}$ on  $Q\up{2}$ along $f_{\theta_2} \circ f_{\theta_1}$.}
\label{fig:slab_layer_2}
\end{subfigure}
\caption{Illustration of the inductive construction in \Cref{thm:existence_TPIC_LRA} that satisfies cTPIC and LRA.} 
\label{fig:induction_TPIC_LRA}
\end{figure}

\subsection{The Inductive Construction}

Combining \Cref{lem:hyperplane_to_slab_layer}, \Cref{lem:slab_layer_to_hyperplane} and \Cref{lem:pullback_transversality}, 
we establish in Theorem~\ref{thm:existence_TPIC_LRA} the existence of an open set of parameters satisfying both cTPIC and LRA. 
Our construction proceeds inductively, layer by layer, as illustrated in 
\Cref{fig:induction_TPIC_LRA}. 

	\begin{theorem}
    \label{thm:existence_TPIC_LRA}
		Let $\architecture= (n_0,\ldots,n_{L+1})$ be any architecture with $n_i \geq 2$ for all $i \leq L$ and $P \subseteq \R^{n_0}$ a full-dimensional polytope. Then there is an open set $U \subseteq \paramspace{\architecture}$ such that $f_\theta$ satisfies cTPIC and LRA on $P$ for all $\theta \in U$.
	\end{theorem}
	
	\begin{proof}
If $L=1$, there is no adjacent pair of hidden layers, so cTPIC is vacuous. Choose the
first hidden layer to be a generic oriented slab layer on $P$, and choose the output
weights generically so that every first-layer hyperplane carries nonzero tropical
weight. Then LRA holds on $P$ by \Cref{lem:LRA_tropical_weight}. These conditions are
open, so the desired open set exists.

Assume now that $L\ge2$.
     We iteratively construct a sequence of full-dimensional polytopes $P = P\up{0} \supseteq P\up{1} \supseteq \cdots \supseteq P\up{L-2}$ and a generic parameter $\theta$ satisfying LRA on $P$ and such that the connected pieces of the bent hyperplanes from layers $\ell$ and $\ell+1$ intersect transversely in $P\up{\ell-1}$. 
    The construction is illustrated in \Cref{fig:induction_TPIC_LRA}.      

    \paragraph{Inductive construction}  
     We show the following by induction.  
     For any $\ell \geq 2$, there are full-dimensional polytopes $P = P\up{0} \supseteq \cdots \supseteq P\up{\ell-1}$ and a generic parameter $\theta\up{\ell} =(\theta_1,\ldots,\theta_\ell)\in \Theta\up{\ell}$ such that:  
     \begin{enumerate}
         
         \item $f_{\theta\up{\ell-1}}$ is affine linear on $P\up{\ell-1}$ 
         \item $f_{\theta_\ell}$ is a generic oriented slab layer on $Q\up{\ell-1} \coloneqq f_{\theta\up{\ell-1}}(P\up{\ell-1})$ (\Cref{def:inside and slab layer}) 
         and transparent on $X\up{\ell-1}=f_{\theta\up{\ell-1}}(P)$ (\Cref{def:transparent})
         
         \item $f_{\theta_\ell} \circ f_{\theta_{\ell-1}}$ satisfies cTPIC on $Q\up{\ell-2}$ (\Cref{def:intersect transversely})
         
         \item The affine linear map $T\up{\ell-2} \coloneqq f_{\theta\up{\ell-2}}|_{P\up{\ell-2}}$ 
       (with $T^{(0)} \coloneqq \id|_P$ if $\ell=2$)
         is transverse to the bent hyperplanes of $f_{\theta_\ell} \circ f_{\theta_{\ell-1}}$ 
         (\Cref{def:transverse linear map})

     \end{enumerate}

     Moreover, $f_{\theta_1}$ is a generic oriented slab layer on $Q\up{0} \coloneqq P\up{0} =P$ transparent on $X\up{0} \coloneqq P\up{0} =P$. 
     
\medskip 

Before we start the induction, we note that the fourth point is true for generic parameters. 
     The linear part of the map $T^{(\ell-2)} = f_{\theta^{(\ell-2)}}|_{P^{(\ell-2)}}$ 
     is given by the matrix product $A = D_{S_{\ell-2}}W^{(\ell-2)} D_{S_{\ell-3}} \cdots D_{S_1} W^{(1)}$. 
     The (local) normals of the (bent) hyperplanes of $f_{\theta_\ell} \circ f_{\theta_{\ell-1}}$ are given by the vectors $v_j=\weights{\ell-1}_j$ 
     and vectors  
     $u_{S,i}=(W^{(\ell)} D_{S} W^{(\ell-1)})_i$ for $j \in [n_{\ell-1}], i \in [n_\ell]$ and $S\subseteq [n_{\ell-1}]$.
     By definition, $T^{(\ell-2)}$ is transverse to these bent hyperplanes if $v_j A \neq 0$ and $u_{S,i} A \neq 0$ for all possible choices $j \in [n_{\ell-1}], i \in [n_\ell], S \subseteq [n_{\ell-1}]$. 
If $v_j A = 0$, then  $D_{\{j\}}W^{(\ell-1)} D_{S_{\ell-2}} W^{(\ell-2)} \cdots 
W^{(1)}=0$, which would contradict Condition 2 of \Cref{def:generic_parameter}, that states that the matrix product $D_{\{j\}}W^{(\ell-1)} D_{S_{\ell-2}} W^{(\ell-2)} \cdots 
W^{(1)}$ must achieve the maximum possible rank.
In the same manner, if $u_{S,i} A = 0$, then  $D_{\{i\}}W^{(\ell)} D_{S_{\ell-1}} W^{(\ell-1)} \cdots 
W^{(1)}=0$, which again would contradict \Cref{def:generic_parameter}. 
Consequently, for any generic parameter, $T^{(\ell-2)}$ is necessarily transverse to the bent hyperplanes of the subsequent layer. 

\paragraph{Base case of the induction}  

     For the base case of the induction, let $\ell=2$. 
     The full-dimensional polytope $P$ has dimension $n_0\geq 2$, and we can pick a hyperplane $H$ inside $P$ (\Cref{def:inside and slab layer}). 
     Now by \Cref{lem:hyperplane_to_slab_layer}, there is an oriented generic slab layer $f_{\theta_1} \colon \R^{n_0} \to \R^{n_1}$ on $P$ that is transparent on $P$. 
     By \Cref{lem:slab_layer_to_hyperplane}, there is a hyperplane $H' \subseteq \R^{n_1}$ such that $f_{\theta_1}^{-1}(H')$ intersects every $H \in \HA_{\theta_1}$ transversely and $f_{\theta_1}^{-1}(H') \cap P$ is connected. 
     By \Cref{lem:generic_perturbation_intersection}, there is a $\varepsilon_1 > 0$ such that for every hyperplane $H''$ that is $\varepsilon_1$-close to $H'$, the bent hyperplane $f_{\theta_1}^{-1}(H'')$ still has transverse intersection with each $H \in \HA_{\theta_1}$ and is still connected in $P$. 
     Moreover, again by \Cref{lem:slab_layer_to_hyperplane}, there is a region $P\up{1} \in \complex_{\theta_1}(P)$ such that $H'$ is inside $Q\up{1} = f_{\theta_1}(P\up{1})$. 
     By \Cref{rem:generic_perturbation_stay_inside}, there is an $\varepsilon_2$ such that every hyperplane $H''$ that is $\varepsilon_2$-close to $H'$ is also inside $Q\up{1}$. 
     Let $\varepsilon = \min\{\varepsilon_1,\varepsilon_2\}$. 
     Now, by \Cref{lem:hyperplane_to_slab_layer}, there is a generic oriented slab layer $f_{\theta_2}$ on $Q\up{1}$ that is transparent on the bounded set $X\up{1}= f_{\theta_1}(P)$ and $\varepsilon$-close to $H'$. 
     Finally, 
     a small perturbation of the construction ensures that the parameter is generic, implying that the fourth point holds. 
     This completes the base case. 

\paragraph{Induction step} 

     Now for the induction step, let $\ell \geq 2$ and assume that the induction hypothesis holds. 
     We proceed to construct layer $\ell+1$ analogously to the base case. 

By the induction hypothesis, $f_{\theta_\ell}$ is a generic oriented slab layer on $Q^{(\ell-1)}$. 
Thus, by \Cref{lem:slab_layer_to_hyperplane}, there exists a hyperplane $H' \subseteq \R^{n_\ell}$ 
and a region $R \in \complex_{\theta_\ell}(Q^{(\ell-1)})$ 
such that:  
$f_{\theta_\ell}^{-1}(H')$ is connected in $Q^{(\ell-1)}$,  intersects every $H \in \HA_{\theta_\ell}$ transversely in $Q^{(\ell-1)}$, 
and such that $H'$ is inside $f_{\theta_\ell}(R)$. 

We define $P^{(\ell)} \coloneqq f_{\theta^{(\ell-1)}}^{-1}(R)$, which is a region in $\complex_{\theta^{(\ell)}}(P^{(\ell-1)})$ and hence has dimension $\dim(P^{(\ell-1)})=\dim(P)$. 
We have $Q^{(\ell)} \coloneqq f_{\theta^{(\ell)}}(P^{(\ell)}) = f_{\theta_\ell}(R)$. 
By \Cref{lem:hyperplane_to_slab_layer}, there exists a generic oriented slab layer $f_{\theta_{\ell+1}}$ on $Q^{(\ell)}$ that is transparent on the bounded set $X^{(\ell)} \coloneqq f_{\theta_\ell}(X^{(\ell-1)}) = f_{\theta^{(\ell)}}(P)$ and whose hyperplane arrangement $\mathcal{H}_{\theta_{\ell+1}}$ is $\varepsilon$-close to $H'$, for any choice of $\varepsilon>0$. 
We set $\varepsilon = \min\{\varepsilon_1, \varepsilon_2\}$, where $\varepsilon_1 > 0$ is given by \Cref{lem:generic_perturbation_intersection} and ensures that, for every $H''$ that is $\varepsilon_1$-close to $H'$, $f_{\theta_\ell}^{-1}(H'')$ is still connected in $Q^{(\ell-1)}$, intersects every $H \in \HA_{\theta_\ell}$ transversely in $Q^{(\ell-1)}$, 
and $\varepsilon_2 > 0$ is given by \Cref{rem:generic_perturbation_stay_inside} and ensures that every $H''$ that is $\varepsilon_2$-close to $H'$ is inside $Q^{(\ell)}$. 
By this construction, we have:  
\begin{itemize}
    \item $f_{\theta^{(\ell)}}$ is affine-linear on $P^{(\ell)}$; 
    \item $f_{\theta_{\ell+1}}$ satisfies the generic oriented slab and transparency requirements on $Q^{(\ell)}$ and $X^{(\ell)}$, respectively;  
    \item The $\varepsilon$-closeness of $\mathcal{H}_{\theta_{\ell+1}}$ to $H'$ ensures that $f_{\theta_{\ell+1}} \circ f_{\theta_\ell}$ satisfies cTPIC on $Q^{(\ell-1)}$;     
    \item If necessary, we can slightly perturb the construction to ensure that the resulting parameter is generic, which guarantees that  $T^{(\ell-1)}=f_{\theta^{(\ell-1)}}|_{P^{(\ell-1)}}$ is transverse to the bent hyperplanes of $f_{\theta_{\ell+1}} \circ f_{\theta_\ell}$. 
\end{itemize} 
This concludes the induction step. 
\paragraph{The constructed parameters satisfy cTPIC and LRA}

We first verify cTPIC on the input polytope $P$. Fix a hidden neuron $(i,\ell)$ with $\ell\in[L]$.
If $\ell=1$, we set
$
C_{i,1}\coloneqq B_{i,1}(\theta)\cap P.
$
Since $f_{\theta_1}$ is a slab layer on $P$, this is the intersection of $P$ with an affine hyperplane, hence a connected piecewise linear hypersurface.

Now let $\ell\ge 2$. By Property~3, the two-layer network
$
f_{\theta_\ell}\circ f_{\theta_{\ell-1}}
$
satisfies cTPIC on $Q^{(\ell-2)}$. Hence, for the neuron $(i,\ell)$, there exists a connected piecewise linear hypersurface
$
\widetilde{C}_{i,\ell}\subseteq B_{i,\ell}(\theta)\cap Q^{(\ell-2)}.
$
We define
$
C_{i,\ell}\coloneqq \bigl(f_{\theta^{(\ell-2)}}|_{P^{(\ell-2)}}\bigr)^{-1}(\widetilde{C}_{i,\ell}) \subseteq B_{i,\ell}(\theta)\cap P.
$
By Property~1, the map $f_{\theta^{(\ell-2)}}|_{P^{(\ell-2)}}$ is affine-linear. Since $P^{(\ell-2)}$ is convex, the preimage of a connected subset of its image under this affine map is connected. Hence $C_{i,\ell}$ is connected.

Now let $(j,\ell)$ and $(i,\ell+1)$ be neurons in adjacent layers. By Property~3, the two-layer network
$
f_{\theta_{\ell+1}}\circ f_{\theta_\ell}
$
satisfies cTPIC on $Q^{(\ell-1)}$, so the corresponding connected hypersurface pieces intersect transversely there. By Property~4 and \Cref{lem:pullback_transversality}, these intersections pull back along
$
f_{\theta^{(\ell-1)}}|_{P^{(\ell-1)}}
$
to transverse intersections between $C_{j,\ell}$ and $C_{i,\ell+1}$ in $P^{(\ell-1)}\subseteq P$.
Thus the family $\{C_{i,\ell}\}_{i,\ell}$ witnesses that $f_\theta$ satisfies cTPIC on $P$.

Moreover, Property~2 shows that $f_{\theta_\ell}$ is transparent on
$
X^{(\ell-1)}=\activation{\ell-1,\theta}(P)
$
for every $\ell\in[L]$. Since $\theta$ is generic, \Cref{lem:transparent_implies_LRA_on_X} implies that $\theta$ satisfies LRA on $P$.

Finally, the construction is stable under sufficiently small perturbations: transparency, the slab-layer conditions, connectedness of the chosen pullbacks, and the required transverse intersections are all open conditions. Hence we obtain an open subset of parameters satisfying cTPIC and LRA on $P$. This concludes the proof.
\qedhere
\end{proof}

\subsection{Functional Dimension and Identifiability}
By \Cref{thm:tpic_generic_identifiable}, generic parameters satisfying cTPIC and
LRA are identifiable among generic parameters. On the resulting open set of generic parameters, the
quotient realization map attains maximal rank, and therefore the dimension of the function space equals the quotient parameter dimension.

\begin{theorem}
\label{thm:identifiable_dimension}
Let $\architecture=(n_0,\ldots,n_{L+1})$ with $n_i\ge 2$ for all $i\le L$. Then there exists a nonempty open set
$
U_0\subseteq \widetilde{\Theta}_\architecture
$
such that every $\theta\in U_0$ is identifiable among generic parameters. Consequently,
$
\dim(\mu_\architecture(\Theta_\architecture))
=
\sum_{\ell=1}^{L+1}n_\ell n_{\ell-1}+n_{L+1}.
$
\end{theorem}
    
\begin{proof}
By \Cref{thm:existence_TPIC_LRA}, there exists a nonempty open set
$
U\subseteq \Theta_\architecture
$
such that every $\theta\in U$ satisfies cTPIC and LRA on a full-dimensional polytope.
Intersecting with the open dense generic locus $\widetilde{\Theta}_\architecture$, we obtain a nonempty open set
$
U_0\subseteq \widetilde{\Theta}_\architecture
$
such that every $\theta\in U_0$ is identifiable among generic parameters by \Cref{thm:tpic_generic_identifiable}.

Modulo permutation and positive scaling symmetries, the realization map restricted to $U_0$ therefore has a trivial fiber on $U_0$. Hence the induced quotient realization map attains rank
$
\dim(\overline{\Theta}_\architecture)
=
\sum_{\ell=1}^{L+1}(n_\ell n_{\ell-1}+n_\ell)-\sum_{\ell=1}^{L}n_\ell
$
\citep[see, e.g.,][]{grigsby2023hiddensymmetriesrelunetworks}. This shows that
$
\dim(\mu_\architecture(\Theta_\architecture))
\ge
\dim(\overline{\Theta}_\architecture).
$
The reverse inequality is automatic, since the image of the realization map cannot have larger dimension than the quotient parameter space. Therefore equality holds, and expanding the right-hand side gives
$
\dim(\mu_\architecture(\Theta_\architecture))
=
\sum_{\ell=1}^{L+1}n_\ell n_{\ell-1}+n_{L+1}.
\qedhere
$
\end{proof}

The preceding theorem still yields identifiability only among the generic parameters. To upgrade this to unconditional identifiability, we exclude functions that admit nongeneric realizations. Since the nongeneric locus is a proper algebraic subset of parameter space, its image in function space is lower-dimensional than the full image established above.
\begin{lemma}
\label{lem:nongeneric_lower_dim_quotient}
Let $N=\Theta_\architecture\setminus\widetilde{\Theta}_\architecture$ and let
$q:\Theta_\architecture\to\overline{\Theta}_\architecture$ be the quotient map. Then
$
\dim q(N)<\dim\overline{\Theta}_\architecture .
$
\end{lemma}

\begin{proof}
The generic locus $\widetilde{\Theta}_\architecture$ is invariant under permutation and
positive scaling symmetries, hence so is $N$. Since
$\widetilde{\Theta}_\architecture$ is open dense in $\Theta_\architecture$ and $q$ is an
open map, $q(\widetilde{\Theta}_\architecture)$ is open dense in
$\overline{\Theta}_\architecture$. Therefore
$
q(N)=\overline{\Theta}_\architecture\setminus q(\widetilde{\Theta}_\architecture)
$
has empty interior. Since $q(N)$ is semialgebraic, it has dimension strictly smaller than
$\overline{\Theta}_\architecture$.
\end{proof}
\begin{theorem}
\label{thm:true_identifiability}
Let $\architecture=(n_0,\ldots,n_{L+1})$ with $n_i\ge 2$ for all $i\le L$. Then there exists a nonempty open set
$
U\subseteq \Theta_\architecture
$
such that every $\theta\in U$ is identifiable.
\end{theorem}

\begin{proof}
Let
$
U_0\subseteq \widetilde{\Theta}_\architecture
$
be the open set from \Cref{thm:identifiable_dimension}, so that every $\theta\in U_0$ is identifiable among generic parameters.
and let
$
N\coloneqq \Theta_\architecture\setminus \widetilde{\Theta}_\architecture
$ the nongeneric locus. 

Let $q:\Theta_\architecture\to\overline{\Theta}_\architecture$ be the quotient map. By
\Cref{lem:nongeneric_lower_dim_quotient},
$
\dim q(N)<\dim\overline{\Theta}_\architecture .
$
Since $\mu_\architecture=\overline{\mu}_\architecture\circ q$, we get
$
\dim \mu_\architecture(N)
\le
\dim q(N)
<
\dim\overline{\Theta}_\architecture .
$
By \Cref{thm:identifiable_dimension},
$
\dim(\mu_\architecture(\Theta_\architecture))
=
\dim(\overline{\Theta}_\architecture),
$
so $\mu_\architecture(N)$ is a proper lower-dimensional subset of the full function-space image.

Now consider
$
B\coloneqq U_0\cap \mu_\architecture^{-1}(\mu_\architecture(N)).
$
Since $\mu_\architecture(N)$ is lower-dimensional whereas $\mu_\architecture(U_0)$ has full dimension, the set $B$ cannot contain a nonempty open subset of $U_0$.Since $B$ has empty interior in $U_0$, choose a nonempty open subset
$
U\subseteq U_0\setminus B.
$

We claim that every $\theta\in U$ is identifiable. Let $\theta\in U$, and let $\eta\in \Theta_\architecture$ satisfy
$
f_\eta=f_\theta.
$
If $\eta$ were nongeneric, then
$
f_\theta=f_\eta\in \mu_\architecture(N),
$
which would imply $\theta\in B$, contradicting $\theta\in U$. Thus $\eta$ must be generic. Since $\theta\in U_0$ and every point of $U_0$ is identifiable among generic parameters, it follows that
$
\eta\sim\theta.
$
Therefore $\theta$ is identifiable.
\end{proof}

\begin{remark}
\label{rem:width1}
   \Cref{thm:true_identifiability} leaves open the existence of identifiable parameters in architectures that contain non-output layers of width $1$. If a layer $\ell$ consists of a single neuron, then the activation image $\activation{\ell, \theta}(\R^d)$ is at most one-dimensional. Consequently, the Transverse Pairwise Intersection Condition (TPIC) cannot be satisfied for any neurons in subsequent layers $k > \ell$, as their preactivations depend on a single direction. 
   While the locations of  breakpoints along this
   single direction are informative about the weights, 
   the absence of transverse intersections appears to prevent the unique assignment of  neurons to breakpoints
   (modulo within-layer permutations). 
   By adapting the construction in \Cref{lem:hyperplane_to_slab_layer} to this setting, 
   it appears possible to construct parameters for which the breakpoints are fixed 
   while their neurons of origin remain ambiguous. 
   This motivates the following conjecture. 
\end{remark}

\begin{conjecture}\mbox{} 
\label{conj:width1}
    \begin{enumerate}
        \item There exist architectures $\architecture=(n_0,\ldots,n_{L+1})$ with $n_\ell=1$ for some $\ell\leq L$ for which no identifiable parameters exist. 
        
        \item For every architecture, there exists a parameter that is \emph{finitely identifiable}, meaning that its fiber in the quotient space $\overline{\paramspace{\architecture}}$ is a finite set.
    \end{enumerate}
\end{conjecture}

\Cref{conj:width1} 1) is motivated by the observation that TPIC fails for architectures containing layers of width~1; however, it remains somewhat speculative, and we hope that its resolution will stimulate the development of new methods for handling such cases.  

If \Cref{conj:width1} 2) were to hold on an open subset of parameters, it would follow that architectures containing layers of width $1$ also attain the expected dimension, namely the number of parameters minus the number of hidden neurons. Indeed, a finite fiber implies that the realization map is locally an embedding.

\section{Minimality Does Not Imply Identifiability}
\label{sec:minimality_identifiability}
Note that in a ReLU network identifiability implies minimality.  
Indeed, identifiability implies that there is no other parameter zeroing out neurons that represents the same function. 

In this subsection, we use the dependency graph and a modification of our construction from \Cref{thm:existence_TPIC_LRA} to show 
that minimal generic parameters do not need to be identifiable. 
Moreover, even if a network has an open subset of identifiable parameters it can still have an open subset of minimal parameters that are not identifiable.  
This stands in contrast to polynomial networks, where the identifiability of a generic parameter implies the identifiability of all generic parameters \citep[see, e.g.,][]{usevich2025identifiabilitydeeppolynomialneural,finkel2025activationdegreethresholdsexpressiveness,shahverdi2025learningrazorsedgesingularity}. More broadly, this example highlights that, in the ReLU setting, the relation between (non-)identifiability and reducibility of the architecture is subtle, a theme that also arises in recent abstract work connecting identifiability and equivariance in deep networks \citep{shahverdi2026identifiableequivariantnetworkslayerwise}.

We first rigidify the prefix of the network using the dependency graph, then attach an additional layer with a two-neuron always-active linear block producing a
2-dimensional fiber, and finally use a dimension comparison with strict subarchitectures to deduce minimality on an open set.

\subsection{Embedding an Identifiable Prefix}
The following lemma states a condition under which we can embed an identifiable parameter into a deeper network such that the embedded part of the parameter remains identifiable in the deeper network. 

\begin{lemma}
\label{lem:identifiability_split} 
    Let $\architecture=(n_0,\ldots,n_L,m)$ be an architecture, and let $\theta \in \gparamspace{\architecture}$. 
    Decompose the network function as $f_\theta = f_{\theta_L,\theta_{L+1}} \circ  f_{\theta\up{L-1}}$, where 
    $f_{\theta\up{L-1}} \colon \R^{n_0} \to \R^{n_{L-1}}$ and $f_{\theta_L,\theta_{L+1}} \colon \R^{n_{L-1}} \to \R^m$. 
Suppose that $\theta$ satisfies: 
    \begin{enumerate}
        \item $\theta\up{L-1}$ satisfies cTPIC and LRA on some polytope $P\subseteq\mathbb{R}^{n_0}$, 
       \item $f_{\theta_L}$ is transparent on $f_{\theta\up{L-1}}(P)$, and
        \item 
        at least one bent hyperplane $B_{L,j}(\theta)$ intersects transversely with all hyperplanes $B_{L-1,i}(\theta)$ in $P$. 
    \end{enumerate} 
    Then, for each generic $\eta=(\eta\up{L-1},\eta_L,\eta_{L+1})$ with $f_\eta=f_\theta$,
the truncated parameters $\eta\up{L-1}$ and $\theta\up{L-1}$ are equivalent up to the
trivial symmetries of the sub-architecture $(n_0,\ldots,n_{L-1})$.
\end{lemma} 
\begin{proof}
Let $G'$ be the dependency graph arising from the first $L-1$ layers of $\theta$ on $P$. Since $\theta^{(L-1)}$ satisfies cTPIC and LRA on $P$, \Cref{lem:cTPIC_implies_full_dep_subgraph} implies that $G'$ contains a layered subgraph with $L-1$ levels corresponding to the hidden layers $1,\dots,L-1$, and with all adjacent levels fully connected.

By assumption~(2), the layer $f_{\theta_L}$ is transparent  on $f_{\theta^{(L-1)}}(P)$, so the visible breakpoint geometry created by the first $L-1$ layers is preserved in the full function $f_\theta$. Hence the same layered subgraph $G'$ is contained in the dependency graph of $f_\theta$.

By assumption~(3), there exists a bent hyperplane $B_{L,j}(\theta)$ in the last hidden layer intersecting transversely every bent hyperplane $B_{L-1,i}(\theta)$ in $P$. Thus the dependency graph of $f_\theta$ contains a vertex $v_L$ corresponding to $B_{L,j}(\theta)$ together with edges from every vertex in the $(L-1)$-st level of $G'$ to $v_L$.

Now let $\eta$ be generic with $f_\eta=f_\theta$. By \Cref{prop:embedd_candidate_bh}, there is an order-preserving map from the candidate bent hyperplanes of $f_\theta$ to the hidden neurons of $\eta$. Applied to the layered subgraph $G'$, this forces its $L-1$ levels to occupy $L-1$ distinct hidden layers of $\eta$. Since $v_L$ lies strictly above the last level of $G'$, it must occupy a strictly later hidden layer. As $\eta$ has only $L$ hidden layers, the vertices of $G'$ must therefore occupy exactly the first $L-1$ hidden layers of $\eta$.

Hence the visible breakpoint geometry generated by the prefix $\theta^{(L-1)}$ is realized in the first $L-1$ layers of $\eta$. Since $\theta^{(L-1)}$ satisfies cTPIC and LRA on $P$, the reverse-engineering theorem \Cref{thm:tpic_generic_identifiable_inductive} recovers these first $L-1$ layers up to permutation and positive scaling. As in the proof of \Cref{thm:tpic_generic_identifiable}, generically, there is no sign ambiguity.

\end{proof}

Note that in contrast to cTPIC, the condition in \Cref{lem:identifiability_split} does not require that every bent hyperplane from layer $L-1$ to be intersected by every bent hyperplane from layer $L$. Consequently, \Cref{thm:tpic_generic_identifiable} does not apply to the parameters of the last hidden layer, which therefore need not be identifiable.
\subsection{Add a Linear Block}
The strategy to construct a minimal but non-identifiable parameter is to compose an identifiable parameter (as in \Cref{thm:existence_TPIC_LRA}) 
with a one-hidden-layer network that is 
minimal but has a continuous parameter symmetry, since some of its neurons do not introduce any breakpoints on the image of the preceding layer.

\begin{prop}
\label{prop:open_set_two_dimensional_fiber_drop}
For architectures $\architecture=(n_0,\ldots,n_L,n_{L+1})$ with $L\ge 2$ and
$n_i\ge 2$ for $i\in\{0,\ldots,L+1\}$ and $n_L\ge 4$, there exists a nonempty open subset
$U\subseteq \Theta_{\architecture}$ such that for every $\theta\in U$,
$
\dim\bigl(\overline{\mu}_{\architecture}^{-1}(f_\theta)\bigr)\ge 2,
$
and
$
\dim \mu_{\architecture}(U)
=
\dim \mu_{\architecture}(\Theta_{\architecture})-2.
$
\end{prop}
\begin{proof}
Let $P \subseteq \R^{n_0}$ be a full-dimensional polytope. 
We construct a parameter $\theta \in \paramspace{\architecture}$ in two stages.

By \Cref{thm:existence_TPIC_LRA}, there exists a generic parameter
$
\theta\up{L-1} \in \Theta\up{L-1}
$
for the architecture $(n_0,\dots,n_{L-1})$ such that
$f_{\theta\up{L-1}}$ satisfies cTPIC and LRA on $P$. In particular,
$\theta\up{L-1}$ is identifiable among generic parameters. Let
$
Q\up{L-1} \coloneqq f_{\theta\up{L-1}}(P) \subseteq \R^{n_{L-1}}.
$

We also note that, by the slab-layer construction, the prefix image affinely spans the
last latent space:
$
\aff\bigl(f_{\theta\up{L-1}}(\mathbb R^{n_0})\bigr)
=
\mathbb R^{n_{L-1}}.
$
Indeed, the last layer of the prefix is an oriented slab layer. As in the proof of
\Cref{lem:slab_layer_to_hyperplane}, the images of points on its $n_{L-1}$ slab
hyperplanes are linearly independent in $\mathbb R^{n_{L-1}}$, and the image of one
slab region contains a point outside their affine span. Hence the prefix image contains
$n_{L-1}+1$ affinely independent points. Set
$
Y\coloneqq f_{\theta\up{L-1}}(\mathbb R^{n_0}).
$

For the last hidden layer and the output layer, partition the neurons of layer $L$ into
$
J \coloneqq [n_L-2],
$ and $
K \coloneqq \{n_L-1,n_L\}.
$
We think of the neurons in $J$ as \emph{bending neurons} and those in $K$ as
\emph{linear neurons}.

Choose the parameters $(\weights{L}_J,\bias{L}_J)$ of the bending neurons so that they
form a generic oriented slab layer on $Q\up{L-1}$ intersecting all the bent hyperplanes
from layer $L-1$. Since $n_L-2\ge 2$, this can be achieved by the same slab-layer
construction used in \Cref{thm:existence_TPIC_LRA}.

Next choose the parameters $(\weights{L}_K,\bias{L}_K)$ of the two linear neurons so
that their augmented input matrix
$
W_{\mathrm{lin}}
\coloneqq
\begin{bmatrix}
\weights{L}_K & \bias{L}_K
\end{bmatrix}
\in \R^{2\times (n_{L-1}+1)}
$
has strictly positive entries and rank $2$. Since
$
Y\subseteq \R_{\geq 0}^{n_{L-1}},
$
these two preactivations are strictly positive on $Y$. Hence the two neurons in $K$ are
active on the entire prefix image and contribute only an affine-linear map on the realized
domain of the prefix. Choose also
$
V_{\mathrm{lin}}
\coloneqq
\weights{L+1}_{:,K}
\in \R^{n_{L+1}\times 2}
$
generically of rank $2$.

This completes the construction of $\theta$.

\medskip
\noindent
\textbf{Positive-dimensional fiber.}
The contribution of the two linear neurons is the affine map represented by
$
B_{\mathrm{lin}}=V_{\mathrm{lin}}W_{\mathrm{lin}}.
$
For any $M\in \GL_2(\R)$ sufficiently close to the identity such that
$MW_{\mathrm{lin}}$ still has strictly positive entries, define
$
W'_{\mathrm{lin}}=MW_{\mathrm{lin}},
$ and $
V'_{\mathrm{lin}}=V_{\mathrm{lin}}M^{-1}.
$
Then the transformed preactivations are again strictly positive on $Y$, and the affine
contribution of the two-neuron block is unchanged:
$
V'_{\mathrm{lin}}W'_{\mathrm{lin}}
=
V_{\mathrm{lin}}M^{-1}MW_{\mathrm{lin}}
=
V_{\mathrm{lin}}W_{\mathrm{lin}}.
$
Therefore this transformation produces a local $\GL_2(\R)$-family of parameters
realizing the same function. Modulo the usual positive scaling symmetry of the two
neurons, this family has dimension
$
\dim(\GL_2(\R))-\dim((\R_{>0})^2)=4-2=2.
$
Hence the fiber of $f_\theta$ in the quotient parameter space is positive-dimensional.

\medskip
\noindent
\textbf{Image dimension.}
The properties used above are open:
\begin{itemize}
    \item cTPIC and LRA for the prefix on $P$,
    \item transversality of the intersections between layers $L-1$ and $L$,
\item affine spanning of the prefix image,
\item strict positivity of the two linear neurons on the prefix image,
\item the rank-$2$ conditions on $W_{\mathrm{lin}}$ and $V_{\mathrm{lin}}$.
\end{itemize}
Hence these conditions persist on an open neighborhood $U_0$ of the constructed
parameter $\theta$, and for every $\eta\in U_0$ we still have
$
\dim\bigl(\overline{\mu}_{\architecture}^{-1}(f_\eta)\bigr)\ge 2.
$
Now let
$
U=U_0\cap \gparamspace{\architecture}.
$
We show that
$
\dim(\mu_\architecture(U))
=
\dim(\mu_\architecture(\paramspace{\architecture}))-2.
$

Let $\eta,\eta'\in U$ satisfy
$
f_\eta=f_{\eta'}.
$
Since $\eta,\eta'$ are generic and in $U_0$, by \Cref{lem:identifiability_split}, the
first $L-1$ layers of $\eta$ and $\eta'$ coincide up to the trivial symmetries. Likewise,
by the reverse-engineering algorithm of \cite{rolnick2020reverse}
(\Cref{thm:tpic_generic_identifiable_inductive}), the parameters of the bending neurons
in $J$ coincide up to the trivial symmetries and finitely many sign choices. In particular,
they contribute no positive-dimensional fiber directions. After shrinking $U$ if necessary, the finite sign choices for the bending neurons are fixed, and the above local description applies uniformly on $U$.

It remains to control the always-active block. After applying the trivial symmetries to identify the first $L-1$ layers of
$\eta$ and $\eta'$, let
$
Y_\eta \coloneqq f_{\eta\up{L-1}}(\mathbb R^{n_0}).
$
By the openness assumptions defining $U_0$, this set affinely spans
$\mathbb R^{n_{L-1}}$. Therefore equality of realized functions forces the affine
map contributed by the two always-active neurons to be fixed as an affine map on
$\mathbb R^{n_{L-1}}$.  Equivalently, in the notation above, any local variation in the
fiber must satisfy
$V_{\mathrm{lin}}^{\eta'}W_{\mathrm{lin}}^{\eta'}
=
V_{\mathrm{lin}}^{\eta}W_{\mathrm{lin}}^{\eta}.$
Since both $W_{\mathrm{lin}}^{\eta}$ and $V_{\mathrm{lin}}^{\eta}$ have rank $2$, all nearby rank-$2$
factorizations of this fixed matrix through a two-dimensional hidden space are of the
form
$
W_{\mathrm{lin}}^{\eta'}=MW_{\mathrm{lin}}^{\eta},
$ and $
V_{\mathrm{lin}}^{\eta'}=V_{\mathrm{lin}}^{\eta}M^{-1}
$
for some $M\in\GL_2(\R)$. Thus, modulo the trivial symmetries, the only local continuous
fiber directions in $U$ are the two directions coming from the $\GL_2$-factorization
freedom of the always-active two-neuron block.

Hence the quotient realization map has local fiber dimension $2$ on $U$. By the fiber-dimension theorem for semialgebraic maps,
$
\dim \mu_{\architecture}(U)
=
\dim(\overline{\Theta}_{\architecture})-2.
$
Since
$
\dim(\overline{\Theta}_{\architecture})
=
\dim \mu_{\architecture}(\Theta_{\architecture}),
$
we conclude that
$
\dim \mu_{\architecture}(U)
=
\dim \mu_{\architecture}(\Theta_{\architecture})-2.
$
This completes the proof.
\end{proof}

The previous proposition isolates the two ingredients needed for the minimality argument.
First, it produces an open set with positive-dimensional fibers, arising from the
$\GL_2$-symmetry of the always-active linear block. Second, it shows that the image of
this family has codimension $2$ inside the full function space of the ambient
architecture. 
\subsection{Dimension Drop and Minimality}
To obtain minimality, we now compare this image with the function spaces of
strict subarchitectures. Since every strict subarchitecture has strictly smaller
functional dimension, indeed, at least three dimensions less, the union of their images is
lower-dimensional than the image of the constructed set. Removing this exceptional
union therefore leaves a nonempty open subset of functions that are not realizable by any
strict subarchitecture, while their fibers remain positive-dimensional.
\begin{lemma}
\label{lem:strict_subarchitecture_gap}
Let $\architecture=(n_0,\ldots,n_L,n_{L+1})$ with  $L\ge 2$ and $n_i\ge 2$ for all $i\le L$.
If $\architecture'=(n_0,n_1',\ldots,n_L',n_{L+1})$ is a strict subarchitecture of
$\architecture$, then
$
\dim \mu_{\architecture'}(\Theta_{\architecture'})
\le
\dim \mu_{\architecture}(\Theta_{\architecture})-3.
$
\end{lemma}
\begin{proof} 
It suffices to consider the case where $\architecture'$ is obtained from $\architecture$
by decreasing one hidden layer width $n_k$ by one, since the general case follows by iterating
this estimate.

So let
$
\architecture'=(n_0,\ldots,n_{k-1},n_k-1,n_{k+1},\ldots,n_L,n_{L+1}),
$ for some $ k\in[L].
$
Then
\begin{align*}
\dim \mu_{\architecture}(\Theta_{\architecture})
-
\dim \mu_{\architecture'}(\Theta_{\architecture'})
&\geq
\bigl(n_k n_{k-1}+n_{k+1}n_k\bigr)
-
\bigl((n_k-1)n_{k-1}+n_{k+1}(n_k-1)\bigr) \\
&=
n_{k-1}+n_{k+1}.
\end{align*} The inequality stems from the fact that in case in $\architecture'$ we have a layer of width $1$, then the expected functional dimension is only an upper bound.
Since $n_{k-1}\ge 2$ and $n_{k+1}\ge 1$, we obtain
$
n_{k-1}+n_{k+1}\ge 3
$, which proves the claim.
\end{proof}
\begin{theorem}
    For architectures $\architecture=(n_0,\ldots,n_L,n_{L+1})$ with  $L\ge 2$ and $n_i \geq 2$ for $i \in \{0,\ldots, L+1\}$ and $n_L \geq 4$, there is a nonempty open subset $U \subseteq \paramspace{\architecture}$ such that every $\theta \in U$ is minimal but has a positive-dimensional fiber, 
$
\dim(\overline{\mu}_\architecture^{-1}(f_\theta)) >0.
$
\end{theorem}

\begin{proof}
By \Cref{prop:open_set_two_dimensional_fiber_drop}, there exists a nonempty open subset
$U_0\subseteq \Theta_{\architecture}$ such that for every $\theta\in U_0$,
$
\dim\bigl(\overline{\mu}_{\architecture}^{-1}(f_\theta)\bigr)\ge 2,
$
and
$
\dim \mu_{\architecture}(U_0)
=
\dim \mu_{\architecture}(\Theta_{\architecture})-2.
$
Let
$
\mathcal R
\coloneqq
\bigcup_{\architecture' < \architecture}
\mu_{\architecture'}(\Theta_{\architecture'})
\subseteq
\mu_{\architecture}(\Theta_{\architecture}),
$
where the union runs over all strict subarchitectures $\architecture'$ of
$\architecture$.
By \Cref{lem:strict_subarchitecture_gap}, each strict subarchitecture satisfies
$
\dim \mu_{\architecture'}(\Theta_{\architecture'})
\le
\dim \mu_{\architecture}(\Theta_{\architecture})-3.
$
Since there are only finitely many strict subarchitectures of $\architecture$, it follows
that
$
\dim \mathcal R
\le
\dim \mu_{\architecture}(\Theta_{\architecture})-3.
$
On the other hand,
$
\dim \mu_{\architecture}(U_0)
=
\dim \mu_{\architecture}(\Theta_{\architecture})-2.
$
Hence $\mathcal R\cap \mu_{\architecture}(U_0)$ is a lower-dimensional subset of
$\mu_{\architecture}(U_0)$. Therefore there exists a nonempty relatively open subset
$
V\subseteq \mu_{\architecture}(U_0)\setminus \mathcal R.
$
Since $V$ is relatively open in $\mu_{\architecture}(U_0)$, the preimage
$
U\coloneqq U_0\cap \mu_{\architecture}^{-1}(V)
$
is a nonempty open subset of $U_0$.
Now let $\theta\in U$. Since $f_\theta\in V\subseteq \mu_{\architecture}(U_0)\setminus
\mathcal R$, the function $f_\theta$ is not realized by any strict subarchitecture of
$\architecture$. Hence $\theta$ is minimal.

Moreover, because $U\subseteq U_0$, we still have
$
\dim\bigl(\overline{\mu}_{\architecture}^{-1}(f_\theta)\bigr)\ge 2>0.
$
Thus every parameter in $U$ is minimal and has positive-dimensional fiber. 
\end{proof}
While the proof above establishes minimality by a dimension argument and therefore does not require a more explicit structural analysis, \Cref{app:minimal_layers} provides a complementary one-hidden-layer perspective explaining why the linear neurons in the last hidden layer are nonetheless necessary, despite not producing additional visible breakpoints.

\section{Generic Depth Hierarchy}
\label{sec:depth_hierarchy}

A significant challenge in neural network theory is to determine which functions are representable at fixed depth and whether there are functions representable by deep networks that are not at all representable by shallower ones, regardless of the width \citep{bakaev2026betterneuralnetworkexpressivity,grillo2025depthbounds,averkov2025on,haase2023lower,safran2026depthhierarchycomputingmaximum,doi:10.1137/22M1489332}. 

Most existing results and approaches rely on a representation of CPWL functions that reduces the depth separation question to determining the depth required to compute the maximum of $d$ numbers. 
Our polyhedral framework, in contrast, allows us to study depth separation in the generic case. 
By leveraging the dependency graph as a functional invariant, we can show that, on an open subset of parameters, the layer hierarchy is structurally rigid. 
Unlike degenerate cases, where layer hierarchy might ``collapse'' due to exact cancellations or alignment of breakpoints, 
the functional geometry of a generic network encodes its depth. 
The following proposition formalizes this observation by showing that, for each architecture, there is an open set of parameters whose realized functions cannot be represented by any shallower architecture with generic parameters. 
For example, the function illustrated in \Cref{fig:slab_layer_2} cannot be represented with $2$ hidden layers and generic weights. While a
related statement is likely already implicit in the framework of
\citet{phuong2020functional}, we find it useful to state it explicitly in this generality.

\begin{prop}
    For any architecture $\architecture=(n_0,\ldots,n_L,m)$ with $n_i \geq 2$ for $i \in \{0,\ldots,L\}$, there exists a nonempty open subset of generic parameters $U$ such that, for all $\theta \in U$, the function $f_\theta$ cannot be represented by any network of architecture $\architecture'=(n_0,n_1',\ldots,n_{L-1}',m)$ with generic parameters. 
\end{prop}

\begin{proof}
Let $\theta$ be a parameter constructed in the proof of \Cref{thm:existence_TPIC_LRA}.
By construction, for each $\ell\in[L-1]$ there exists a candidate bent hyperplane
$v_\ell$ arising from layer $\ell$ and a candidate bent hyperplane $v_{\ell+1}$
arising from layer $\ell+1$ such that $(v_\ell,v_{\ell+1})\in E_{f_\theta}$.
Hence the dependency graph contains a chain
$
v_1 \to v_2 \to \cdots \to v_L.
$
where each $v_\ell$ is associated with a neuron in hidden layer $\ell$. 

    Suppose there exists a generic parameter $\eta$ in an architecture $\architecture'$ with $L' < L$ hidden layers such that $f_\eta = f_\theta$. According to \Cref{prop:embedd_candidate_bh}, there must exist a map $\phi_\eta \colon V_{f_\theta} \to V_{\architecture'}$ such that, for every edge $(u, v) \in E_{f_\theta}$, the corresponding hidden layers satisfy $\ell_{\phi(u)} < \ell_{\phi(v)}$. 
    Applying this property to the chain in our dependency graph, the assigned layer indices in the architecture $\architecture'$ must satisfy: 
    \(
    \ell_{\phi(v_1)} <  \dots < \ell_{\phi(v_L)}.
    \)
    This condition requires at least $L$ distinct layer indices. 
    However, an architecture with only $L' < L$ hidden layers provides at most $L'$ available layer indices. 
    Since a strictly increasing sequence of $L$ integers cannot be contained in the set $\{1, \dots, L'\}$,  there is no generic parameter in $\paramspace{\architecture'}$ that computes $f_\theta$. The chain
$
v_1\to\cdots\to v_L
$
in the dependency graph arises from transverse visible intersections between adjacent layers in the constructed parameter. These intersections persist under sufficiently small perturbations, and hence the same chain exists for all parameters in some open neighborhood $U$ of $\theta$. Therefore the above obstruction applies to every parameter in $U$.
\end{proof}

    Without the genericity assumption, it is no longer guaranteed that one can  
    determine the layer ordering of the neurons associated to the different facets of the breakpoint complex of $f_\theta$, 
    as illustrated by the following example. 
\begin{figure}
\begin{subfigure}[t]{0.22\textwidth}
    
\begin{tikzpicture}[scale=0.6]
\footnotesize

\draw[line width=0.9pt] (1,0)--(4,0);

\draw[line width=0.9pt] (1,0) -- (4,3);
\draw[line width=0.9pt] (1,0) -- (4,1.5);

\draw node at (5,2) {$x_2-x_1+1$};
\draw node at (4,0.5) {$-x_2$};
\draw node at (-0.3,-0.3) {$0$};
\draw node at (2,-1) {$0$};
\draw node at (1,2) {$0$};
\end{tikzpicture}
\caption{The function}
\label{fig:the function}
\end{subfigure}
\hspace{0.2cm}
\begin{subfigure}[t]{0.22\textwidth}

\begin{tikzpicture}[scale=0.35]
\footnotesize

\draw[dashed,line width=0.9pt] (-4,0) -- (1,0);
\draw[dashed,line width=0.9pt] (0,-4) -- (0,4);

\draw[line width=0.9pt] (1,0)--(4,0);
\draw[red!70!black,line width=0.9pt] (1,0) -- (4,3);
\draw[red!70!black,line width=0.9pt] (1,0) -- (4,1.5);
\draw[red,dashed,line width=0.9pt] (1,0) -- (1,-3);

\end{tikzpicture}
\caption{One network computing the function}
\label{fig:non-generic-intersections}
\end{subfigure}
\hspace{0.2cm}
\begin{subfigure}[t]{0.22\textwidth}
    
\begin{tikzpicture}[scale=0.35]
\footnotesize

\draw[dashed,line width=0.9pt] (-3,3) -- (3,-3);
\draw[dashed,line width=0.9pt] (-3,-4) -- (1,0);
\draw[line width=0.9pt] (1,0)--(4,3);

\draw[red!70!black,line width=0.9pt] (1,0) -- (4,1.5);
\draw[red!70!black,line width=0.9pt] (1,0) -- (4,0);
\draw[red,dashed,line width=0.9pt] (1,0) -- (-1,2);

\end{tikzpicture}
\caption{Another network computing the function}
\end{subfigure}
\begin{subfigure}[t]{0.22\textwidth}

\begin{tikzpicture}[scale=0.35]
\footnotesize

\draw[dashed,line width=0.9pt] (-3,-1) -- (4,-1);
\draw[dashed,line width=0.9pt] (-3,-2) -- (1,0);
\draw[red!70!black,line width=0.9pt] (1,0)--(4,3);

\draw[line width=0.9pt] (1,0) -- (4,1.5);
\draw[red!70!black,line width=0.9pt] (1,0) -- (4,0);

\end{tikzpicture}
\caption{Another network computing the function}
\label{fig:yet another network}
\end{subfigure}
\caption{Canonical complexes for  \Cref{ex:not-generic}.  
Dashed lines indicate breakpoints that are not visible from the function alone, i.e., faces of the canonical complex that are not part of the breakpoint complex. Red denotes breakpoints from the second layer, and black from the first. 
}
\end{figure}
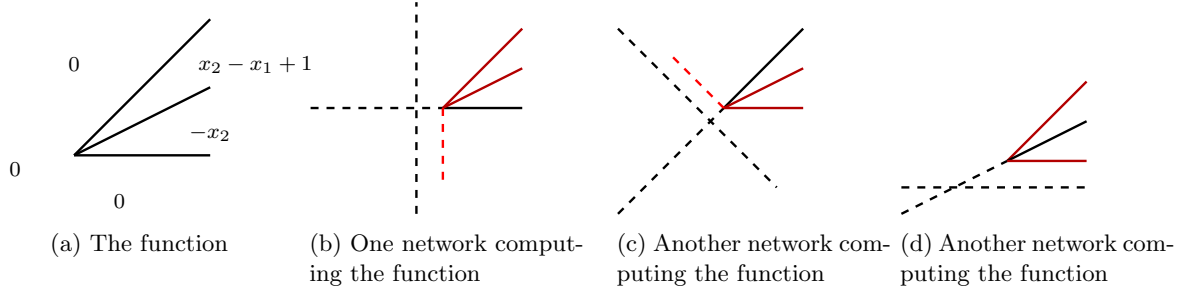

\begin{example}
\label{ex:not-generic}
    Consider the function $f \colon \R^2 \to \R$ given by 
\(
x \mapsto \min\{0, \max\{x_2 - x_1 + 1, -x_2\}\},
\)
whose breakpoint complex $\bcomplex{f}$ consists of three rays emanating from the vertex $(1,0)$, as shown in \Cref{fig:the function}. 

For each of these three rays, there exists a parameter $\eta$ realizing $f$ in such a way that the chosen ray originates from a first-layer neuron, while the other two rays originate from the second hidden layer. 
This demonstrates that the internal hierarchy of the network is not a functional invariant for non-generic parameters. 
The polyhedral subdivisions for these three different realizations are shown in \Cref{fig:non-generic-intersections}--\Cref{fig:yet another network}, and the respective weights and biases are given below. 
\begin{enumerate}
    \item The first layer hyperplanes are the axes $x_1=0$ and $x_2=0$ 
    \begin{align*} 
    \weights{1}= \begin{pmatrix} 1 & 0 \\ 0 & 1 \end{pmatrix}, \bias{1}= \begin{pmatrix} 0 \\ 0 \end{pmatrix}, \quad 
    \weights{2}= \begin{pmatrix} 1 & -1 \\ 1 & -2 \end{pmatrix}, \bias{2} = \begin{pmatrix} -1 \\ -1 \end{pmatrix}, \quad 
    \weights{3}= \begin{pmatrix} -1 & 1 \end{pmatrix} . 
    \end{align*}

    \item  The first layer hyperplanes are $x_1 - x_2 - 1 = 0$ and $x_1 + x_2 = 0$
    \begin{align*} 
    \weights{1}= \begin{pmatrix} 1 & 1 \\ 1 & -1 \end{pmatrix}, \bias{1}= \begin{pmatrix} 0 \\ -1 \end{pmatrix}, \quad 
    \weights{2} = \begin{pmatrix} 1/2 & -1/2 \\ 1/2 & -3/2 \end{pmatrix}, \bias{2} = \begin{pmatrix} -1/2 \\ -1/2 \end{pmatrix}, \quad 
    \weights{3}= \begin{pmatrix} -1 & 1 \end{pmatrix} . 
    \end{align*}

    \item  The first layer hyperplanes are $x_2+1=0$ and $-x_1 + 2x_2 + 1 = 0$ 
    \begin{align*} 
    \weights{1}= \begin{pmatrix} 0 & 1 \\ -1 & 2 \end{pmatrix}, \bias{1}= \begin{pmatrix} 1 \\ 1 \end{pmatrix}, \quad 
    \weights{2} = \begin{pmatrix} 1 & -1 \end{pmatrix}, \bias{2} = -1, \quad 
    \weights{3}= -1 . 
    \end{align*}
\end{enumerate}
Consequently, for non-generic parameters, the geometry of the breakpoint complex alone is insufficient to identify the layer from which a specific non-linearity originates.
\end{example}

\section{Open Questions}

The polyhedral framework introduced in this work establishes 
the existence of identifiable parameters and minimal but non-identifiable parameters for deep ReLU networks, and describes the local functional invariants. 
However, several important questions remain open for future research: 

\begin{enumerate}

    \item \textbf{Measure of the Identifiable Set:} While we have shown that, for architectures with non-output layers of width at least $2$, both the sets of identifiable parameters and non-identifiable parameters are open, the volume and
    distribution of the identifiable set %
    are not yet well understood. 
 
Such an investigation will likely require identifying conditions beyond the relatively strong TPIC that still guarantee identifiability. 
In particular, characterizing the boundary at which identifiability fails is an important open problem. 
The problem is of interest for both strict identifiability and finite identifiability. 
    
\item \textbf{Characterizing Fibers and Redundancies in Deep Networks:}
   While our results identify open regions of identifiable parameters and open regions with positive-dimensional fibers, the global structure of the fibers of the realization map remains poorly understood. A major open problem is to characterize, as explicitly as possible, the set of parameters representing a given function, including both its continuous and discrete redundancies. This includes finding useful semi-algebraic descriptions of fibers, understanding how visible breakpoint geometry constrains hidden cancellations, and classifying the nontrivial symmetries that can persist in minimal deep ReLU networks. A first step in this direction is given in \Cref{app:alg_superset}, where we describe an algebraic superset of the generic fiber obtained by fixing the combinatorial assignment of candidate bent hyperplanes and activation patterns.
    
    \item \textbf{Identifiability for Width $1$:} Our results apply to architectures with non-output layers of width at least $2$. 
    As discussed in \Cref{rem:width1}, if an architecture contains layers of width $1$, the activation becomes one-dimensional, making it impossible to satisfy TPIC in subsequent layers. 
    It remains an open problem whether such architectures have any identifiable parameters, or whether identifiability can hold at best in a finite sense.

    \item \textbf{Identifiability Over Finite Data:} We investigated identifiability of the parameters in relation to the represented functions. Also of interest is the investigation of identifiability when one only has access to certain measurements taken from the functions, such as the function values on some finite input data set, or a list of moments of the outputs over an input data distribution. 

\end{enumerate}

\subsection*{Acknowledgments} 

We would like to thank Elisenda Grigsby for insightful discussions on the depth hierarchy for generic parameters and for bringing to our attention the example of \cite{ramakrishnan2026completesymmetryclassificationshallow} showing that TPIC and LRA are not sufficient for identifiability. 
This project has been supported by the
Deutsche Forschungsgemeinschaft (DFG, German Research Foundation)  project 464109215 
within the priority programme SPP 2298 ``Theoretical Foundations of Deep Learning''.  
GM was partially supported by 
DARPA AIQ grant HR00112520014, 
NSF grants 
DMS-2522495, 
DMS-2145630, 
CCF-2212520, 
and BMFTR in DAAD project 57616814 (SECAI).

\newpage
	\bibliographystyle{unsrtnat}
\bibliography{ref}

\newpage

\appendix

\section{Remarks on Previous Identifiability Results and Constructions}
This appendix clarifies two points relevant to comparison with prior work: first, the distinction between uniqueness within restricted parameter classes and identifiability in the full parameter space; second, the way in which our construction differs from previous reverse-engineering-based constructions.
\subsection{On restricted notions of identifiability in prior work}
\label{app:restricted-identifiability}

We briefly clarify a point about the interpretation of previous identifiability results for deep ReLU networks. 
There are several different notions of uniqueness that one may consider. 
Given a parameter $\theta$, one may ask whether $f_\eta=f_\theta$ implies $\eta\sim\theta$
\begin{enumerate}
    \item among all parameters $\eta$,
    \item among all generic parameters $\eta$, or
    \item only among parameters $\eta$ belonging to some restricted class, for example parameters satisfying geometric conditions such as TPIC and LRA.
\end{enumerate}
These are progressively weaker notions of uniqueness and should not be conflated. In this hierarchy, the result of \citet{rolnick2020reverse} is of the third type, being formulated on a restricted class of parameters satisfying geometric assumptions such as TPIC and LRA, while the result of \citet{phuong2020functional} is of the second type, establishing uniqueness among generic parameters. 
That argument of \cite{rolnick2020reverse} reconstructs parameters under assumptions such as TPIC and LRA, and therefore yields uniqueness only within the class of parameters satisfying those assumptions. 
Technically, this does not imply identifiability among all generic parameters, since a second generic realization of the same function need not a priori satisfy TPIC and LRA; nor does it imply identifiability in the full parameter space. 
In particular, uniqueness inside the TPIC/LRA class does not exclude the possibility of additional realizations outside that class.

This subtlety also affects how one interprets later results that build on the reverse-engineering framework such as the work of \cite{grigsby2023hiddensymmetriesrelunetworks}. 
For example, if one constructs a parameter satisfying TPIC and LRA and then invokes the reverse-engineering theorem, the conclusion obtained directly is uniqueness only among other parameters satisfying the same conditions. 
An additional argument is required to transfer this to uniqueness among all generic parameters, and a further argument is required to upgrade this to identifiability among all parameters. 
In our proof, the first upgrade is achieved using the breakpoint complex and dependency graph, which show that cTPIC and LRA are inherited across generic realizations of the same function; the second upgrade is achieved by a dimension argument excluding the image of the nongeneric locus on a nonempty open subset.

\subsection{Comparison of our construction to prior constructions}
\label{rem:grigsby-gap} 
\label{sec:grigsby-gap}

Our construction of identifiable parameters for deep ReLU networks is inspired 
by the construction of \cite{grigsby2023hiddensymmetriesrelunetworks}. 
There are, however, important differences between the two approaches. 

We begin by recalling the main ideas of their construction, which proceeds iteratively, layer by layer. 
A key ingredient is the notion of a \emph{positive-axis hyperplane}, namely a hyperplane in the image space of the preceding layer
whose defining affine function has strictly positive coefficients and sufficiently large negative bias. 
Such hyperplanes intersect every ray in the image of a ReLU layer and are used to enforce intersections with the image of the hyperplanes within a distinguished unbounded polyhedron. 
The construction maintains, at each step, a distinguished unbounded polyhedron in the canonical polyhedral complex whose image under the truncated network remains full-dimensional in the positive orthant. 
Subsequent layers are chosen by perturbing a high-bias positive-axis hyperplane so that its pullback intersects all bent hyperplanes from the previous layer. 
This iterative procedure is designed to enforce pairwise transverse intersections between bent hyperplanes of subsequent layers, enabling the application of the reverse-engineering framework of \cite{rolnick2020reverse}.

Our construction is similar in that it also proceeds iteratively and designs each layer to enforce (c)TPIC and LRA. 
However, instead of working with positive-axis hyperplanes and unbounded polyhedra, we work with bounded polytopes and slab layers. 
Moreover, we incorporate transparency together with cancellation-freeness to ensure that the relevant intersections remain visible in the final realized function. 

A key difference is that the inductive construction of \cite{grigsby2023hiddensymmetriesrelunetworks} does not appear to guarantee the linear regions assumption (LRA) in a neighborhood of the relevant intersections for the final realized function. 
While their construction ensures pairwise transverse intersections in the canonical bent-hyperplane arrangement, this does not by itself imply that these intersections remain visible as breakpoints of the realized function. 
Indeed, later layers may be inactive in a neighborhood of such intersections, 
so that the affine linear maps on the adjacent polyhedra of the canonical complex coincide in the final output (and are in fact constant). 
In that case, the intersection exists in the canonical complex but is absent from the breakpoint complex, and LRA fails near that point. 

This phenomenon is illustrated in \Cref{fig:counterxample} for the architecture $(2,2,2,2)$. 
In the left panel, the pullback of a positive-axis hyperplane along the first layer produces a bent hyperplane in the input space. 
In the middle panel, perturbing this hyperplane yields two second-layer bent hyperplanes. 
As seen in the figure, only one orientation of these hyperplanes is compatible with the inductive choice of an unbounded polyhedron $S_2$. 
In the right panel, the same procedure is applied to the third layer. 
Although the blue bent hyperplanes intersect those from earlier layers in the canonical complex, the third layer is inactive near the point $\tau$, so this point is not a breakpoint of the final function. 
Consequently, the local four-region configuration used in \citealp[Lemmas D.11 and D.14 in the work of ][]{grigsby2023hiddensymmetriesrelunetworks} need not correspond to four distinct linear regions of the realized function. 
In particular, an intersection between two bent hyperplanes from the first two hidden layers is not visible as breakpoint in the final function.

In our framework, this issue is addressed by explicitly distinguishing between the canonical polyhedral complex and the breakpoint complex. 
The missing ingredient in the earlier construction is a mechanism that prevents later layers from canceling previous visible nonlinearities. 
To address this, our proof of \Cref{thm:existence_TPIC_LRA} imposes transparency and leverages generic cancellation-freeness to ensure LRA via the structure of tropical weights (\Cref{prop:tropweightofNNs}).

\begin{figure}[t]
\centering
\begin{subfigure}[t]{0.31\textwidth}
\centering
\begin{tikzpicture}[scale=0.25,x=0.9cm,y=0.9cm,line cap=round,line join=round]
    \footnotesize
    \draw[black!80, line width=0.8pt] (-6,0) -- (15,0) node[right] {$H_2$};
    \draw[black!80, line width=0.8pt] (0,-6) -- (0,15)  node[above]{$H_1$};
    \draw[->,>=stealth,black!80] (0,2) -- (2,2);
    \draw[->,>=stealth,black!80] (7,0) -- (7,2);
    \filldraw[gray,opacity=0.2] (0,0)--(15,0)  -- (0,15)  -- cycle; 
    \draw[red!70!black, line width=0.9pt] (-6,12) -- (0,12) -- (3,0) -- (3,-6) node[below]{$f_{\theta_1}^{-1}(H')$};
\end{tikzpicture}
\caption{The first hidden layer consists of the coordinate-axis hyperplanes $H_1$ and $H_2$. The pullback of a positive-axis hyperplane $H'$ in the image is a bent hyperplane in the input space.}
\label{fig:positive_axis_hyperplanes_a}
\end{subfigure}
\hfill
\begin{subfigure}[t]{0.31\textwidth}
\centering
\begin{tikzpicture}[scale=0.25,x=0.9cm,y=0.9cm,line cap=round,line join=round]
    \footnotesize
    \draw[black!80, line width=0.8pt] (3,0) -- (15,0);
    \draw[black!80, line width=0.8pt] (0,4) -- (0,15)  node[above]{$H_1$};
    \draw[black!80, dotted, line width=0.8pt] (-6,0) -- (3,0);
    \draw[black!80,dotted, line width=0.8pt] (0,-6) -- (0,4);

    \draw[red!70!black, line width=0.9pt] (-6,12) -- (0,12) -- (3,0) -- (3,-6);
    \draw[red!70!black, line width=0.9pt] (-6,4) -- (0,4) -- (12,0) -- (12,-6) node[below]{$B_1$};

    \filldraw[gray,opacity=0.2] (24/11,36/11)--(12,0) -- (15,0) -- (0,15) -- (0,12) -- cycle;  
    \draw[->,>=stealth,red!70!black] (-3,12) -- (-3,14);
    \draw[->,>=stealth,red!70!black] (-4,4) -- (-4,6);
    \node[above right] at (-2,2) {$\tau$};
    \fill[black!80] (0,4) circle (6pt);
    \filldraw[gray,opacity=0.2] (0,0)--(15,0)  -- (0,15)  -- cycle; 
\end{tikzpicture}
\caption{After perturbation, the second-layer bent hyperplanes are shown in red. The dark gray region is the distinguished unbounded polyhedron $S_2$ used in the inductive construction.}
\label{fig:positive_axis_hyperplanes_b}
\end{subfigure}
\hfill
\begin{subfigure}[t]{0.31\textwidth}
\centering
\begin{tikzpicture}[scale=0.25,x=0.9cm,y=0.9cm,line cap=round,line join=round]
    \footnotesize
    \draw[black!80, line width=0.8pt] (6,0) -- (15,0);
    \draw[black!80, line width=0.8pt] (0,7) -- (0,15);
    \draw[black!80, dotted, line width=0.8pt] (-6,0) -- (6,0);
    \draw[black!80,dotted, line width=0.8pt] (0,-6) -- (0,7);

    \draw[red!70!black,line width=0.9pt] (-6,12) -- (0,12) -- (15/11, 72/11);
    \draw[red!70!black, dotted, line width=0.9pt] (15/11, 72/11) -- (3,0) -- (3,-6);
    \draw[red!70!black, dotted, line width=0.9pt] (-6,4) -- (0,4) --(60/11, 24/11);
    \draw[red!70!black, line width=0.9pt] (60/11, 24/11) -- (12,0) -- (12,-6);

    \filldraw[gray,opacity=0.2] (24/11,36/11)--(12,0) -- (15,0) -- (0,15) -- (0,12) -- cycle;  
    \filldraw[gray,opacity=0.2] (0,0)--(15,0)  -- (0,15)  -- cycle; 

    \draw[blue, line width=1pt]
        (-6,7) -- (0,7)
        -- (15/11, 72/11)
        -- (105/11, 9/11)
        -- (39/4, 0)
        -- (39/4, -6);

    \draw[blue!70!black, line width=1pt]
        (-6,10) -- (0,10)
        -- (6/11, 108/11)
        -- (60/11, 24/11)
        -- (6,0)
        -- (6,-6);

    \draw[->,>=stealth,blue!70!black] (-4,10) -- (-4,11.5);
    \draw[->,>=stealth,blue!70!black] (-5,7) -- (-5,8.5);
    \node[above right] at (-2,2) {$\tau$};
    \fill[black!80] (0,4) circle (6pt);
     \filldraw[gray,opacity=0.2]
        (0,15) --
        (0,12) --
       (6/11, 108/11)--
        (735/209,1053/209) --
        (105/11, 9/11)--
        (12,0) --
        (15,0) -- cycle;
\end{tikzpicture}
\caption{The same step is repeated for the third layer (blue). Although the new bent hyperplanes intersect the earlier ones in the canonical complex, the point $\tau$ is no longer visible as a breakpoint of the final function.}
\label{fig:kill_TPIC}
\end{subfigure}

\caption{Illustration of the inductive construction of \cite{grigsby2023hiddensymmetriesrelunetworks} for the architecture $(2,2,2,2)$. 
Black lines are the first-layer hyperplanes, red lines the second-layer bent hyperplanes, and blue lines the third-layer bent hyperplanes. 
The shaded regions indicate the distinguished unbounded polyhedra used in the construction. 
Dashed segments represent pieces of bent hyperplanes that are present in the canonical complex but do not survive as breakpoints of the final function.} 
\label{fig:counterxample}
\end{figure}
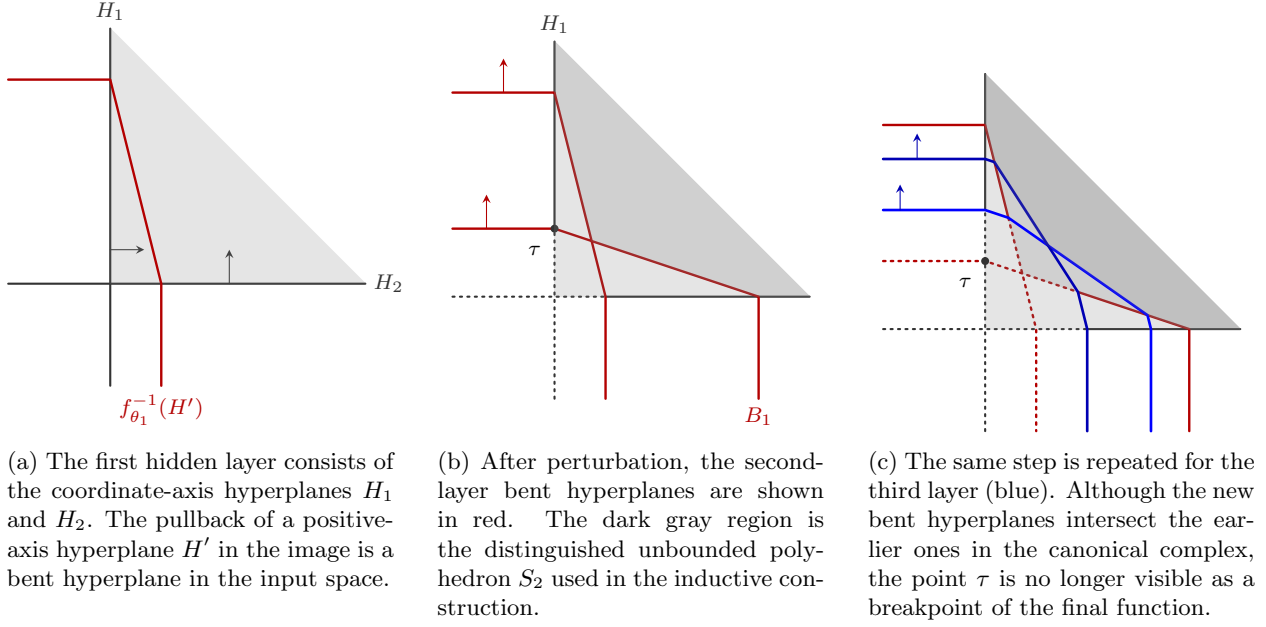

\section{Observations on Minimal Layer Width}
\label{app:minimal_layers}

The following one-hidden-layer lemmas collected in this section observe the effect of ``linear" neurons on minimality: neurons that do not induce visible breakpoint hyperplanes may still be
necessary. Once the visible nonlinear structure is fixed, the remaining discrepancy
between two realizations is forced to be affine, and the rank of this affine part can
require additional hidden neurons that are invisible from the breakpoint geometry alone.
In this sense, affine contributions can enforce minimal width even when they do not
create any new visible breakpoints. 

\begin{lemma}
\label{lem:fixed_width_affine_difference}

Let $P\subseteq \R^d$ be a polytope, and let
$
f_\theta(x)=\weights{2}[\weights{1}x+\bias{1}]_+ + \bias{2}
$
be a one-hidden-layer network with $n$ hidden neurons such that the breakpoint set of $f_\theta$ in $\relint(P)$ is the union of $n$ pairwise distinct hyperplanes
$
H_i=\{x\in\R^d\mid \weights{1}_i x+\bias{1}_i=0\},
$ for $ i\in[n].
$
Let
$
f_\eta
$
be another function realized by a  one-hidden-layer network with $n$ hidden neurons whose breakpoint hyperplanes intersecting $\relint(P)$ are the same hyperplanes $H_1,\dots,H_n$. 
If $f_\theta-f_\eta$ is affine linear on $P$, then the linear part of $f_\theta-f_\eta$ on $P$ is of the form
$
\sum_{i=1}^n \alpha_i\,\weights{2}_{:,i}\weights{1}_i
$ with $\alpha_i\in\{-1,0,1\}.
$
\end{lemma}

\begin{proof}
Since both networks have exactly $n$ neurons and exactly the same $n$ pairwise distinct breakpoint hyperplanes in $\relint(P)$, after reordering the neurons of $f_\eta$ we may assume that for each $i\in[n]$, the $i$th neuron of $f_\eta$ has breakpoint hyperplane $H_i$.
Let $f_\eta(x)=V^{(2)}[V^{(1)}x+c^{(1)}]_+ + c^{(2)}$.
For each $i$, the $i$th neuron of $f_\eta$ therefore has preactivation of the form
$
V\up{1}_i x + c^{(1)}_i = \lambda_i(\weights{1}_i x+\bias{1}_i)
$
for some nonzero scalar $\lambda_i\in\R$.
Since
$
[V\up{1}_i x+c^{(1)}_i]_+
$
has breakpoint hyperplane $H_i$, the sign of $\lambda_i$ determines on which side of $H_i$ the neuron is active.

By Proposition~\ref{prop:tropical_weight_1layer}, the tropical weight of $f_\theta$ along $H_i$ is
$
c_{f_\theta}(H_i)=\weights{2}_{:,i}\|\weights{1}_i\|.
$
Likewise, the tropical weight of $f_\eta$ along $H_i$ is
$
c_{f_\eta}(H_i)=V^{(2)}_{:,i}\|V\up{1}_i\|.
$
Since $f_\theta-f_\eta$ is affine linear on $P$, its tropical weight vanishes on every facet contained in $H_i$, hence
$
c_{f_\theta}(H_i)=c_{f_\eta}(H_i)
$ for all $i\in[n].
$
Therefore
$
V^{(2)}_{:,i}\|V\up{1}_i\|
=
\weights{2}_{:,i}\|\weights{1}_i\|.
$
Since $V\up{1}_i=\lambda_i\weights{1}_i$, we have $\|V\up{1}_i\|=|\lambda_i|\,\|\weights{1}_i\|$, and hence
$
V^{(2)}_{:,i}=\frac{1}{|\lambda_i|}\weights{2}_{:,i}.
$

Now consider the contribution of the $i$-th neuron to the linear part on a region $R$ of $\complex_\theta(P)$.
If $\lambda_i>0$, then the $i$th neuron of $f_\eta$ is active on exactly the same side of $H_i$ as the $i$th neuron of $f_\theta$, and its linear contribution on that side is
$
V^{(2)}_{:,i}V\up{1}_i
=
\frac{1}{|\lambda_i|}\weights{2}_{:,i}\cdot \lambda_i\weights{1}_i
=
\weights{2}_{:,i}\weights{1}_i.
$
If $\lambda_i<0$, then it is active on the opposite side, and its linear contribution there is
$
V^{(2)}_{:,i}V\up{1}_i
=
-\weights{2}_{:,i}\weights{1}_i.
$

Hence, the linear part of $f_\theta$ on $R$ is
$
A_R^\theta=\sum_{i\in S_\theta(R)} \weights{2}_{:,i}\weights{1}_i,
$
while the linear part of $f_\eta$ on $R$ is
$
A_R^\eta=\sum_{i\in S_\eta(R)} \pm\,\weights{2}_{:,i}\weights{1}_i,
$
where for each $i$, the sign is determined by $\lambda_i$.
Therefore the difference
$
A_R^\theta-A_R^\eta
$
is of the form
$
\sum_{i=1}^n \alpha_i\,\weights{2}_{:,i}\weights{1}_i
$ with $\alpha_i\in\{-1,0,1\}.
$
Since $f_\theta-f_\eta$ is affine linear on $P$, this difference is independent of the choice of region $R$, and equals the linear part of $f_\theta-f_\eta$ on $P$.
This proves the claim.
\end{proof}
\begin{lemma}
\label{lem:compress_visible_hyperplanes}

Let $P\subseteq \R^d$ be a polytope, and let
$
f(x)=V^{(2)}[V^{(1)}x+c^{(1)}]_+ + c^{(2)}
$
be a one-hidden-layer network with $n+k$ hidden neurons.
Assume that the breakpoint set of $f$ in $\relint(P)$ is the union of exactly $n$ pairwise distinct hyperplanes
$
H_1,\dots,H_n.
$
Then there exists another one-hidden-layer network
$
\tilde f(x)=\tilde V^{(2)}[\tilde V^{(1)}x+\tilde c^{(1)}]_+ + \tilde c^{(2)}
$
with $n+k$ hidden neurons such that:
\begin{enumerate}
    \item $\tilde f(x)=f(x)$ for all $x\in P$;
    \item exactly $n$ neurons of $\tilde f$ have breakpoint hyperplanes intersecting $\relint(P)$, namely $H_1,\dots,H_n$;
    \item the remaining $k$ neurons have breakpoint hyperplanes disjoint from $P$.
\end{enumerate}
\end{lemma}

\begin{proof}
Group the neurons of $f$ according to the visible hyperplane they induce in $\relint(P)$.
For each $j\in[n]$, let $I_j$ be the set of neurons whose breakpoint hyperplane in $\relint(P)$ is $H_j$, and let $K$ be the set of neurons whose nonzero locus does not appear as breakpoint of the final function, that is, whose breakpoint hyperplanes do not intersect $\relint(P)$ or are canceled by some other neurons breakpoint hyperplane. 
Then
$
[n+k]=I_1\sqcup\cdots\sqcup I_n\sqcup K.
$

Fix for each $j\in[n]$ an affine linear map
$
h_j(x)=u_jx+\beta_j
$
with zero set $H_j$.
Every neuron in $I_j$ has preactivation of the form $\lambda_{j,\ell}h_j(x)$ for some $\lambda_{j,\ell}\neq 0$.
Using
$
[-h_j(x)]_+=[h_j(x)]_+-h_j(x),
$
we can rewrite the total contribution of the neurons in $I_j$ as
$
\sum_{\ell\in I_j} V^{(2)}_{:,\ell}[\lambda_{j,\ell}h_j(x)]_+
=
a_j[h_j(x)]_+ + b_j h_j(x)
$
for suitable vectors $a_j,b_j\in\R^m$.

Thus
$
f(x)=g(x)+\alpha(x)
\text{ for all }x\in P,
$
where
$
g(x)\coloneqq \sum_{j=1}^n a_j[h_j(x)]_+
$
has exactly $n$ visible neurons with breakpoint hyperplanes $H_1,\dots,H_n$, and $\alpha$ is an affine linear map.

By construction, $g$ has the same tropical weight as $f$ along each visible hyperplane $H_j$.
Since $f$ and $g$ are compatible with the same arrangement on $P$, \Cref{lem:structure_theorem_tropical_geometry} implies that $f-g=\alpha$ is affine linear on $P$.

It remains to realize the affine linear map $\alpha$ by $k$ neurons whose hyperplanes are disjoint from $P$.
Replacing the $|I_j|$ neurons on $H_j$ by one visible neuron removes
$
\sum_{j=1}^n (|I_j|-1)
$
neurons, and together with the neurons in $K$ this leaves
$
|K|+\sum_{j=1}^n (|I_j|-1)=(n+k)-n=k
$
neurons available for the affine part.

The affine linear map $\alpha$ is the sum of the contributions of exactly these $k$ leftover neurons on $P$, so its linear part has rank at most $k$.
Conversely, any affine linear map of rank $r$ can be represented on $P$ by $r$ neurons whose hyperplanes are disjoint from $P$, together with the output bias term.
Hence $\alpha$ can be represented on $P$ by at most $k$ such neurons; if fewer are needed, we add dummy neurons with zero output weights.

Combining these $k$ outside neurons with the visible network $g$ yields a one-hidden-layer network $\tilde f$ with $n+k$ hidden neurons that agrees with $f$ on $P$, has exactly $n$ visible neurons with breakpoint hyperplanes $H_1,\dots,H_n$, and whose remaining $k$ neurons have breakpoint hyperplanes disjoint from $P$.
\end{proof}
See also \Cref{fig:push_away_hyperplanes} for an illustration.
\begin{lemma}
\label{lem:linear_part_matters}
Let $P \subseteq \R^d$ be a polytope of dimension at least $r$, and 
let $Q$ denote the orthogonal projection onto $\aff(P)$. 
Let $f_\theta \colon \R^d \to \R^n \to \R^m$ be a generic one-hidden-layer network with parameter $\theta = (\weights{1},\bias{1},\weights{2},\bias{2})$ such that all hyperplanes of $\HA_{\weights{1},\bias{1}}$ are inside $P$. 
Let $A \in \R^{m\times d}$ be a matrix such that 
    $(A+\sum_{i=1}^n \alpha_i\,\weights{2}_{:,i}\weights{1}_i)Q$ 
    has rank $r$ 
    for all  $\alpha_i\in\{-1,0,1\}.$ 
    Then any one-hidden-layer network $f$ satisfying $f(x)= f_\theta(x) + Ax$ for all $x \in P$ must have at least $n+r$ hidden neurons. 
\end{lemma}
\begin{proof}

Let
$
f(x)=V\up{2}[V\up{1}x+c\up{1}]_+ + c\up{2}
$
be a one-hidden-layer network with $N$ hidden neurons such that
$
f(x)=f_\theta(x)+Ax
\text{ for all }x\in P.
$
Since $Ax$ is affine linear, the breakpoint set of $f$ in $\relint(P)$ coincides with that of $f_\theta$.
By assumption, $f_\theta$ has exactly $n$ distinct breakpoint hyperplanes
$
H_1,\dots,H_n
$
inside $P$.
By \Cref{lem:compress_visible_hyperplanes}, there exists another one-hidden-layer network $\tilde f$ with $N$ hidden neurons such that:
\begin{enumerate}
    \item $\tilde f(x)=f(x)$ for all $x\in P$;
    \item exactly $n$ neurons of $\tilde f$ have breakpoint hyperplanes intersecting $\relint(P)$, namely $H_1,\dots,H_n$;
    \item the remaining $N-n$ neurons have breakpoint hyperplanes disjoint from $P$.
\end{enumerate}
Write
$
\tilde f(x)=g(x)+h(x)
\text{ for all }x\in P,
$
where $g$ is the subnetwork consisting of the $n$ visible neurons and $h$ is the subnetwork consisting of the remaining $N-n$ neurons.
Then $h$ is affine linear on $P$, since all of its breakpoint hyperplanes are disjoint from $P$.
Moreover,
$
g(x)=f_\theta(x)+Ax-h(x)
\text{ for all }x\in P,
$
so $f_\theta-g$ is affine linear on $P$.
Since both $f_\theta$ and $g$ have exactly $n$ visible neurons with visible hyperplanes $H_1,\dots,H_n$, Lemma~\ref{lem:fixed_width_affine_difference} implies that the linear part of $f_\theta-g$ is of the form
$
\sum_{i=1}^n \alpha_i\,\weights{2}_{:,i}\weights{1}_i
$.
Hence the affine-linear map
$
h(x)=Ax + (f_\theta-g)(x)
$
has linear part
$
A+ \sum_{i=1}^n \alpha_i\,\weights{2}_{:,i}\weights{1}_i.
$
By assumption, the projection
$
(A+\sum_{i=1}^n \alpha_i\,\weights{2}_{:,i}\weights{1}_i)Q
$
has rank $r$.
Therefore the linear part of $h|_{\aff(P)}$ has rank at least $r$.
On the other hand, $h$ is represented by $N-n$ neurons, each of which is affine linear on $P$ and contributes a matrix of rank at most $1$ to the linear part.
It follows that
$
N-n\ge r.
$
Therefore
$
N\ge n+r,
$
which proves the claim.
\end{proof}

\begin{figure}[t]
\centering
\begin{subfigure}[t]{0.48\textwidth}
\centering
\begin{tikzpicture}[scale=0.3,x=1cm,y=1cm,line cap=butt,line join=miter]
    \footnotesize
    \path[use as bounding box] (-8,-1) rectangle (19,15);

    \draw[black!80] (-1,-1) -- (-1,10) -- (10.0,10.0) -- (10,-1) -- cycle;

    \draw[red!70!black,line width=0.9pt] 
    (1,-1) -- (10,5);
    \draw[->,>=stealth,red!70!black] (7,3) -- (6.3333,4);
    \draw[->,>=stealth,red!70!black] (7,3) -- (7.6666,2);
    \draw[red!70!black,line width=0.9pt] 
    (-1,9) -- (4,-1);
    \draw[->,>=stealth,red!70!black] (2,3) -- (3,3.5);
    \draw[->,>=stealth,red!70!black] (2,3) -- (1,2.5);

\end{tikzpicture}
\caption{Several neurons induce the same visible breakpoint hyperplanes inside $P$.}
\end{subfigure}\hfill
\begin{subfigure}[t]{0.48\textwidth}
\centering
\begin{tikzpicture}[scale=0.3,x=1cm,y=1cm,line cap=butt,line join=miter]
    \footnotesize
    \path[use as bounding box] (-8,-1) rectangle (19,15);

    \draw[black!80] (-1,-1) -- (-1,10) -- (10.0,10.0) -- (10,-1) -- cycle;

    \draw[red!70!black,line width=0.9pt] 
    (1,-1) -- (10,5);
    \draw[red!70!black,line width=0.9pt] 
    (1-9,-1+9) -- (10-9,5+9);
    \draw[->,>=stealth,red!70!black] (7-9,3+9) -- (8-9,1.5+9);
    \draw[->,>=stealth,red!70!black] (7,3) -- (5,6);
    \draw[red!70!black,line width=0.9pt] 
    (-1,9) -- (4,-1);
    \draw[red!70!black,line width=0.9pt] 
    (-1+10,9+6) -- (4+10,-1+6);
    \draw[->,>=stealth,red!70!black] (2,3) -- (4,4);
    \draw[->,>=stealth,red!70!black] (2+10,3+6) -- (1+10,2.5+6);

\end{tikzpicture}
\caption{Only one neuron remains on each visible breakpoint hyperplane, whereas the remaining affine linear function is computed by neurons whose hyperplanes lie outside of $P$ and hence are affine linear on $P$.}
\end{subfigure}
\caption{Illustration of Lemma~\ref{lem:compress_visible_hyperplanes}.}
\label{fig:push_away_hyperplanes}
\end{figure}
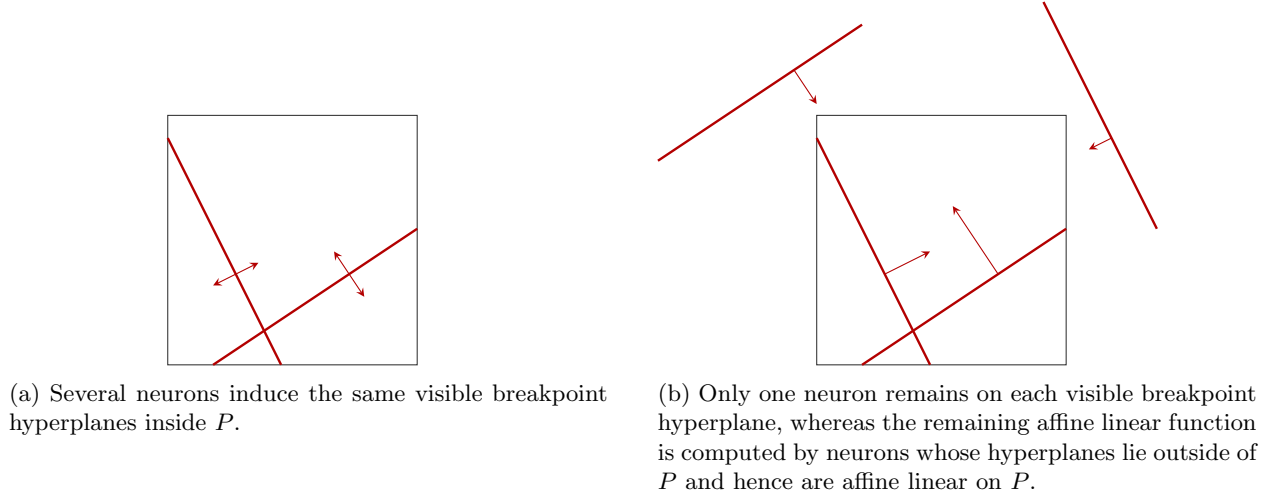
\section{Algebraic Superset of Fiber}
\label{app:alg_superset}
In this section, we describe (a superset of) the fiber algebraically. To do so, we first fix the combinatorial structure of the parameterization and then deriving the necessary polynomial constraints. 
Specifically, we first choose a discrete assignment of candidate bent hyperplanes to 
and an activation pattern for each facet, 
and then we establish equations for the weights and biases based on two local functional invariants: the affine span of the facets (local geometry) and the gradient jumps across them (tropical weights). 
The resulting set is a superset of the fiber, as it captures the necessary local algebraic consistency conditions but ignores the global semi-algebraic sign constraints and the base affine map required for functional identity.
\begin{definition}[Discrete fiber configuration]
Let $\theta \in \paramspace{\architecture_{\mathrm{orig}}}$ be a generic parameter, and let
$f_\theta$ be the corresponding network function. For each facet
$\sigma \in \bcomplex{\theta}^{d-1}$ of the breakpoint complex, let
$c_\theta(\sigma)$ be the corresponding non-zero tropical weight, and let
$a_\sigma \in \R^d$ be a unit vector and $\beta_\sigma \in \R$ be such that
$
\sigma \subseteq \{x \in \R^d \mid \inner{a_\sigma,x} + \beta_\sigma = 0 \}.
$

Given another (possibly different) architecture $\architecture$, a
\emph{discrete fiber configuration} for $f_\theta$ with respect to $\architecture$
is a pair $(\phi,\mathbf{s})$, where:
\begin{enumerate}
    \item $\phi \colon V_{f_\theta} \to V_\architecture$ is a map from the set of
    candidate bent hyperplanes of $f_\theta$ to the set of hidden neurons of
    $\architecture$ such that for every edge $(u,v) \in E_{f_\theta}$, if
    $
    \phi(u)=(i,\ell)$ and $\phi(v)=(j,k),
    $
    then $\ell < k$;

    \item $\mathbf{s} = \{\mathbf{s}(\sigma)\}_{\sigma \in \bcomplex{\theta}^{d-1}}$
    is an assignment of an activation pattern to each facet
    $\sigma \in \bcomplex{\theta}^{d-1}$.
\end{enumerate}

For such a configuration, we write
$
\varphi \colon \bcomplex{\theta}^{d-1} \to V_\architecture$ with $
\varphi(\sigma) \coloneqq \phi(\pi(\sigma)),
$
where $\pi \colon \bcomplex{\theta}^{d-1} \to V_{f_\theta}$ maps a facet to its unique
candidate bent hyperplane.
\end{definition}

\begin{definition}
Let $(\phi,\mathbf{s})$ be a discrete fiber configuration for $f_\theta$ with
respect to $\architecture$, and let $\varphi(\sigma)=\phi(\pi(\sigma))$ as above.
We say that a parameter $\eta \in \paramspace{\architecture}$ \emph{realizes}
$(\phi,\mathbf{s})$ if, for every facet
$\sigma \in \bcomplex{\theta}^{d-1}$,
\begin{enumerate}
    \item the facet $\sigma$ is contained in the bent hyperplane of the neuron
    $\varphi(\sigma)$ in the realization $\eta$, and
    \item the activation pattern induced by $\eta$ on $\sigma$ is equal to
    $\mathbf{s}(\sigma)$.
\end{enumerate}
\end{definition}

\begin{definition}[Configuration variety]
For an architecture $\architecture$ and a discrete fiber configuration
$(\phi,\mathbf{s})$ for $f_\theta$, let
$
\varphi(\sigma)=\phi(\pi(\sigma))$ for all $ \sigma \in \bcomplex{\theta}^{d-1}.
$
The \emph{configuration variety}
$
\mathcal{V}_{(\phi, \mathbf{s})}
\subseteq
\paramspace{\architecture} \times \R^{\bcomplex{\theta}^{d-1}} \times \R^{\bcomplex{\theta}^{d-1}}
$
is the algebraic variety in the variables
$
\eta = (\weights{\ell},\bias{\ell})_{\ell \in [L+1]}
$ and $
(\lambda_\sigma,\delta_\sigma)_{\sigma \in \bcomplex{\theta}^{d-1}}
$
defined as follows.

For each facet $\sigma \in \bcomplex{\theta}^{d-1}$, let
$
\varphi(\sigma)=(j,\ell)
$
and let
$
S_k(\sigma) = \{ r \in [n_k] \mid \mathbf{s}_{k,r}(\sigma)=+ \}.
$
We write
$
g_\sigma(\eta)
\coloneqq
\bigl(W^{(\ell)} D_{S_{\ell-1}(\sigma)} W^{(\ell-1)} \cdots D_{S_1(\sigma)} W^{(1)}\bigr)_j
\in \R^d
$
for the $j$-th row of the matrix product, i.e., the gradient of the
$(j,\ell)$-neuron preactivation evaluated under the activation pattern
$\mathbf{s}(\sigma)$, and
\[
t_\sigma(\eta)
\coloneqq
\left(
\bias{\ell}
+
\sum_{k =1}^{\ell-1}
W^{(\ell)} D_{S_{\ell-1}(\sigma)} W^{(\ell-1)} \cdots
D_{S_{k+1}(\sigma)}W^{(k+1)}D_{S_k(\sigma)}\bias{k}
\right)_j
\]
for the value of this preactivation at the origin.
Moreover, we define
$$
v_\sigma(\eta)
\coloneqq
W^{(L+1)} D_{S_L(\sigma)} W^{(L)} \cdots D_{S_{\ell+1}(\sigma)} \weights{\ell+1} e_j .
$$
Then $\mathcal{V}_{(\phi, \mathbf{s})}$ is cut out by the following equations,
imposed for every facet $\sigma \in \bcomplex{\theta}^{d-1}$:
\begin{enumerate}
    \item Geometric Alignment:
    $
    g_\sigma(\eta) = \delta_\sigma\lambda_\sigma a_\sigma
$ and $
    t_\sigma(\eta) = \delta_\sigma\lambda_\sigma \beta_\sigma.
    $

    \item Tropical Weight Matching: 
    $
    \lambda_\sigma v_\sigma(\eta) = c_\theta(\sigma).
    $
    \item Sign Equation: $\delta_\sigma^2=1$.
\end{enumerate}

\end{definition}
For fixed activation patterns $\mathbf s(\sigma)$, all expressions
$g_\sigma(\eta)$, $t_\sigma(\eta)$, and $v_\sigma(\eta)$ are polynomial in the parameters $\eta$.
Hence the above equations define an algebraic variety.
We denote by $\pi_\architecture(\mathcal{V}_{(\phi, \mathbf{s})})$ the projection onto $\paramspace{\architecture}$ and call this the \emph{configuration set}.
For a CPWL function $f$ and an architecture $\mathcal{A}$, let $\gfiber{f,\architecture} =\{\eta \in \gparamspace{\architecture} \mid f = f_\eta\}$ be its generic fiber.

\begin{prop}
\label{prop:containing_variety}
Let $\theta$ be a generic parameter and $\architecture$ a (potentially different)
architecture. For any discrete fiber configuration $(\phi,\mathbf{s})$ for
$f_\theta$ with respect to $\architecture$, let
$
\gfiber{f_\theta,\phi,\mathbf{s}}
\coloneqq
\{ \eta \in \gfiber{f_\theta,\architecture} \mid \eta \text{ realizes } (\phi,\mathbf{s}) \}.
$
Then
$
\gfiber{f_\theta,\phi,\mathbf{s}}
\subseteq
\pi_\architecture(\mathcal{V}_{(\phi,\mathbf{s})}).
$
Consequently, the generic fiber $\gfiber{\theta}$ is contained in the finite union
$
\gfiber{\theta}
\subseteq
\bigcup_{(\phi,\mathbf{s})}
\pi_\architecture(\mathcal{V}_{(\phi,\mathbf{s})}),
$
taken over all valid discrete fiber configurations.
\end{prop}

\begin{proof}
Let $\eta \in \gfiber{f_\theta,\phi,\mathbf{s}}$.
Then $f_\eta=f_\theta$, and $\eta$ realizes the configuration $(\phi,\mathbf{s})$.

Fix a facet $\sigma \in \bcomplex{\theta}^{d-1}$ and let
$
\varphi(\sigma)=\phi(\pi(\sigma))=(j,\ell).
$
Since $\eta$ realizes $(\phi,\mathbf{s})$, the facet $\sigma$ is contained in the
bent hyperplane of neuron $(j,\ell)$, and the activation pattern induced by $\eta$
on $\sigma$ is equal to $\mathbf{s}(\sigma)$. Hence the corresponding preactivation vanishes on
 $\sigma$. Therefore, setting
$
\lambda_\sigma \coloneqq \|g_\sigma(\eta)\|_2 \ge 0,
$
there exists $\delta_\sigma \in \{-1,1\}$ such that
$
g_\sigma(\eta)=\delta_\sigma\lambda_\sigma a_\sigma
$ and $
t_\sigma(\eta)=\delta_\sigma\lambda_\sigma \beta_\sigma
$.
Since $f_\eta=f_\theta$, the tropical weights agree on the breakpoint complex:
$
c_\eta(\sigma)=c_\theta(\sigma).
$
On the other hand, by \Cref{prop:tropweightofNNs}, the tropical weight induced by
the neuron $(j,\ell)$ under the activation pattern $\mathbf{s}(\sigma)$ is given by
$
c_\eta(\sigma)=\lambda_\sigma v_\sigma(\eta).
$
Hence
$
\lambda_\sigma v_\sigma(\eta)=c_\theta(\sigma).
$

Since this holds for every facet $\sigma \in \bcomplex{\theta}^{d-1}$, we obtain
$
(\eta,(\lambda_\sigma,\delta_\sigma)_\sigma)\in \mathcal V_{(\phi,\mathbf{s})}.
$
Therefore
$
\eta \in \pi_\architecture(\mathcal V_{(\phi,\mathbf{s})}),
$
which proves the claim.
\end{proof}

\end{document}